\documentclass[11pt]{article}
\usepackage{xinstyle}
\usepackage[a4paper,margin=0.8in]{geometry}
\usepackage[textsize=tiny]{todonotes}

\newcommand{\Lip}{{\rm Lip}}
\newcommand{\deff}{d_{\rm eff}}
\newcommand{\NN}{\sfN\kern-0.27em\sfN}
\newcommand{\Pa}{\mathsf{\Theta}}
\newcommand{\const}{{\rm const}}
\newcommand{\poly}{{\rm poly}}
\newcommand{\eqd}{\stackrel{\rm d}{=}}
\newcommand{\mcE}{\mathcal{E}}
\newcommand{\mcQ}{\mathcal{Q}}
\newcommand{\msF}{\mathscr{F}}
\newcommand{\msG}{\mathscr{G}}
\newcommand{\msN}{\mathscr{N}}
\newcommand{\sfE}{\mathsf{E}}
\newcommand{\sfV}{\mathsf{V}}

\newcommand{\softmax}{{\rm softmax}}

\newcommand{\ind}{\mathbb{1}}
\newcommand{\bu}{{\boldsymbol{u}}}
\newcommand{\bv}{{\boldsymbol{v}}}
\newcommand{\ip}[2]{\left\langle #1,#2 \right\rangle}

\newtheorem{prop}[thm]{Proposition}

\crefname{prop}{proposition}{propositions}
\Crefname{prop}{Proposition}{Propositions}
\crefname{asm}{assumption}{assumptions}
\Crefname{asm}{Assumption}{Assumptions}
\providecommand{\citep}[2][]{\cite{#2}}

\title{Localized Diffusion Models}
\author{Georg A.~Gottwald\thanks{School of Mathematics and Statistics, The University of Sydney (\texttt{georg.gottwald@sydney.edu.au}).} \and Shuigen Liu\thanks{Department of Mathematics, National University of Singapore (\texttt{shuigen@u.nus.edu}, \texttt{mattxin@nus.edu.sg}).} \and Youssef Marzouk\thanks{Laboratory for Information and Decision Systems, Massachusetts Institute of Technology (\texttt{ymarz@mit.edu}).} \and Sebastian Reich\thanks{Institut für Mathematik, Universität Potsdam (\texttt{sereich@uni-potsdam.de}).} \and Xin T.~Tong\footnotemark[2]}
\date{\today}

\begin{document}
\maketitle
\begin{abstract}

Diffusion models are state-of-the-art tools for various generative tasks.
Yet training these models involves estimating high-dimensional score functions, a task that in principle suffers from the curse of dimensionality. It is therefore important to understand how low-dimensional structure in the target distribution can be exploited in these models.
Here we consider \emph{locality structure}, which describes certain sparse conditional dependencies among the target random variables. Given some locality structure, the score function is effectively low-dimensional, so that it can be estimated by a localized neural network with significantly reduced sample complexity.
This observation motivates the \emph{localized diffusion model}, where a localized score matching loss is used to train the score function within a localized hypothesis space. We prove that such localization enables diffusion models to circumvent the curse of dimensionality with dimension-independent error bounds, at the price of additional localization error. Under realistic sample size scaling, we then show both theoretically and numerically that a moderate localization radius can balance the statistical and localization errors, yielding better overall performance.
Locality structure also facilitates parallel training, making localized diffusion models potentially more efficient for large-scale applications.

\end{abstract}

\section{Introduction}

Over the past decade, numerous neural network (NN)-based sampling algorithms have emerged in the machine learning literature, demonstrating remarkable performance in varied applications. These methods, often referred to as generative models, include normalizing flows \cite{pmlr-v37-rezende15}, variational auto-encoders \cite{kingma2013auto}, generative adversarial networks \cite{NIPS2014_5ca3e9b1}, and diffusion models \cite{NEURIPS2019_3001ef25,NEURIPS2020_4c5bcfec,song2021scorebased}.
Among generative models, diffusion models (DMs)---a term usually referring to ``denoising diffusion probabilistic models'' \cite{NEURIPS2020_4c5bcfec}---are a state-of-the-art and widely used approach.
They have gained great popularity due to their ability to generate high-quality samples, particularly in tasks such as image synthesis \cite{NEURIPS2020_4c5bcfec,NEURIPS2021_49ad23d1}.
Several recent studies \cite{NEURIPS2022_8ff87c96,chen2023sampling,benton2024nearly,pmlr-v202-oko23a,pmlr-v202-chen23o,pmlr-v247-wibisono24a} theoretically justify the approximation and generalization capabilities of DMs over a broad class of target distributions.
However, our understanding of their effectiveness in representing high-dimensional distributions remains limited.

DMs are known to be expensive when training with high-dimensional data.
The training sample size needs to grow exponentially with the problem dimension \cite{10.1214/aos/1176345969,pmlr-v202-oko23a}, a scaling known as the \emph{curse of dimensionality} (CoD) in the literature.
Various attempts have been made to avoid this scaling by leveraging low-dimensional structure within the target distribution.
The \emph{manifold hypothesis} \cite{MR3522608}, which postulates that the data lie on a low-dimensional manifold, is often invoked in these settings.
For such data, \cite{pmlr-v202-oko23a,pmlr-v202-chen23o,pmlr-v238-tang24a,azangulov2024convergence} show that the sample complexity of DMs depends on the dimension of the manifold rather than the ambient dimension, and that DMs can thus avoid the CoD with appropriate NN structure.
There are also studies considering Gaussian mixtures \cite{NEURIPS2023_3ec077b4,pmlr-v235-zhang24cn,gatmiry2024learning} to avoid the CoD.
In both the manifold hypothesis and Gaussian mixture models, although the ambient space is high-dimensional, there is a low-dimensional latent structure that effectively characterizes the target distribution.

While these concepts of low ``effective dimension'' cover many applications, there are still important cases left open.
One large class of high-dimensional distributions consists of those with \emph{locality structure} \cite{MR373208,MR4111677,24M1694781,cui2025stein}.
We say a joint distribution has locality structure if each coordinate random variable
has strong conditional dependence only on a small subset of the other coordinates.
Such sparse dependence structure arises naturally in spatial models, and has been successfully applied in various fields such as spatial statistics \cite{MR373208}, data assimilation \cite{MR3375889}, quantum mechanics \cite{PhysRevLett.76.3168}, and sampling \cite{MR4111677}.
We refer to \cite{cui2025stein} for a detailed review on locality structure.

An important property of distributions with locality structure, or \emph{localized distributions}, is that their score functions are effectively low-dimensional \cite{MR4111677,cui2025stein}.
Due to conditional independence, the $j$th component of the score function, $s_j(x) = \nabla_j \log p(x)$, depends only on $x_{\mcN_j}$, where $x_{\mcN_j}$ contains the components conditionally dependent on $x_j$.
If the conditional dependencies are sparse, the dimension of $x_{\mcN_j}$ is much smaller than the ambient dimension $D$, so that the score function can be regarded as a collection of low-dimensional functions $\{s_j\}_{j=1}^D$. This suggests that learning the score functions of localized distributions does not suffer from the CoD.

Motivated by this, we propose the \emph{localized diffusion model}, which embeds the locality structure into the hypothesis space of the score function, reducing a high-dimensional score matching problem to a low-dimensional one.
With small effective dimension, the statistical error of score estimation is significantly reduced.
On the other hand, localizing the hypothesis space introduces additional localization error.
Through a complete approximation and generalization analysis, we show that by adjusting the localization radius, one can balance the trade-off between the statistical error and the localization error to achieve smaller overall error.
This can be interpreted as a trade-off between variance and bias. Such a trade-off is validated by numerical experiments on high-dimensional time series data.
Finally, we find that the localized diffusion models can be interpreted as a collection of diffusion models on low-dimensional marginals. That is, we construct the samplers by combining local samplers for the marginals of the localized distributions.
This allows the localized diffusion models to be trained in parallel, which is practically important for large-scale applications.

The paper is organized as follows.
In \Cref{sec:DMnLoc}, we review diffusion models and the locality structure, and show that the locality structure is approximately preserved in the forward diffusion process.
In \Cref{sec:LocDM}, we introduce the localized diffusion model and analyze its approximation and statistical error.
We also discuss how to modify the localized diffusion model for mixture distributions.
In \Cref{sec:NumExp}, we present numerical experiments to validate our theoretical results.

\subsection{Related work}

\subsubsection{Analysis of diffusion models}

Since the introduction of DMs \cite{NEURIPS2019_3001ef25,NEURIPS2020_4c5bcfec,song2021scorebased}, there has been a surge of interest in understanding their theoretical properties.
Our work is built on two main lines of research: the convergence of DMs and the statistical analysis of DMs.
A comprehensive review of all related work is beyond the scope of this paper; we refer to \cite{chen2024overview,gatmiry2024learning} for an in-depth overview.

The convergence analysis of DMs studies error bounds for the sampled distribution given the learned score function.
Early work \cite{NEURIPS2022_8ff87c96} provides a TV guarantee by assuming a log-Sobolev inequality.
Using the Girsanov theorem, later work relaxes this condition to bounded moment conditions \cite{chen2023sampling,pmlr-v202-chen23q}.
A growing body of work seeks to further relax assumptions and improve error bounds.
For instance, \cite{benton2024nearly} proves a linear-in-dimension bound under the KL divergence; \cite{doi:10.1137/23M1613670} uses a relative score approach and derives bounds without early stopping; and \cite{potaptchik2024linear} considers manifold data and improves the bound of the discretization error to scale linearly with the manifold dimension.

The statistical analysis of DMs essentially studies the sample complexity of estimating the score function.
\cite{pmlr-v202-oko23a,pmlr-v247-wibisono24a} prove that the diffusion model reaches the minimax rate for distribution estimation.
To avoid the CoD, \cite{pmlr-v202-oko23a,pmlr-v202-chen23o} consider linear subspace data, and later \cite{pmlr-v238-tang24a,azangulov2024convergence} extend it to general manifold data.
Recently, \cite{yakovlev2025generalization} relaxes the manifold assumption and improves the ambient dimension dependence in the generalization bound.
Other types of low-dimensional structures are also considered.
\cite{NEURIPS2023_3ec077b4} considers certain Gaussian mixtures, and shows that the sample complexity does not depend exponentially on the dimension. \cite{gatmiry2024learning} further extends it to general Gaussian mixtures with edited diffusion models.
See \cite{stephanovitch2025generalization} for a comprehensive review on the statistical analysis of DMs.

Recently, \cite{JMLR:v27:25-0121,fan2025optimal} show that diffusion-model estimation can avoid the CoD for factorizable densities.
The class of factorizable densities is broader than the localized distributions considered here, but their results focus on the original diffusion models and study how sparse networks can be used to efficiently estimate the score function.
In contrast, we consider the design and analysis of diffusion models tailored to the locality structure, and establish dimension-independent error bounds for low-dimensional marginals, which are not provided in \cite{JMLR:v27:25-0121,fan2025optimal}.
\cite{10857317} studies a similar high-dimensional graphical model setting and proves that diffusion models can avoid the CoD. However, the analysis relies on an explicit solution of a denoising algorithm used to approximate the score function, and is restricted to Ising-type models. Our method applies more broadly to general high-dimensional graphical models.

Locality in diffusion models has also been studied from a different perspective, i.e., as a mechanism for ``creativity'' or generalization of DMs.
\cite{kamb2024analytic,niedoba2024towards} suggest that the inductive bias of localized denoisers, arising from architectures such as convolutional neural networks, may enable DMs to generalize rather than memorize.
However, their localized denoisers are introduced to explain this phenomenon, not to improve sample complexity, and their focus is empirical rather than rigorous numerical or statistical analysis.
\cite{lukoianov2025locality} argues that locality emerges from the data statistics rather than from the inductive bias of the architecture.
They provide empirical evidence that the learned score exhibits locality across different architectures, which is consistent with our main theoretical results on localized diffusion models.

\subsubsection{Localized samplers}

In recent years, there has been growing interest in sampling methods that leverage locality structures \cite{pmlr-v80-zhuo18a,MR3901708,MR4111677,gottwald2024localized}. These localized samplers follow the general strategy of building samplers by combining local samplers for the marginals.
\cite{MR3901708} proposes to apply the localization technique in Markov chain Monte Carlo (MCMC) and introduces a localized Metropolis-within-Gibbs sampler.
\cite{MR4111677} extends this idea and develops the MALA-within-Gibbs sampler, which is proven to admit a dimension-independent convergence rate.
Beyond MCMC, \cite{pmlr-v80-zhuo18a} proposes ``message passing Stein variational gradient descent,'' which finds the descent direction coordinate-wise and reduces the degeneracy issue of kernel methods in high dimensions.
\cite{gottwald2024localized} proposes a localized version of the Schr\"odinger bridge (SB) sampler \cite{rsta.2024.0332}, which replaces a single high-dimensional SB problem by $D$ low-dimensional SB problems, avoiding the exponential dependence of the sample complexity on the dimension.

\subsection{Notation}

\begin{itemize}
    \item Sets. Write $[D] = \{1,2,\dots,D\}$, and write the cardinality of a set $A$ as $|A|$. Given $x\in \mR^D$ and $A\subseteq [D]$, write $x_{A}$ for the subvector of $x$ with indices in $A$.
    \item Norms. For a vector $x\in\mR^D$, write $\norm{x}$ for its $\ell_2$-norm. For a matrix $A\in\mR^{m\times n}$, write $\norm{A} = \sup_{x\ne 0} \frac{\norm{Ax}}{\norm{x}}$ for the $2$-matrix norm. For a probability distribution $p$ and a function $f$, write $\norm{f}_{L^2(p)} = \Brac{\int f^2(x) p(x) \mdd x}^{1/2}$ for the weighted $L^2$-norm.
    \item Analysis. Write $a_n = \mcO(b_n)$ if there exists a constant $C>0$ such that $a_n \le C b_n$ for all $n$.
    Write $a_n = \Theta(b_n)$ if there exist constants $c,C>0$ such that $c b_n \le a_n \le C b_n$ for all $n$.
    \item Probability. Write $\Law(X)$ for the distribution of a random variable $X$.
    Without ambiguity, for a distribution $p$, we use $p(x)$ to denote its density at $x$.
    Write the covariance matrix of $X,Y$ as $\Cov_p(X,Y) := \mE_p [ \Brac{X-\mE_p[X]} \Brac{Y-\mE_p[Y]}\matT ]$. Write $\GN(\mu,\Sigma)$ for the Gaussian distribution with mean $\mu$ and covariance $\Sigma$. Write $X\ci Y \mid Z$ if $X$ is independent of $Y$ given $Z$, i.e., $\mP(X,Y|Z) = \mP(X|Z)\mP(Y|Z)$.
\end{itemize}

\section{Diffusion models and localized distributions}
\label{sec:DMnLoc}

\subsection{Diffusion models}

Diffusion models operate by simulating a process that gradually transforms a simple initial distribution, often Gaussian noise, into a complex target distribution, which represents the data of interest.
The core formulation involves two processes: a forward diffusion process which evolves data samples from the data distribution $p_0$ to noisy samples, and a reverse diffusion process that learns to progressively denoise the samples and effectively reconstruct the original data distribution.

Consider a forward Ornstein--Uhlenbeck (OU) process $(X(t))_{t\in[0,\mean{T}]}$ that is initialized with the target distribution $p_0$, i.e.,
\begin{equation}
\label{eqn:OU}
    \mdd X (t) = - X(t) \mdd t + \sqrt{2} \mdd W(t), \quad X(0) \sim p_0.
\end{equation}
Here $W(t)$ is a standard Brownian motion.
Let $(Y(t))_{t\in[0,\mean{T}]}$ denote its reverse process, so that $Y(t) = X(\mean{T}-t)$. Under mild conditions, $Y(t)$ follows the reverse SDE \cite{song2021scorebased}
\begin{equation}
\label{eqn:OUrev}
    \mdd Y(t) = \Brac{ Y(t) + 2 \nabla \log p_{\mean{T}-t} (Y(t)) } \mdd t + \sqrt{2} \mdd \widetilde{W}(t), \quad Y(0) \sim p_{\mean{T}},
\end{equation}
where we denote $p_t = \Law(X(t))$. The target distribution $p_0$ can then be sampled by first sampling $Y(0) \sim p_{\mean{T}}$ and then evolving $Y(t)$ according to \eqref{eqn:OUrev} to obtain a sample $Y(\mean{T}) \sim p_0$.

To implement the above scheme, several approximations are needed:
\begin{itemize}
    \item Score estimation. The score function $s(x,t) := \nabla \log p_t (x)$ is not accessible, and needs to be estimated from the data via the denoising score matching scheme \cite{6795935,NEURIPS2019_3001ef25,NEURIPS2020_4c5bcfec}
    \begin{equation}
        \widehat{s} = \argmin_{s_\theta} \mcL(s_\theta),
    \end{equation}
    \begin{equation}
    \label{eqn:DSM}
        \mcL(s_\theta) := \int_0^{\mean{T}} \mE_{x \sim p_0} \Rectbrac{ \mE_{x(t)\sim p_{t|0}(x(t)|x)} \Rectbrac{ \norm{ s_\theta(x(t),t) - \nabla_{x(t)} \log p_{t|0}(x(t)|x) }^2 } } \mdd t.
    \end{equation}
    Here $p_{t|0}(\cdot|x)$ denotes the conditional distribution of $X(t)$ given $X(0) = x$.
    In the sampling process, the true score $\nabla \log p_{\mean{T}-t}$ in \eqref{eqn:OUrev} is approximated by the estimated score $\widehat{s}(\cdot,\mean{T}-t)$.
    \item Approximation of $p_{\mean{T}}$. The initial distribution $p_{\mean{T}}$ in the reverse process is intractable. But since the OU process converges exponentially to $p_\infty = \GN(0,I)$, we can approximate $p_{\mean{T}}$ by $\GN(0,I)$ in \eqref{eqn:OUrev}.
    \item Early stopping. The reverse process is usually stopped at $t = \mean{T} - \underline{T}$ for some small $\underline{T} > 0$ to avoid potential blow up of the score function $s(t)$ as $t \to 0$. The early stopped samples satisfy $Y(\mean{T}-\underline{T}) \sim p_{\underline{T}}$, which is close to $p_0$ when $\underline{T}$ is small.
    \item Time discretization. The Euler-Maruyama scheme is used to discretize \eqref{eqn:OUrev}. Pick time steps $0 = t_0 < t_1 < \cdots < t_N = \mean{T} - \underline{T} $, and evolve $n=0,1,\dots,N-1$ by
    \begin{equation}
        Y (t_{n+1}) = Y (t_n) + \Delta t_n \Brac{ Y (t_n) + 2 \widehat{s}(Y (t_n),\mean{T}-t_n) } + \sqrt{2 \Delta t_n } \xi_n,
    \end{equation}
    where $\Delta t_n = t_{n+1} - t_n$ and $\xi_n \sim \GN(0,I)$. Design of the time steps (the schedule) is crucial for the empirical performance of the sampling process.
\end{itemize}

Note the OU process admits an explicit transition kernel
\begin{equation}
\label{eqn:OUkernel}
    p_{t|0}(x(t)|x) = \GN(x(t); \alpha_t x, \sigma_t^2 I), \quad \alpha_t := \mee^{-t}, \quad \sigma_t := \sqrt{1 - \mee^{-2t}}.
\end{equation}
Thus $\nabla_{x(t)} \log p_{t|0} (x(t)|x) = - \sigma_t^{-2}(x(t)-\alpha_t x)$, and $p_{t|0}(x(t)|x)$ can be realized as
\begin{equation}
    x(t) \eqd \alpha_t x + \sigma_t \epsilon(t), \quad \epsilon(t) \sim \GN(0,I).
\end{equation}
Here $\eqd$ denotes equality in distribution.
Thus the denoising score matching loss \eqref{eqn:DSM} can be written as
\begin{equation}
\label{eqn:DSM2}
    \mcL(s_\theta) = \int_{\underline{T}}^{\mean{T}} \mE_{x \sim p_0} \mE_{\epsilon(t) \sim \GN(0,I)} \Rectbrac{ \norm{ s_\theta(\alpha_t x + \sigma_t \epsilon(t), t) + \sigma_t^{-1} \epsilon(t) }^2 } \mdd t,
\end{equation}
where we introduce the early stopping truncation $\underline{T}$.
The above loss provides a convenient form for implementation \cite{NEURIPS2020_4c5bcfec}.

\subsection{Locality structure}

We will use the undirected graphical model \cite{MR2493908,MR2778120}, also known as Markov random field, to describe the locality structure.
In this model, the conditional dependencies of a collection of random variables are encoded in the underlying dependency graph.
Thus, the sparsity of the graph can be used to characterize the locality structure in the joint distribution of these random variables.
We will define the localized and approximately localized distributions based on the dependency graph.

\subsubsection{Sparse graphical models}

Following \cite{cui2025stein}, consider an undirected graph $\sfG=(\sfV,\sfE)$ and an associated random variable
\begin{equation}
    X = (X_i)_{i\in \sfV} \in \mR^D, \quad X_i \in \mR^{d_i}, \quad D = \sum_{i\in \sfV} d_i.
\end{equation}
Here we assume that the dimension of each component $d_i$ is small, but the total dimension $D$ is large. We say $X$ has \textit{dependency graph} $\sfG$ if, for any nonadjacent vertices $i,j\in \sfV$, $X_i, X_j$ are conditionally independent given the rest of the components $(X_k)_{k\ne i,j}$, i.e.,
\begin{equation}
\label{eqn:CI}
    X_i \ci X_j \mid (X_k)_{k\ne i,j}.
\end{equation}
$X$ is called a sparse graphical model if the dependency graph $\sfG$ is sparse, which essentially encodes the sparse local dependencies in $X$.
The following equivalent characterization \cite{MR4111677} of the sparse graphical models will be crucial. Let $p = \Law(X)$. If $p(x)$ is twice differentiable, then \eqref{eqn:CI} is equivalent to
\begin{equation}
\label{eqn:SparseHess}
    \forall \text{ nonadjacent } i,j \in \sfV  ~\St~ \nabla_{ij}^2 \log p(x) = 0.
\end{equation}
Let $b = |\sfV|$, and attach each vertex in $\sfV$ with a unique index $j\in [b]$. Denote
\begin{equation}
    \mcN_j := \{i \in \sfV: (i,j)\in \sfE\}
\end{equation}
as the neighboring vertices of $j$. For simplicity, we require that $\sfE$ includes all the self-loops in $\sfG$; i.e., $j \in \mcN_j$. We further define the extended neighborhood of $j$ as
\begin{equation}
\label{eqn:Nr}
    \mcN_j^r = \{ i \in \sfV: \sfd_\sfG(i,j) \le r \},
\end{equation}
where $\sfd_\sfG(i,j)$ is the graph path distance between $i,j\in \sfV$, i.e.,
\begin{equation}
\label{eqn:GraphDist}
    \sfd_\sfG(i,j) = \min \{ n \ge 0: \exists~\text{path of length } n \text{ from } i \text{ to } j \}.
\end{equation}

\subsubsection{Localized distributions}

Now we define localized and approximately localized distributions. The former corresponds precisely to sparse graphical models, and the latter is a relaxation based on \eqref{eqn:SparseHess}, which allows exponentially small long-range dependencies.

\begin{defn}
\label{def:LocDist}
A distribution $p$ is called localized w.r.t.~an undirected graph $\sfG$ if it satisfies \eqref{eqn:CI}. A distribution $p$ is called approximately localized w.r.t.~$\sfG$, if there exist dimension-independent constants $c_p,C_p>0$ such that
\begin{equation}
\label{eqn:wLocDist}
    \norm{ \nabla_{ij}^2 \log p }_\infty \le C_p \exp \Brac{-c_p \sfd_\sfG(i,j)}.
\end{equation}
Here $\norm{\cdot}_\infty$ denotes the $L^\infty$-norm, and $\sfd_\sfG(i,j)$ is the graph distance \eqref{eqn:GraphDist}.
\end{defn}

For a localized distribution, consider the $j$-th component of its score function
\begin{equation}
    s_j(x) = \nabla_j \log p(x).
\end{equation}
Note that it is only a function of $x_{\mcN_j}$, since by \eqref{eqn:SparseHess}, for any $i\notin \mcN_j$,
\begin{equation}
    \nabla_i s_j(x) = \nabla^2_{ij} \log p(x) = 0.
\end{equation}
For a sparse graph $\sfG$, $|\mcN_j|\ll |\sfV|$, so that $s_j$ is essentially a low-dimensional function, which implies that estimating $s_j$ does not suffer from the curse of dimensionality. This motivates us to leverage the locality structure in the hypothesis space of the score function and to localize the score matching procedure. The detailed methods will be discussed in \Cref{sec:LocDM}.

However, the low dimensionality of the score functions only holds for localized distributions. For approximately localized distributions, the score functions can only be approximated by low-dimensional functions. To improve the approximation accuracy, we can use the expanded neighborhood \eqref{eqn:Nr} for the approximate scores, i.e.,
\begin{equation}
\label{eqn:ApproxScore}
    s_j(x) \approx s_{\theta,j}(x_{\mcN_j^r}).
\end{equation}
Here $r$ is the radius of the neighborhood and can be tuned to balance the approximation accuracy and the sample complexity. Note by \eqref{eqn:wLocDist}, the approximation error decays exponentially with the radius $r$, while the dimension of $s_{\theta,j}$ only grows polynomially with $r$. \Cref{sec:LocDM} will provide a detailed analysis of the approximation error and the trade-off in the choice of $r$.

In this work, we assume that the dependency graph, or at least a simple neighborhood structure, is available \emph{a priori}.
This is natural in many applications with spatial or temporal structure, as in standard localization methods in data assimilation \cite{1998MWRv..126..796H,2001MWRv129.7690}, where neighborhoods are often prescribed by geometry or time ordering.
When such prior structure is unavailable, one could first estimate the graph using existing graph learning methods \cite{drton2017structure,NIPS2017_ea8fcd92,JMLR:v25:21-0022,Liaw2025} and then apply our method based on the learned graph.
We do not pursue this direction in this paper.

\subsection{Locality structure in diffusion models}
\label{sec:LocInDM}

We show in this section that the locality structure is preserved in the forward OU process, which lays the foundation for the localized score matching in diffusion models.

The explicit transition kernel \eqref{eqn:OUkernel} of the OU process implies that $p_t$ has an explicit density
\begin{equation}
    p_t(y) = \int \GN ( y; \alpha_t x, \sigma_t^2 I ) p_0(x) \mdd x.
\end{equation}
$p_t$ can be viewed as an interpolation between $p_0$ and $p_\infty = \GN(0,I)$. Suppose $p_0$ is a localized distribution w.r.t.~an undirected graph $\sfG$.
The limit distribution $p_\infty$ is localized, but the interpolation $p_t$ may not remain strictly localized. However, $p_t$ is still approximately localized, as proved in the following theorem.

\begin{thm}
\label{thm:wLocDist}
Suppose $p_0$ is localized w.r.t.~an undirected graph $\sfG$. Assume additionally that $p_0$ is log-concave and smooth, i.e., $\exists 0<m\le M <\infty$ such that $m I \preceq - \nabla^2 \log p_0(x) \preceq M I$. Then for any $t\in (0,\mean{T}]$, $p_t$ is approximately localized w.r.t.~$\sfG$. Specifically, for any $i \ne j$,
\begin{equation}
    \norm{ \nabla_{ij}^2 \log p_t }_\infty \le \frac{\alpha_t^2}{ \sigma_t^2 \Brac{ m \sigma_t^2 +\alpha_t^2 } } \Brac{ 1 - \frac{ m \sigma_t^2 + \alpha_t^2 }{ M \sigma_t^2 + \alpha_t^2 } }^{\sfd_\sfG(i,j)}.
\end{equation}
Here $\alpha_t = \mee^{-t}$ and $\sigma_t = \sqrt{1 - \mee^{-2t}}$ (cf.~\eqref{eqn:OUkernel}), and $\sfd_\sfG(i,j)$ is the graph distance \eqref{eqn:GraphDist}.
\end{thm}

See \Cref{app:wLocDist} for the proof. The first step is to show that when $i\ne j$,
\begin{equation}
    \nabla_{ij}^2 \log p_t(y) = \alpha_t^2 \sigma_t^{-4} \Cov_{p_{0|t}(x|y)} \Brac{ x_i, x_j }.
\end{equation}
Here $p_{0|t}(\cdot|y)$ is the conditional distribution of $X(0)$ given $X(t) = y$.
The bound then directly follows \Cref{prop:CorrExpDecay} below, which establishes the exponential decay of correlations between $x_i, x_j$ w.r.t.~their graph distance $\sfd_\sfG(i,j)$ for localized distributions. This is a ubiquitous property for distributions with locality structure \cite{PhysRevLett.76.3168,MR3375889,cui2025stein}.

\begin{rem}
In \Cref{thm:wLocDist}, log-concavity is used through \Cref{prop:CorrExpDecay} to obtain the exponential decay of correlations with explicit quantitative control.
We expect similar locality properties to hold beyond the log-concave setting, but establishing such results would require different assumptions and a more delicate analysis, which is outside the scope of this paper.
In \Cref{sec:ldm_mix}, we discuss a special case of non-log-concave distributions, where the distribution is a mixture of localized distributions.
We show that the score function can still be approximated by low-dimensional functions effectively.

\end{rem}

\begin{prop}
\label{prop:CorrExpDecay}
Suppose $p$ is localized w.r.t.~an undirected graph $\sfG$ and is log-concave and smooth, i.e., $\exists 0<m\le M <\infty$ such that $m I \preceq - \nabla^2 \log p(x) \preceq M I$. Then for any $i,j$ and Lipschitz functions $f: \mR^{d_i} \to \mR$ and $g: \mR^{d_j} \to \mR$, it holds that
\begin{equation}
    \norme{ \Cov_{p(x)} \Brac{ f(x_i), g(x_j) } } \le \frac{1}{m} \Brac{ 1 - \frac{m}{M} }^{\sfd_\sfG(i,j)} \norme{ f }_\Lip \norme{g}_\Lip.
\end{equation}
\end{prop}
See \Cref{app:CorrExpDecay} for the proof.

\begin{rem}
\label{rem:CondNum}
We note that the condition number $\kappa:=M/m$ of typical localized distributions is independent of the dimension $D$.
This is in contrast to the distributions satisfying the manifold hypothesis, where the condition number typically grows with the ambient dimension $D$.
The key difference is how the dependence structure scales with the dimension.
See \Cref{app:ManiVsLoc} for an illustrative example.
\end{rem}

\section{Localized diffusion models}
\label{sec:LocDM}

\subsection{Localized denoising score matching}

\subsubsection{Localized hypothesis space}

To exploit locality structure in diffusion models, we introduce the localized hypothesis space
\begin{equation}
\label{eqn:loc_hyp_space}
    \msH_r = \left\{ s_\theta : \mR^{D+1} \to \mR^D \mid s_{\theta,j}(x,t) = u_{\theta,j} (x_{\mcN_j^r},t),~ u_{\theta,j} \in \msU_j,~ j \in [b] \right\},
\end{equation}
where $r$ denotes the localization radius, $\mcN_j^r$ is the extended neighborhood \eqref{eqn:Nr}, and $\msU_j$ is the hypothesis space for the $j$-th component of the score function.
Here we use $s_{\theta,j}(\cdot,t)$ with an extended neighborhood to approximate the score function of $p_t$ in light of \Cref{thm:wLocDist}.
In practice, $\msH_r$ can be realized by a neural network (NN) with locality constraints; for instance, we can take each $\msU_j$ as a sparse ReLU NN introduced in \Cref{app:NN}.

Define the \emph{effective dimension} of $s_\theta$ as
\begin{equation}
\label{eqn:deff}
    \deff := \max_{j \in [b]} d_{j,r}, \quad d_{j,r} := \sum_{i\in \mcN_j^r} d_i.
\end{equation}
Since $s_\theta(\cdot,t)$ can be viewed as a collection of functions $ \{ u_{\theta,j}(\cdot,t): \mR^{d_{j,r}} \to \mR^{d_j} \}_{j\in[b]}$, it is essentially a function of $\deff$ variables. For a sparse graph, $\deff \ll D$, so intuitively estimating $s_\theta$ in $\msH_r$ does not suffer from the CoD.

\subsubsection{Localized score matching}

Given the hypothesis space $\msH_r$ \eqref{eqn:loc_hyp_space} with localized function classes $\{\msU_j\}_{j=1}^b$, we can learn the localized score function by minimizing the denoising score matching loss \eqref{eqn:DSM2}.
Given an i.i.d.\ sample $\{X^{(i)}\}_{i=1}^N$ from $p_0$, the population loss \eqref{eqn:DSM2} is approximated by the empirical loss
\begin{equation}
\label{eqn:emp_dsm}
    \widehat{\mcL}^N (s_\theta) = \frac{1}{N} \sum_{i=1}^N \int_{\underline{T}}^{\mean{T}} \mE_{\epsilon (t) \sim \GN(0,I)} \Rectbrac{ \norm{ s_\theta(\alpha_t X^{(i)} + \sigma_t \epsilon(t), t) + \sigma_t^{-1} \epsilon(t) }^2 } \mdd t.
\end{equation}
Note that $\widehat{\mcL}^N$ is decomposable: $\widehat{\mcL}^N (s_\theta) = \sum_{j=1}^b \widehat{\mcL}_j^N (u_{\theta,j})$, where
\begin{equation}
\label{eqn:emp_dsmj}
    \widehat{\mcL}_j^N (u_{\theta,j}) = \frac{1}{N} \sum_{i=1}^N \int_{\underline{T}}^{\mean{T}} \mE_{\epsilon (t) \sim \GN(0,I)} \Rectbrac{ \norm{ u_{\theta,j}(\alpha_t X_{\mcN_j^r}^{(i)} + \sigma_t \epsilon_{\mcN_j^r}(t), t) + \sigma_t^{-1} \epsilon_j(t) }^2 } \mdd t.
\end{equation}
Thus, one can minimize $\widehat{\mcL}_j^N$ for each $j$ to learn $u_{\theta,j}$.
This allows for \emph{parallel training} of the localized loss, i.e., the components of the score function can be trained independently. Note the score function need not be a gradient field, which introduces considerable flexibility in designing the hypothesis space.
Also note that the components of the score function are correlated for general distributions, so that $\{s_{\theta,j}(x)\}_{j=1}^b$ should be trained simultaneously.
However, for approximately localized distributions, most components of $s_\theta$ are almost uncorrelated, so that they can be trained in parallel.

\subsection{Error analysis}
\label{sec:AppErr}

In this section, we consider the error of the sampling process with the localized score function.
We do not consider time discretization here for simplicity.
The sampling process is
\begin{equation}
\label{eqn:SamplSDE}
    \mdd \widehat{Y} (t) = \Brac{ \widehat{Y} (t) + 2 \widehat{s}(\widehat{Y} (t),\mean{T}-t) } \mdd t + \sqrt{2} \mdd W(t), \quad \widehat{Y} (0) \sim \GN(0,I).
\end{equation}
We take the early-stopped distribution $\widehat{q}_{\mean{T}-\underline{T}} = \Law(\widehat{Y}_{\mean{T}-\underline{T}})$ as the approximation of $p_0$. It suffices to consider the error between $\widehat{q}_{\mean{T}-\underline{T}}$ and $p_{\underline{T}}$, as it is easier to control the early-stopping error, i.e., the distance between $p_{\underline{T}}$ and $p_0$.

\subsubsection{Error decomposition}

The following error decomposition is standard \cite{chen2023sampling}.

\begin{prop}
\label{prop:ErrDecomp}
Under Novikov's condition \cite{chen2023sampling}:
\begin{equation}
    \mE_\sfQ \Rectbrac{ \exp \Brac{ \frac{1}{2} \int_0^{\mean{T}-\underline{T}} \norm{ \widehat{s}(Y(t),\mean{T}-t) - s(Y(t),\mean{T}-t) }^2 \mdd t } } < \infty,
\end{equation}
where $\sfQ = \Law((Y(t))_{t \in [0,\mean{T}-\underline{T}]})$ denotes the path measure of the reverse process \eqref{eqn:OUrev}. It holds that
\begin{equation}
\label{eqn:errDecomp}
    \KL ( p_{\underline{T}} \| \widehat{q}_{\mean{T}-\underline{T}} ) \le \mee^{-2\mean{T}} \KL ( p_0 \| \GN(0,I) ) + \int_{\underline{T}}^{\mean{T}} \mE_{x(t)\sim p_t} \Rectbrac{ \norm{ \widehat{s}(x(t),t) - s(x(t),t) }^2 } \mdd t.
\end{equation}
\end{prop}

See \Cref{app:ErrDecomp} for the proof. We note that the first term in \eqref{eqn:errDecomp} can be replaced by $\mee^{-2(\mean{T}-\underline{T})} \KL ( p_{\underline{T}} \| \GN(0,I) )$ when $p_0$ is singular w.r.t.~$\GN(0,I)$, so that it always decays exponentially in $\mean{T}$ regardless of $p_0$.
The key is to control the second term, i.e., the score approximation error.

\subsubsection{Localized score function}

As discussed in \Cref{sec:LocInDM}, strict locality might not be preserved in the forward OU process, so that the true score $s\notin \msH_r$ in general.
We first control the optimal approximation error in $\msH_r$.

Let $\msH_r^*$ denote the hypothesis space \eqref{eqn:loc_hyp_space} when we take $\msU_j = \msU_j^* := C^1(\mR^{d_{j,r}+1})$.
For a $C^1$ function $s$ to be in $\msH_r^*$, the only constraint is the locality structural constraint.
Thus, the best possible approximation error in $\msH_r^*$ can be identified as the \emph{localization error} of the score function.
Motivated by \eqref{eqn:errDecomp}, we consider the optimal approximation in the $L^2(p_t)$ sense, i.e.,
\begin{equation}
    s^* = \argmin_{s_\theta \in \msH_r^*} \int_{\underline{T}}^{\mean{T}} \int \norm{ s_\theta(x,t) - s(x,t) }^2 p_t(x) \mdd x \mdd t
\end{equation}
\begin{equation}
    \ioi~ \forall j, ~s^*_j(x,t) = u_j^*(x_{\mcN_j^r},t), \quad u_j^* = \argmin_{u_{\theta,j} \in \msU_j^* } \int_{\underline{T}}^{\mean{T}} \int \normo{ u_{\theta,j}(x_{\mcN_j^r},t) - s_j(x,t) }^2 p_t(x) \mdd x \mdd t.
\end{equation}
Using the property of conditional expectation, it is straightforward to show that the minimizer is
\begin{equation}
\label{eqn:opt_loc_score}
\begin{split}
    u_j^*(x_{\mcN_j^r},t) =~& \mE_{x'\sim p_t} \Rectbrac{ s_j(x',t) ~\big|~ x_{\mcN_j^r}' = x_{\mcN_j^r} } \\
    =~& \frac{1}{\int p_t(x_{\mcN_j^r},x_{\mcN_j^{r_\bot}}) \mdd x_{\mcN_j^{r_\bot}}} \int \nabla_j \log p_t(x_{\mcN_j^r},x_{\mcN_j^{r_\bot}}) p_t(x_{\mcN_j^r},x_{\mcN_j^{r_\bot}}) \mdd x_{\mcN_j^{r_\bot}}.
\end{split}
\end{equation}
Here we denote $\mcN_j^{r_\bot} := [b] \setminus \mcN_j^r$.

Due to the approximate locality (\Cref{thm:wLocDist}), one can expect that the approximation error decays exponentially with the radius $r$.
Before presenting the approximation result, we introduce a quantitative condition \cite{cui2025stein} characterizing the sparsity of the graph $\sfG$.

\begin{defn}
\label{def:loc_graph}
An undirected graph $\sfG$ is called $(S,\nu)$-local if
\begin{equation}
\label{eqn:loc_graph}
    \forall j \in \sfV, ~ r \in \mN, \quad |\mcN_j^r| \le 1 + S r^{\nu} .
\end{equation}
\end{defn}

In the above definition, $S$ denotes the maximal size of the immediate neighbor, and $\nu$ denotes the \emph{spatial dimension} of the graph, which controls the growth rate of the neighborhood volume with the radius. Here we require it growing at most polynomially to ensure effective locality. Note the spatial dimension $\nu$ is typically a small number.
A motivating example for \Cref{def:loc_graph} is the lattice model $\mZ^{\nu}$, where a naive bound of the neighborhood volume is
\begin{equation}
    |\mcN_j^r| = | \{i\in \mZ^{\nu}: \norm{i}_1 \le r\} | \le (2r+1)^{\nu} < 1 + (3r)^{\nu}.
\end{equation}
Thus, $\mZ^{\nu}$ is $(3^{\nu},\nu)$-local.

Now we state the approximation result.
\begin{thm}
\label{thm:loc_err}
Let $p_0$ satisfy the conditions in \Cref{thm:wLocDist}, and suppose its dependency graph is $(S,\nu)$-local. Consider the hypothesis space $\msH_r^*$ \eqref{eqn:loc_hyp_space} with $\msU_j = C^1(\mR^{d_{j,r}+1})$. Then there exists an optimal approximation $s^* \in \msH_r^*$ such that
\begin{equation}
\label{eqn:loc_err}
    \int_{\underline{T}}^{\mean{T}} \normo{ s_j^*(x,t) - s_j (x,t) }_{L^2(p_t)}^2 \mdd t \le C d_j (r+1)^{\nu} \mee^{-c(r+1)},
\end{equation}
where $C$ and $c$ are some dimension-independent constants depending on $m,M,S,\nu$, i.e.,
\begin{equation}
    C = 2 S \max\{1,m^{-1}\} \nu!  \kappa^{2\nu+1} \log \kappa, \quad c = -2\log (1-\kappa^{-1}).
\end{equation}
Note \eqref{eqn:loc_err} is independent of $\underline{T},\mean{T}$.
Moreover, for any $s_\theta \in \msH_r^*$, the Pythagorean equality holds
\begin{equation}
\label{eqn:PythDecomp}
    \norm{ s_{\theta,j} (x,t) - s_j (x,t) }_{L^2(p_t)}^2 = \normo{ s_{\theta,j} (x,t) - s_j^*(x,t) }_{L^2(p_t)}^2 + \normo{ s_j^*(x,t) - s_j (x,t) }_{L^2(p_t)}^2.
\end{equation}
\end{thm}

See \Cref{app:loc_err} for the proof.
\eqref{eqn:loc_err} provides an upper bound for the hypothesis error of using a localized score function to approximate the true score function.
Note that the bound has no explicit dependence on the ambient dimension $D$, although the true score $s_j(x,t)$ is a $D$-dimensional function.
We also comment that condition number $\kappa$ of typical localized distributions is independent of $D$; see \Cref{app:ManiVsLoc}.
Since the bound decays exponentially (up to a polynomial factor) w.r.t.~the radius $r$, a small $r$ is sufficient to achieve a good approximation.
Finally, taking the sum over $j\in[b]$ in \eqref{eqn:loc_err} gives the total approximation error
\begin{equation}
    \int_{\underline{T}}^{\mean{T}} \norm{ s^* (x,t) - s(x,t) }_{L^2(p_t)}^2 \mdd t \le C D (r+1)^{\nu} \mee^{-c(r+1)}.
\end{equation}

\subsection{Sample complexity}

In this section, we show the key advantage of localized diffusion models, i.e., that they do not suffer from the CoD in terms of sample complexity.
We will show that the denoising score matching with the localized hypothesis space $\msH_r$ is equivalent to learning the $L^2$-optimal localized score \eqref{eqn:opt_loc_score}.
Since the localized scores are low-dimensional functions, the sample complexity should be independent of $D$.

\subsubsection{Equivalence to diffusion models for marginals}

A key observation is that the localized denoising score matching loss \eqref{eqn:emp_dsmj} is \emph{equivalent} to the loss for the $j$-th score component in a standard diffusion model applied to the \emph{marginal distribution} $\pi_0^j(x_{\mcN_j^r}) := \int p_0(x_{\mcN_j^r},x_{\mcN_j^{r_\bot}}) \mdd x_{\mcN_j^{r_\bot}}$.
First, notice that since each component $X_j(t)$ in the forward process \eqref{eqn:OU} evolves independently, the marginal distribution
\begin{equation}
\label{eqn:marginal_pt}
    \pi_t^j(x_{\mcN_j^r}) = \int p_t(x_{\mcN_j^r},x_{\mcN_j^{r_\bot}}) \mdd x_{\mcN_j^{r_\bot}}
\end{equation}
coincides with the time-$t$ distribution of the same forward process initialized from $\pi_0^j$.

\begin{prop}
\label{prop:dsm_j}
The localized denoising score matching is equivalent to denoising score matching for the marginal distributions.
To be precise, the following identity holds:
\begin{equation}
\begin{split}
    \mcL_j(u_{\theta,j}) :=~& \mE_{x\sim p_0} \int_{\underline{T}}^{\mean{T}} \mE_{\epsilon(t) \sim \GN(0,I)} \Rectbrac{ \norm{ u_{\theta,j} \Brac{ \alpha_t x_{\mcN_j^r} + \sigma_t \epsilon_{\mcN_j^r}(t), t } + \sigma_t^{-1} \epsilon_{j}(t) }^2 } \mdd t \\
    =~& \mE_{z \sim \pi_0^j} \int_{\underline{T}}^{\mean{T}} \mE_{z(t) \sim \pi_{t|0}^j(z(t) |z)} \Rectbrac{ \norm{ u_{\theta,j}(z(t), t) - \nabla_j \log \pi_{t|0}^j (z(t) |z) }^2 } \mdd t \\
    =~& \int_{\underline{T}}^{\mean{T}} \mE_{z(t) \sim \pi_t^j} \Rectbrac{ \norm{ u_{\theta,j}(z(t), t) - u_j^* (z(t), t)  }^2 } \mdd t + \const.
\end{split}
\end{equation}
Here $\pi_{t|0}^j(\cdot|z)$ is the conditional distribution of $z(t)$ given $z(0) = z$ when $z(t)$ follows the forward OU process \eqref{eqn:OU} with initial distribution $\pi_0^j$; $u_j^*(\cdot,t)$ is the optimal localized score \eqref{eqn:opt_loc_score}, which coincides with $\nabla_j \log \pi_t^j$;
and the constant depends only on $\pi_0^j$.
\end{prop}

See \Cref{app:dsm_j} for the proof.
\Cref{prop:dsm_j} implies that the localized score matching can be regarded as $b$ diffusion models, each of which aims to fit (one component of) the score function of a low-dimensional marginal distribution $\pi_0^j(x_{\mcN_j^r})$.
Using the minimax results for diffusion models, e.g.~\cite{pmlr-v202-oko23a}, one immediately obtains that the sample complexity of the localized score matching is essentially independent of the ambient dimension $D$.

\subsubsection{A complete error analysis}
\label{sec:CompErr}

In this section, we provide a complete error analysis of  localized diffusion models, which combines the localization error and the statistical error.
The error analysis clearly illustrates the trade-off associated with the choice of localization radius $r$, and shows that a moderate localization radius can balance the statistical and localization errors, yielding better overall performance.

We assume that the target distribution $p_0$ belongs to a probability family $\mcP$ that satisfies the following two assumptions.
The first assumption is to make sure that \Cref{thm:wLocDist} holds.

\begin{asm}
\label{asm:p_loc}
There exists a $(S,\nu)$-local graph $\sfG$ such that any $p_0 \in \mcP$ is localized w.r.t.~$\sfG$ in the sense of \Cref{def:LocDist}.
Moreover, there exist $m,M > 0$ such that for any $p_0 \in \mcP$,
\begin{equation}
    m I \preceq - \nabla^2 \log p_0(x) \preceq M I, \quad \forall x \in \mR^D.
\end{equation}
\end{asm}

Next, we assume some regularity conditions that ensure good approximation properties of the score function.
By \Cref{prop:dsm_j}, each localized score matching problem is equivalent to a denoising score matching problem for the marginal distribution $\pi_0^j$.
Thus, it suffices to assume that the score functions of $\pi_t^j$ can be approximated well by some hypothesis space $\msU_j^\varepsilon$.

To measure the size of a function class, we use the \emph{covering number} $\msN$ \cite{MR1385671} defined as
\begin{equation}
    \msN (t,\msU, \norm{\cdot}_{L^\infty} ) := \min \{ n \in \mN: \exists u_1,\dots,u_n \in \msU,~ \forall u \in \msU,~ \min_{i\in[n]} \norm{u-u_i}_{L^\infty} < t  \}.
\end{equation}
$\log \msN (t,\msU, \norm{\cdot}_{L^\infty} )$ is usually called the \emph{metric entropy} of $\msU$.

\begin{asm}
\label{asm:score_approx}
Suppose that there exists $\gamma>0$ such that, for every sufficiently small $\varepsilon > 0$, there exist $0 < \underline{T}^\varepsilon < \mean{T}^\varepsilon$, $\underline{T}^\varepsilon \asymp \varepsilon^{\alpha}$ and $\mean{T}^\varepsilon \asymp \log \varepsilon^{-1}$ with some $\alpha > 0$.
For any $j \in [b]$, suppose also that there exists a function class $\msU_j^\varepsilon \subseteq \{ u: \mR^{d_{j,r}} \times [\underline{T}^\varepsilon, \mean{T}^\varepsilon] \to \mR^{d_j} \}$ satisfying the following conditions:

\noindent
(1) Uniform boundedness. There exists $M_j^\varepsilon = \mcO(\log \varepsilon^{-1})$ such that
\begin{equation}
\label{eqn:u_bound}
    \forall u \in \msU_j^\varepsilon,\quad \sup_{t \in [\underline{T}, \mean{T}]} \sigma_t \norm{u(\cdot,t)}_\infty \le M_j^\varepsilon.
\end{equation}
(2) Approximation capacity.
Let $\pi_t^j$ denote the marginal distribution of $p_t$ on $x_{\mcN_j^r}$ \eqref{eqn:marginal_pt}.
It holds that
\begin{equation}
\label{eqn:u_approx}
    \sup_{p_0 \in \mcP } \inf_{u \in \msU_j^\varepsilon} \int_{\underline{T}^\varepsilon}^{\mean{T}^\varepsilon} \int \normo{ u(z,t) - \nabla \log \pi_t^j(z) }^2 \pi_t^j(z) \mdd z \mdd t \le \varepsilon.
\end{equation}
(3) Metric entropy condition. There exists $A_j^\varepsilon$ and $\Gamma_j^\varepsilon$ such that for any $t > 0$,
\begin{equation}
\label{eqn:ent_cond}
\begin{aligned}
    \log \msN ( t,\msU_j^\varepsilon, \norm{\cdot}_{L^\infty} ) \le~& A_j^\varepsilon \log_+ ( \Gamma_j^\varepsilon / t ), \\
    \text{where}\quad A_j^\varepsilon = \mcO( \varepsilon^{-d_{j,r}/(2\gamma) } \poly \log \varepsilon^{-1} ),& \quad \Gamma_j^\varepsilon = \exp \Brac{ \mcO( \poly \log \varepsilon^{-1} ) }
\end{aligned}
\end{equation}
Here $\log_+ (x) := \max\{0, \log x\}$.
Note $M_j^\varepsilon, A_j^\varepsilon, \Gamma_j^\varepsilon$ are constants independent of $D$.
\end{asm}

\begin{rem}
\label{rem:asm_score_approx}
(1) The boundedness assumption \eqref{eqn:u_bound} is used in \cite{pmlr-v202-oko23a,stephanovitch2025generalization} and is natural, since the score function $s(\cdot,t)$ typically scales as $\sigma_t^{-1}$.
This condition can be enforced by clipping the function values.

(2) The key assumption above is the entropy condition \eqref{eqn:ent_cond}.
By enlarging the function class $\msU_j^\varepsilon$, one can achieve small approximation error in \eqref{eqn:ent_cond}, but it leads to a larger metric entropy.
The scaling of $A_j^\varepsilon, \Gamma_j^\varepsilon$ in \eqref{eqn:ent_cond} controls how fast the size of $\msU_j^\varepsilon$ grows in order to achieve $\varepsilon$-approximation capacity.

(3) Since $\msU_j^\varepsilon$ is a class of functions from $\mR^{d_{j,r} + 1}$ to $\mR^{d_j}$, the
 typical scaling of $A_j^\varepsilon$ should be $\mcO( \varepsilon^{-(d_{j,r}+1)/(2\gamma) } )$ instead of $\mcO( \varepsilon^{-d_{j,r}/(2\gamma) } )$.
But as a function of $t$, $\nabla \log \pi_t^j(z)$ is completely determined by $\nabla \log \pi_0^j(z)$ through the forward process, so that its effective dimension is $d_{j,r}$.
Based on this observation, several previous works \cite{pmlr-v202-oko23a,yakovlev2025generalization,stephanovitch2025generalization,JMLR:v27:25-0121} have derived the improved rates through constructive approximation using neural network classes.
We refer to \Cref{app:NN} for more details on the construction of $\msU_j^\varepsilon$.
\end{rem}

\begin{thm}
\label{thm:samp_comp}
Assume that $p_0 \in \mcP$ where $\mcP$ satisfies \Cref{asm:p_loc,asm:score_approx}.
Given sample size $N \gg 1$, take $\varepsilon = N^{-2\gamma/(\deff+2\gamma)}$ where $\deff$ is the effective dimension \eqref{eqn:deff}, and $ \underline{T} = \underline{T}^\varepsilon,  \mean{T} = \mean{T}^\varepsilon$ as in \Cref{asm:score_approx}.
Consider the empirical minimizer of \eqref{eqn:emp_dsm} in a localized hypothesis space $\msH_r^N $, i.e., $\widehat{s} \in \argmin_{ s_\theta \in \msH_r^N } \widehat{\mcL}^N (s_\theta)$, where
\begin{equation}
    \msH_r^N := \left\{ s: \mR^{D+1} \to \mR^D ~\big|~ \forall j,~  s_j(x,t) = u_j(x_{\mcN_j^r},t), ~ u_j \in \msU_j^\varepsilon \right\}.
\end{equation}
Let $\widehat{q}_{\mean{T}-\underline{T}}$ denote the sampled distribution using the learned score $\widehat{s}$.
Then for any $\tau \in (0,1)$, with probability at least $1-\mee^{-\tau}$, it holds that
\begin{equation}
\label{eqn:err}
\begin{aligned}
    \KL ( p_{\underline{T}} \| \widehat{q}_{\mean{T}-\underline{T}} ) \le~& \mee^{-2\mean{T}} \KL ( p_0 \| \GN(0,I) ) + C D (r+1)^{\nu} \mee^{-c(r+1)} \\
    &+ C_1 b N^{-2\gamma/ (\deff + 2\gamma )} (\log N)^{C_2} + \Brac{ \tau + \log b } b N^{-1} (\log N)^{C_2}.
\end{aligned}
\end{equation}
Here $C,c,C_1,C_2$ are dimension-independent constants.
\end{thm}

See \Cref{app:samp_comp} for the proof. There are three sources of error in \eqref{eqn:err}:
\begin{itemize}
    \item Approximation error of $p_{\mean{T}}$, which decays exponentially in terminal time $\mean{T}$;
    \item Localization error of the score function, which decays exponentially in localization radius $r$;
    \item Statistical error, which decays polynomially in $N$, with statistical rate $\frac{2\gamma}{\deff+2\gamma}$.
\end{itemize}
The last term in \eqref{eqn:err} is a high-probability term of order $N^{-1}$, and is therefore negligible compared to the statistical error in the large-$N$ regime.

\begin{rem}
(1) Compared to the vanilla method, the localized diffusion models achieve a much faster statistical rate $\frac{2\gamma}{\deff+2\gamma} \gg \frac{2\gamma}{D+2\gamma}$, and thus potentially mitigate the curse of dimensionality.

(2) \eqref{eqn:err} indicates a trade-off in the choice of localization radius $r$. A smaller $r$ leads to smaller statistical error but induces larger localization error. Note $\deff \asymp r^{\nu}$ (see \Cref{def:loc_graph}), so that the optimal choice is $r^* = \mcO( (\log N)^{\frac{1}{\nu+1}} )$.
Let $\deff^*$ be the effective dimension with $r=r^*$.
When $\log N \ll D^{\frac{\nu+1}{\nu}}$, one can show that the overall error is greatly reduced compared to the usual statistical error:
\begin{equation}
    \mee^{-c\deff^*} + N^{-\frac{2\gamma}{\deff^*+2\gamma}} \ll N^{-\frac{2\gamma}{D+2\gamma}}.
\end{equation}
This is the practical scenario, as one cannot obtain a sample size $N$ exponential in $D$.

(3) We compare the sampled distribution to the early-stopped distribution $p_{\underline{T}}$ by convention. In fact, the early-stopping error can be controlled straightforwardly in Wasserstein distance. For instance, by Lemma 3 in \cite{chen2023sampling}, it holds that $\sfW_2^2(p,p_{\underline{T}}) \lesssim D \underline{T}$.
Thus, the overall error is
\begin{equation}
    \sfW_2^2(p,\widehat{q}_{\mean{T}-\underline{T}})
    \lesssim \sfW_2^2(p,p_{\underline{T}}) + \sfW_2^2(p_{\underline{T}},\widehat{q}_{\mean{T}-\underline{T}}) \lesssim D N^{-\alpha} + \KL ( p_{\underline{T}} \| \widehat{q}_{\mean{T}-\underline{T}} ) .
\end{equation}
Here the second inequality uses Talagrand's inequality.
The early-stopping error does not deteriorate the convergence rate if we take $\alpha \ge \frac{1}{2}$.
\end{rem}

\subsection{Localized diffusion models for mixture distributions}
\label{sec:ldm_mix}

In this section, we consider exploiting locality structure in diffusion models for mixture distributions.
Suppose the target distribution $p_0$ is a mixture of $K$ localized distributions $\{p_0^{(k)}\}_{k=1}^K$:
\begin{equation}
\label{eqn:mix_dist}
    p_0(x) = \sum_{k=1}^K w_k p_0^{(k)}(x), \quad w_k > 0,\quad \sum_{k=1}^K w_k = 1.
\end{equation}
Such mixture structure is common in practice, and can be used to model multi-modal data distributions that go beyond the log-concave setting of \Cref{thm:wLocDist}.

\subsubsection{Localized score approximation}
\label{sec:LocScoreMix}

Let $p_t$ be the marginal distribution at time $t$ of the forward process \eqref{eqn:OU} with initial distribution $p_0$.
Direct computation shows that the score function of $p_t$ can be written as
\begin{equation}
\label{eqn:score_decomp}
    \nabla \log p_t(x) = \sum_{k=1}^K \lambda_t^{(k)}(x) \nabla \log p_t^{(k)}(x), \quad
    \lambda_t^{(k)}(x) := \frac{w_k p_t^{(k)}(x)}{\sum_{\ell=1}^K w_\ell p_t^{(\ell)}(x)}
\end{equation}
where $p_t^{(k)}$ is the marginal at time $t$ of the forward process with initial distribution $p_0^{(k)}$, and $\lambda_t^{(k)}(x)$ is the posterior probability of the $k$-th component given $x$.

If each component $p_0^{(k)}$ satisfies the assumptions in \Cref{thm:wLocDist}, then $p_t^{(k)}$ is approximately localized, and thus $s_j^{(k)} := \nabla_j \log p_t^{(k)}(x)$ can be approximated as in \eqref{eqn:ApproxScore}.
Thus, if $\{\lambda_t^{(k)}(x)\}_{k=1}^K$ can be approximated by $\{\widehat{\lambda}_t^{(k)}(x)\}_{k=1}^K$, then one can approximate $s_j(x,t) := \nabla_j \log p_t(x)$ by
\begin{equation}
\label{eqn:ScoreDecomp}
    s_{\theta,j} (x,t) \approx \sum_{k=1}^K \widehat{\lambda}_t^{(k)}(x) u_{\theta,j}^{(k)}(x_{\mcN_j^r},t).
\end{equation}
Although the posterior weights $\{\lambda_t^{(k)}(x)\}_{k=1}^K$ still depend on the full variable $x$, estimating them can be substantially easier in certain regimes.
This problem is closely related to the density-ratio estimation problem \cite{kanamori2010density}.
For example, when there is a sufficient statistic $T(x)$ for the mixture, the posterior weights can be approximated by functions of the sufficient statistic, i.e., $ \widehat{\lambda}_t^{(k)}(x) = f^{(k)}(T(x),t) $ for some functions $f^{(k)}$.
This composite structure can be exploited to design efficient estimators of the posterior weights.
A detailed discussion in the case of Gaussian mixture models is provided in \Cref{app:GMM}.
In the numerical experiments in \Cref{sec:CIR}, we also consider a mixture of stochastic processes.
The trajectory average $\mean{x} = \frac{1}{D} \sum_{i=1}^D x_i$ acts as an approximate sufficient statistic for the mixture, and the posterior probabilities can therefore be efficiently approximated by functions of $\mean{x}$.

\subsubsection{Localized score matching with approximate posterior weights}
\label{sec:LSM_Mix}

Suppose we have an oracle estimator $\widehat{\lambda} := \{\widehat{\lambda}_t^{(k)}(x)\}_{k=1}^K$ of the posterior weights.
We can then apply the localized score matching method using the localized hypothesis space
\begin{equation}
\label{eqn:loc_hyp_spaceMix}
    \msH_r^{\widehat{\lambda}} = \Big\{ s_\theta: \mR^{D+1} \to \mR^D ~\Big|~ s_{\theta,j} (x,t) = \sum_{k=1}^K \widehat{\lambda}_t^{(k)}(x) u_{\theta,j}^{(k)}(x_{\mcN_j^r},t), ~ u_{\theta,j}^{(k)} \in \msU_j, ~ j \in [b],~ k \in [K] \Big\}.
\end{equation}
Denote for simplicity $\bu_{\theta,j} = (u_{\theta,j}^{(1)},\dots,u_{\theta,j}^{(K)}) \in \msU_j^{\otimes K}$.
We can then approximate the score function $s_\theta$ by minimizing the empirical loss $\widehat{\mcL}^N(s_\theta)$ \eqref{eqn:emp_dsm}, which is equivalent to minimizing the following loss for each component $j$:
\begin{equation}
\label{eqn:mix_emp_dsmj}
    \widehat{\mcL}_j^N (\bu_{\theta,j};\widehat{\lambda}) = \frac{1}{N} \sum_{i=1}^N \int_{\underline{T}}^{\mean{T}} \mE_{\epsilon (t) \sim \GN(0,I)} \Rectbrac{ \norm{ \sum_{k=1}^K \widehat{\lambda}_t^{(k)}(x) u_{\theta,j}^{(k)} (\alpha_t X_{\mcN_j^r}^{(i)} + \sigma_t \epsilon_{\mcN_j^r}(t), t) + \sigma_t^{-1} \epsilon_j(t) }^2 } \mdd t.
\end{equation}
Let its empirical minimizer be
\begin{equation}
    \widehat{\bu}_j = (\widehat{u}_j^{(1)},\dots \widehat{u}_j^{(K)}) \in \argmin_{\bu_{\theta,j} \in \msU_j^{\otimes K}} \widehat{\mcL}_j^N (\bu_{\theta,j};\widehat{\lambda}).
\end{equation}
The corresponding learned score function is denoted as
\begin{equation}
\label{eqn:mix_emp_score}
    \widehat{s}_j (x,t) = \sum_{k=1}^K \widehat{\lambda}_t^{(k)}(x) \widehat{u}_j^{(k)}(x_{\mcN_j^r},t).
\end{equation}

\subsubsection{Error analysis}
\label{sec:mix_samp_comp}

As in \Cref{sec:CompErr}, we assume that each component $p_0^{(k)} \in \mcP$ that satisfies \Cref{asm:p_loc,asm:score_approx}.
Additionally, we assume that the posterior weights can be approximated well by some oracle estimator $\widehat{\lambda}$ in the following sense.

\begin{asm}
\label{asm:lambda_approx}
There exists an oracle estimator $\widehat{\lambda} := \{\widehat{\lambda}_t^{(k)}(x)\}_{k=1}^K$ satisfying the feasibility conditions
\begin{equation}
\label{eqn:lambda_feas}
    \widehat{\lambda}_t^{(k)}(x) \ge 0, \quad \sum_{k=1}^K \widehat{\lambda}_t^{(k)}(x) \equiv 1.
\end{equation}
Moreover, given a sample size $N$, one can take $\varepsilon = N^{-2\gamma/(\deff+2\gamma)}$, and for any $\tau \in (0,1)$, with probability at least $1-\mee^{-\tau}$, the posterior weight approximation error satisfies
\begin{equation}
\label{eqn:lambda_approx}
    \Delta_\lambda^N(\tau) := \int_{\underline{T}}^{\mean{T}} \sigma_t^{-2} \E_{x \sim p_t} \Rectbrac{ \Brac{ \sum_{k=1}^K \norme{ \widehat\lambda_t^{(k)}(x) - \lambda_t^{(k)}(x) } }^2 } \mdd t < \infty.
\end{equation}
Here $\underline{T} = \underline{T}^\varepsilon, \mean{T} = \mean{T}^\varepsilon$ defined as in \Cref{asm:score_approx}.
\end{asm}

As an illustrative example, we verify the above assumption in the case of Gaussian mixture models in \Cref{app:GMM}.

\begin{thm}
\label{thm:mix_samp_comp}
Consider the mixture distribution $p_0$ \eqref{eqn:mix_dist}.
Suppose each component $p_0^{(k)} \in \mcP$, where $\mcP$ satisfies \Cref{asm:p_loc,asm:score_approx}.
Assume that the oracle estimator $\widehat{\lambda}$ satisfies \Cref{asm:lambda_approx}.
Given sample size $N \gg 1$, take $\varepsilon = N^{-2\gamma/(\deff+2\gamma)}$, and $ \underline{T} = \underline{T}^\varepsilon,  \mean{T} = \mean{T}^\varepsilon, \msU_j = \msU_j^\varepsilon$ as in \Cref{asm:score_approx}.
Consider the empirical minimizer $\widehat{s}$ \eqref{eqn:mix_emp_score}.
Let $\widehat{q}_{\mean{T}-\underline{T}}$ denote the sampled distribution using $\widehat{s}$.
Then for any $\tau \in (0,1)$, with probability at least $1-2\mee^{-\tau}$, it holds that
\begin{equation}
\label{eqn:mix_samp_comp}
\begin{split}
    \KL ( p_{\underline{T}} \| \widehat{q}_{\mean{T}-\underline{T}} ) \le~& \mee^{-2\mean{T}} \KL ( p_0 \| \GN(0,I) ) + C D (r+1)^{\nu} \mee^{-c(r+1)} \\
    &+ b C' \Brac{ N^{-2\gamma/ (\deff + 2\gamma )} + (\tau+\log b) N^{-1} + \Delta_\lambda^N(\tau) } ( \log N )^{C''}.
\end{split}
\end{equation}
Here $C,c,C',C''$ are dimension-independent constants.
\end{thm}

See \Cref{app:pf_mix_samp_comp} for the proof.

\section{Numerical experiments}
\label{sec:NumExp}

\subsection{Gaussian model}
\label{sec:NumExpGauss}

In this section, we verify the quantitative results obtained above using Gaussian models. First, we use randomly generated Gaussian distributions to show that locality is approximately preserved in the OU process. Second, we consider sampling a discretized OU process and show that a suitable localization radius is important for balancing localization and statistical errors.

\subsubsection{Approximate locality}
\label{sec:AppLoc}

Consider a localized Gaussian distribution $p_0 = \GN(0,C_0)$, where the precision matrix $P_0 := C_0^{-1}$ is a banded matrix such that
\begin{equation}
    P_0(i,j) = 0, \quad \forall |i-j| > r_0.
\end{equation}
We generate random localized precision matrices $P_0$ with different dimensions and bandwidths, by taking $P_0 = LL\matT$, where $L$ is a randomly generated banded lower triangular matrix.
Because the condition number plays an important role in locality, we also record the condition number of the precision matrices.

Consider sampling $p_0$ by diffusion models.
The score function is given by
\begin{equation}
    s(x,t) = \nabla \log p_t(x) = - P_t x, \quad P_t := - \nabla^2 \log p_t = ( \alpha_t^2 C_0 + \sigma_t^2 I )^{-1}.
\end{equation}
Thus, $P_t$ reflects the locality structure of the score function.

\begin{figure}[!t]
\centering
\includegraphics[trim=80 180 70 180, clip, width=6cm]{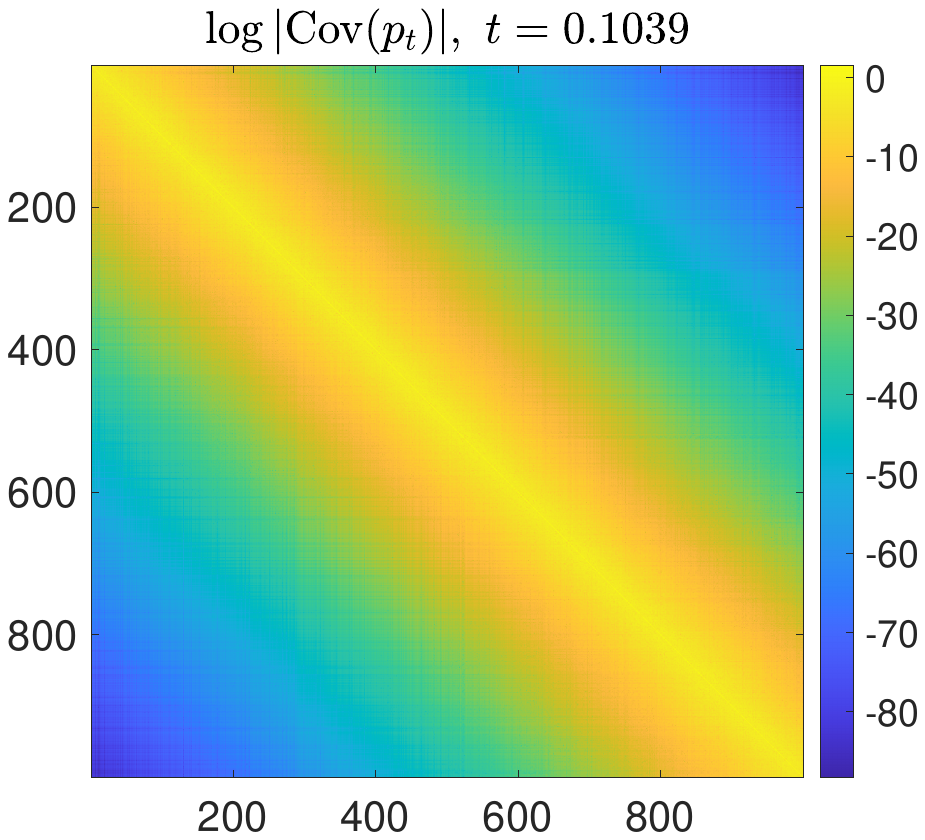} \quad
\includegraphics[width=6cm]{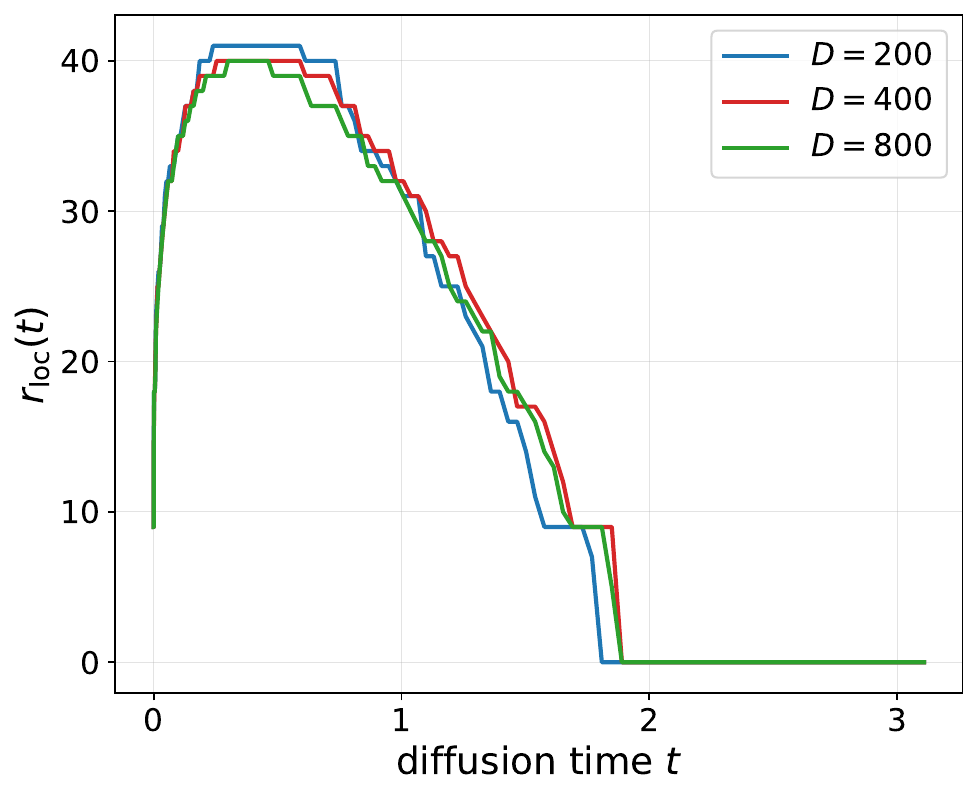} \\
\includegraphics[width=6cm]{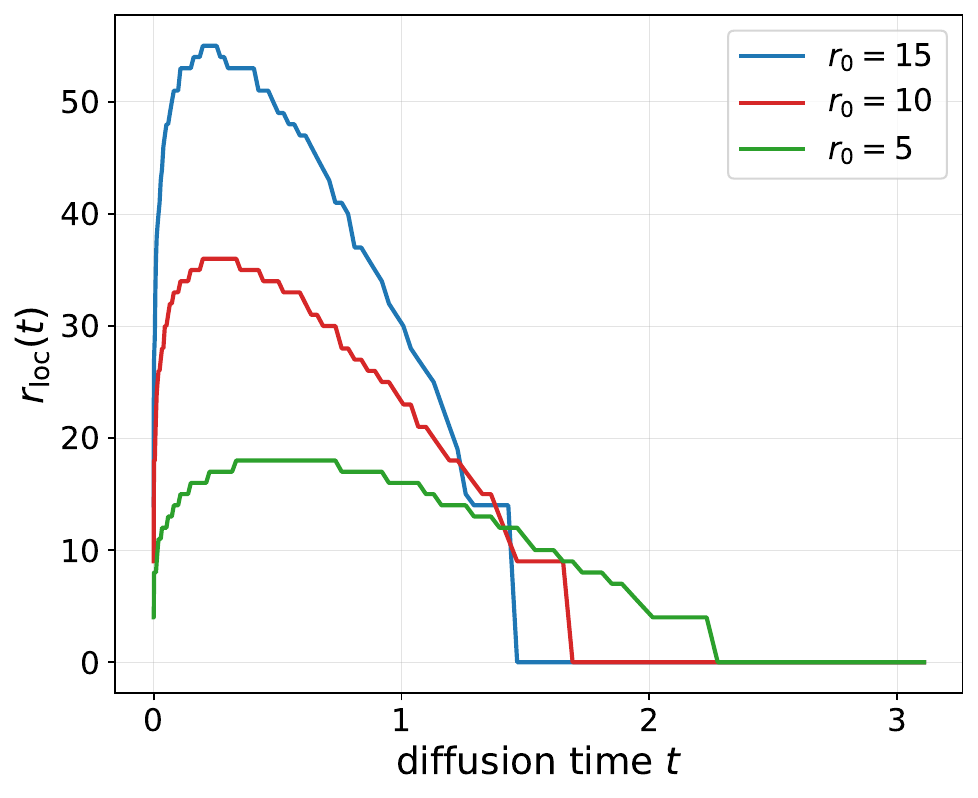} \quad
\includegraphics[width=6cm]{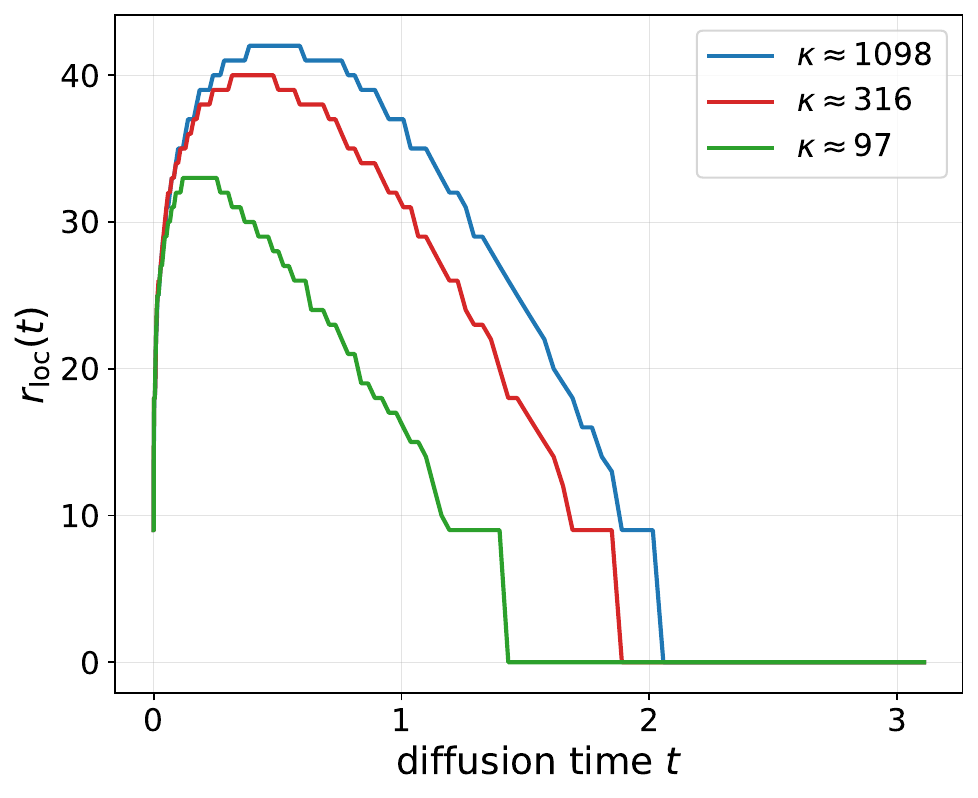}
\caption[Localization radius in Gaussian examples]{Top-left: The precision matrix $P_t$ at $t=0.1039$, plotted on the $\log|P_t|$ scale. The plot shows clear exponential decay of $P_t(i,j)$ in $|i-j|$. The remaining plots show the localization radius $r_{\rm loc}(t)$ \eqref{eqn:EffLoc} for different problem dimensions $D$, precision matrix bandwidths $r_0$, and condition numbers $\kappa$. Top-right: $r_{\rm loc}(t)$ with different dimensions. Here $r_0 = 10$ and the condition numbers are similar ($\kappa \approx 237, 239, 243$). Bottom-left: $r_{\rm loc}(t)$ with different bandwidths. Here $D=800$ and the condition numbers satisfy $\kappa \approx 240, 243, 244$.  Bottom-right: $r_{\rm loc}(t)$ with different condition numbers. Here $D=800$ and $r_0=10$.}
\label{fig:Locality}
\end{figure}

First, we show in the top-left plot in \Cref{fig:Locality} that the $|P_t(i,j)|$ is indeed exponentially decaying with $|i-j|$. Here we take a snapshot of the precision matrix at $t=0.1039$.
We note that this snapshot is not chosen artificially to emphasize the decay, and any snapshot yields similar results.
We then compute the effective localization radius of $P_t$, which is defined as the largest $r$ such that the average of the $r$-th off-diagonal elements is larger than a threshold. More precisely,
\begin{equation}
\label{eqn:EffLoc}
    r_{\rm loc}(t) := \max \Big\{ 0\le r < D ~:~  \frac{1}{D-r} \sum_{1\le i\le D-r } |P_t(i,i+r)| \ge \epsilon \cdot \frac{1}{D} {\rm tr} (P_t) \Big\}.
\end{equation}
We take the threshold rate $\epsilon = 0.001$. We plot the function $r_{\rm loc}(t)$ for different dimensions $D$, bandwidths $r_0$ and condition numbers $\kappa$ in \Cref{fig:Locality}.

From \Cref{fig:Locality}, the effective localization radius $r_{\rm loc}(t)$ first increases with $t$, and then decreases to $1$ when $t$ is large. This occurs because $P_t$ can be regarded as an interpolation between $P_0$ and $P_\infty = I$. This is consistent with the theoretical prediction in \Cref{thm:wLocDist}, where the bound of $\normo{\nabla_{ij}^2 \log p_t}$ first increases with $t$ and then decreases to $0$. Next, the effective localization radius $r_{\rm loc}(t)$ is almost independent of the dimension $D$, consistent with our motivation that the locality structure is approximately preserved with a dimension-independent radius. The effective localization radius $r_{\rm loc}(t)$ is also almost linear in the bandwidth $r_0$ and increases with the condition number $\kappa$.

\subsubsection{Balance of localization error and statistical error}

Consider a discretized OU process $X \in \mR^D~(D=101)$, where $X_n$ follows the dynamics
\begin{equation}
    X_1 \sim \GN(0,1), \quad X_{n+1} = \alpha_h X_n + \sigma_h \xi_n, \quad \xi_n \sim \GN(0,1),
\end{equation}
where $\alpha_h = \mee^{-h}, \sigma_h^2 = 1-\alpha_h^2 ~(h = 0.2)$, and $X_1, \xi_1,\dots,\xi_{100}$ are independent.
Note that $X$ follows a Gaussian distribution
\begin{equation}
    p_0(x) = \GN(x_1;0,1) \prod_{n=1}^{D-1} \GN(x_{n+1}; \alpha_{h} x_n, \sigma_h^2).
\end{equation}
Consider using a diffusion model to sample the above distribution. Since the marginals of the forward process are all Gaussians, the score function is a linear function in $x$. Given a data sample $\{X^{(i)}\}_{i=1}^N$, we estimate the score of the linear form $\widehat{s}(t,x) = - \widehat{P}_t x$ by the loss \eqref{eqn:DSM2}, which admits an explicit solution
\begin{equation}
\label{eqn:EmpP}
    \widehat{P}_t = ( \alpha_t^2 \widehat{C}_0 + \sigma_t^2 I )^{-1},
\end{equation}
where $\widehat{C}_0$ is the empirical covariance of $\{X^{(i)}\}_{i=1}^N$. The non-localized backward process is
\begin{equation}
\label{eqn:Samp_Ref}
    Y (t_{n+1}) = Y (t_n) + \Delta t_n \Brac{ I - 2 \widehat{P}_{\mean{T}-t_n} } Y (t_n) + \sqrt{2 \Delta t_n } \xi_n.
\end{equation}
Here $\widehat{P}_t$ is the estimated optimal precision matrix \eqref{eqn:EmpP}, $\xi_n \sim \GN(0,I), Y_0 \sim \GN(0,I)$, and $\Delta t_n = t_{n+1} - t_n$ is the time step. We use the linear variance schedule $\beta_n = ( \beta_N -\beta_1) \frac{n-1}{N-1} + \beta_1 ~(1 \le n \le N)$ \cite{NEURIPS2020_4c5bcfec}, which corresponds to $\Delta t_n = - \frac{1}{2} \log (1-\beta_{N-n})~(0 \le n \le N-1)$. We take $N=1,000, \beta_1 = 10^{-4}$ and $\beta_N = 0.05$.

A straightforward localization of \eqref{eqn:Samp_Ref} is
\begin{equation}
\label{eqn:Samp_Loc}
\begin{split}
    Y^{{\rm loc},r}(t_{n+1}) = Y^{{\rm loc},r}(t_n) + \Delta t_n &\Brac{ I - 2 \widehat{P}_{\mean{T}-t_n}^{{\rm loc},r} } Y^{{\rm loc},r}(t_n) + \sqrt{2 \Delta t_n } \xi_n, \\
    \widehat{P}_{\mean{T}-t_n}^{{\rm loc},r}(i,j) &:= \widehat{P}_{\mean{T}-t_n}(i,j) \ones_{|i-j|\le r}.
\end{split}
\end{equation}
We will use \eqref{eqn:Samp_Loc} to sample the target distribution with different localization radii $r$, and compare it to the reference sampling process \eqref{eqn:Samp_Ref}. Although the localized score $\widehat{s}^{{\rm loc},r} (t,x) = -\widehat{P}_t^{{\rm loc},r} x$ in \eqref{eqn:Samp_Loc} is not the minimizer of $\widehat{\mcL}_N(s_\theta)$ \eqref{eqn:emp_dsm}, it is very close to the minimizer, and it still yields a good approximation.

\begin{figure}[!t]
\centering
\includegraphics[width=4.4cm]{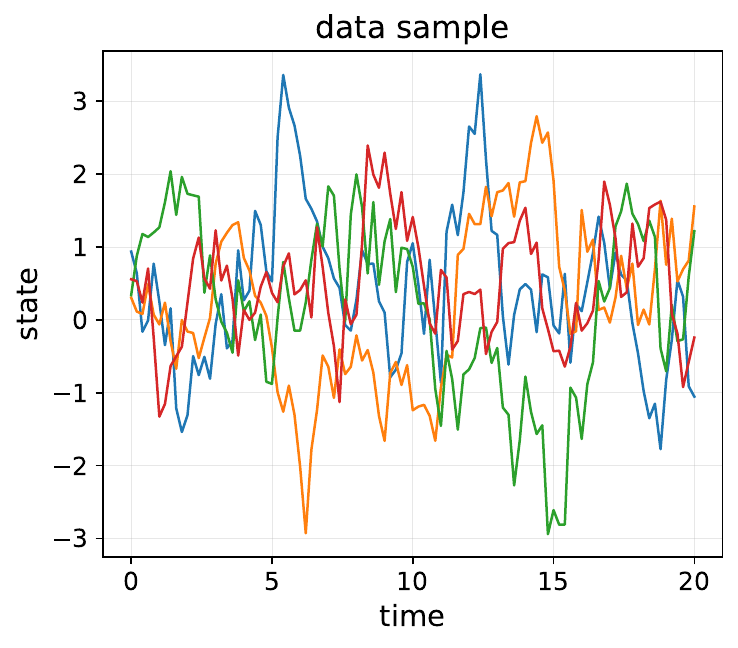}
\includegraphics[width=4.4cm]{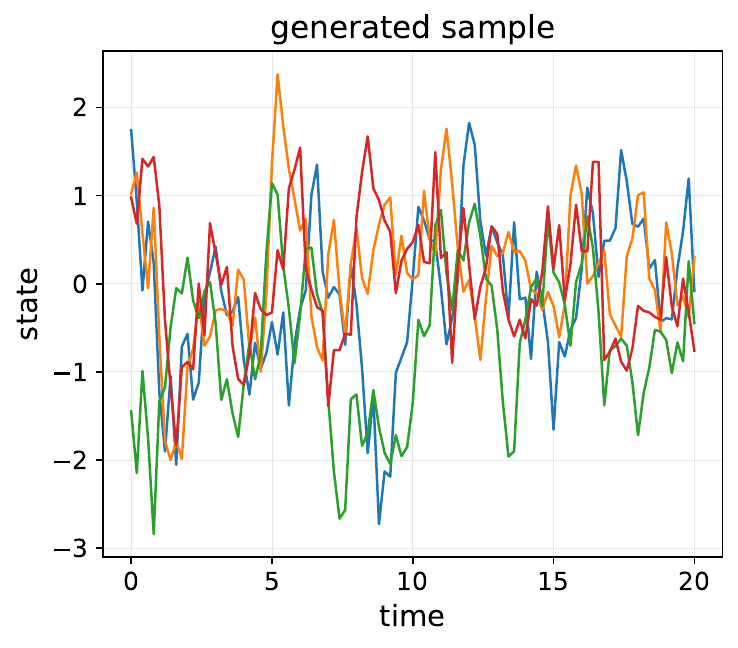}
\includegraphics[width=4.4cm]{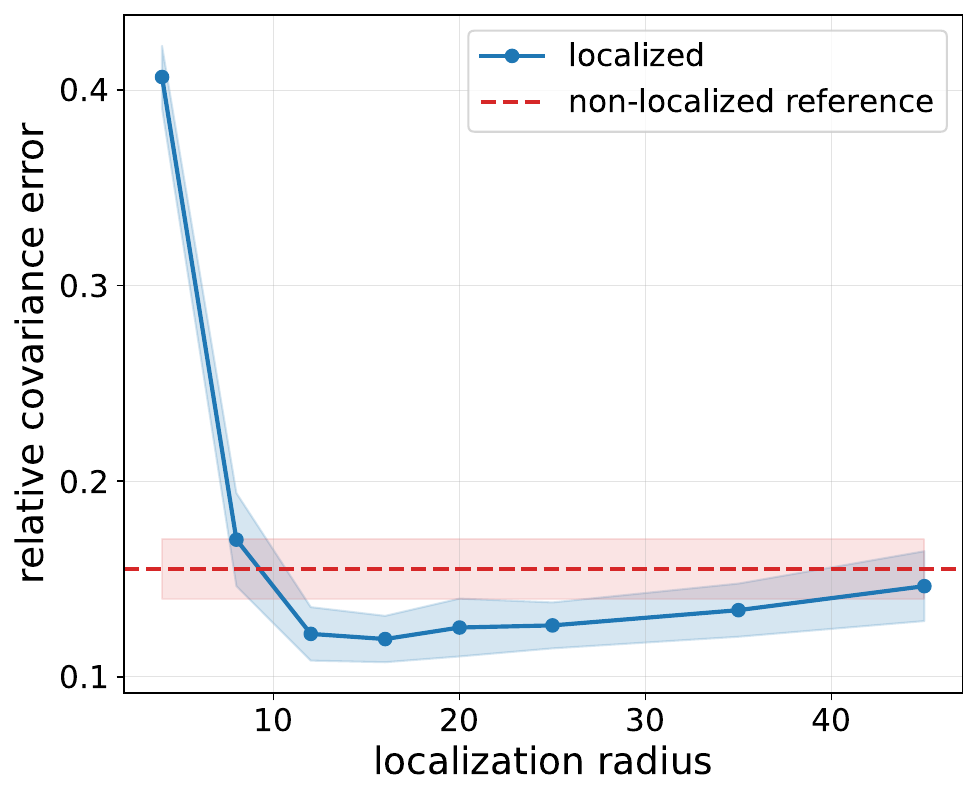} \\
\includegraphics[width=13.3cm]{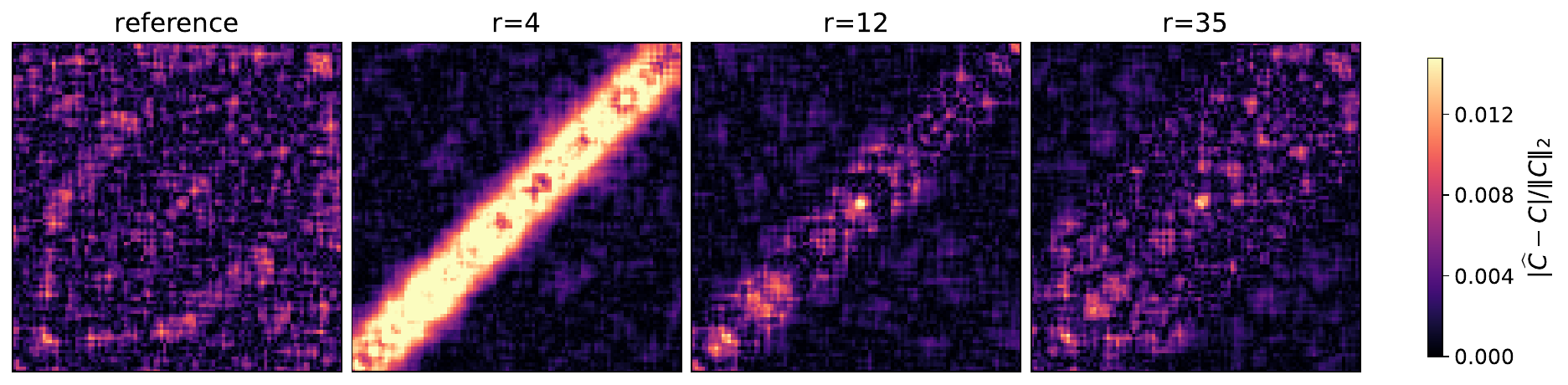}
\caption[Localized sampling for the OU process]{Top left: Trajectories directly sampled from OU process. Top middle: Sampled trajectories using the localized sampling process \eqref{eqn:Samp_Loc} with localization radius $r=12$. Top right: Relative $\ell^2$-error \eqref{eqn:CovErr} of the sample covariance for different localization radii $r$; the reference error is from the non-localized sampling process \eqref{eqn:Samp_Ref}. The shaded area denotes the $1\sigma$ region. Bottom: Entrywise error of the sample covariance with different localization radius $r \in \{4,12,35\}$; the reference error is from the non-localized sampling process \eqref{eqn:Samp_Ref}.}
\label{fig:loc_err}
\end{figure}

As all the distributions involved are Gaussian, we can use the sample covariance to measure the localization error. We use a data sample size $N = 10^3$ and a generated sample size $N_{\rm gen} = 10^4$. The results are shown in \Cref{fig:loc_err}.
In the top-right plot in \Cref{fig:loc_err}, we measure the relative $\ell^2$-error of the sample covariance
\begin{equation}
\label{eqn:CovErr}
    {\rm err} := \frac{ \normo{\widehat{C}-C}_2 }{ \norm{C}_2 },
\end{equation}
where $C = P_0^{-1}$ is the true covariance, $\widehat{C}$ is the sample covariance from \eqref{eqn:Samp_Ref} or \eqref{eqn:Samp_Loc}, and $\norm{\cdot}_2$ is the matrix $2$-norm. The reference error is computed using the sample covariance of the non-localized backward process \eqref{eqn:Samp_Ref}. For each localization radius, we run $30$ independent experiments (with a new data sample) and compute the mean and standard deviation of the relative error. The plot shows that as the localization radius increases, the overall error first decays quickly, and then gradually increases. This is due to the balance between the localization error and the statistical error, as shown more clearly in the bottom plots.

In the bottom row of \Cref{fig:loc_err}, we plot the entrywise error of the sample covariance (normalized by $\norm{C}_2$) for different localization radii $r$.
The localization error dominates when the localization radius is small, and the off-diagonal covariance is not accurately estimated when $r=4$. The off-diagonal part is approximately recovered when $r=12$, and the overall error reaches a minimum. As the localization radius $r$ further increases, the statistical error begins to dominate, leading to spurious long-range correlations as observed in the case $r=35$. This is a well-known phenomenon caused by insufficient sample size \cite{2001MWRv..129..123H}. This suggests that a suitable localization radius is important for balancing localization and statistical errors and reducing the overall error, validating the result in \Cref{thm:samp_comp}.

Additional experiments on a two-dimensional lattice Gaussian model are provided in \Cref{sec:LatticeGMRF}.

\subsection{Cox-Ingersoll-Ross model}
\label{sec:CIR}

We consider the Cox-Ingersoll-Ross (CIR) model \cite{CIR85,CIR85b}
\begin{equation}
\mdd X (t) = 2a (b-X(t)) \mdd t + \sigma \sqrt{X(t)}  \mdd W (t) ,
\label{e.CIR}
\end{equation}
where $W(t)$ is a standard one-dimensional Brownian motion. The CIR model \eqref{e.CIR} possesses a closed-form solution
\begin{equation}
\label{e.CIR_sol}
\frac{X(t)}{c(t)} \sim H(t),\quad c(t)=\frac{\sigma^2}{8a}(1-\mee^{-2a t}),
\end{equation}
where $H(t)$ is a noncentral $\chi$-squared distribution with $8ab /\sigma^2$ degrees of freedom and noncentrality parameter $c(t)^{-1}\mee^{-2a t}X(0)$.

We generate artificial data by integrating the CIR model \eqref{e.CIR} with an Euler--Maruyama discretization and a time step of $h=0.01$, sampling at every $\Delta t=1$ time unit. We determine the score from $N=50$ independent sample trajectories, each of length $D=50$, i.e., each trajectory covers $50$ time units. We choose $a=1.136$, $b=1.1$ and $\sigma=0.4205$.

For the diffusion model, we choose a linear variance schedule with $\beta(t) = (\beta_{\mean{T}}-\beta_0)t/\mean{T} + \beta_0$ with $\mean{T}=0.05$, $\beta_{\mean{T}}=0.5$ and $\beta_0=0.0001$, and we sample the diffusion time $t\in[0,\mean{T}]$ in steps of $0.001$ diffusion time units. The discount factor is given by $\alpha(t) = 1-\beta(t)$. The score is estimated from $5,000$ randomly selected training points with uniformly sampled diffusion times and initial training samples. To learn the score function, we employ a neural network with $3$ hidden layers of sizes $2r+2$, $6$ and $3$, respectively, with an input dimension of $2r+2$ coming from the
localized states of dimension $2r+1$ and the diffusion time. The weights of the neural network are determined by minimizing the MSE loss using an Adam optimizer with a learning rate $\eta=0.00005$.

\begin{figure}[!b]
\begin{center}
\includegraphics[width=4.3cm]{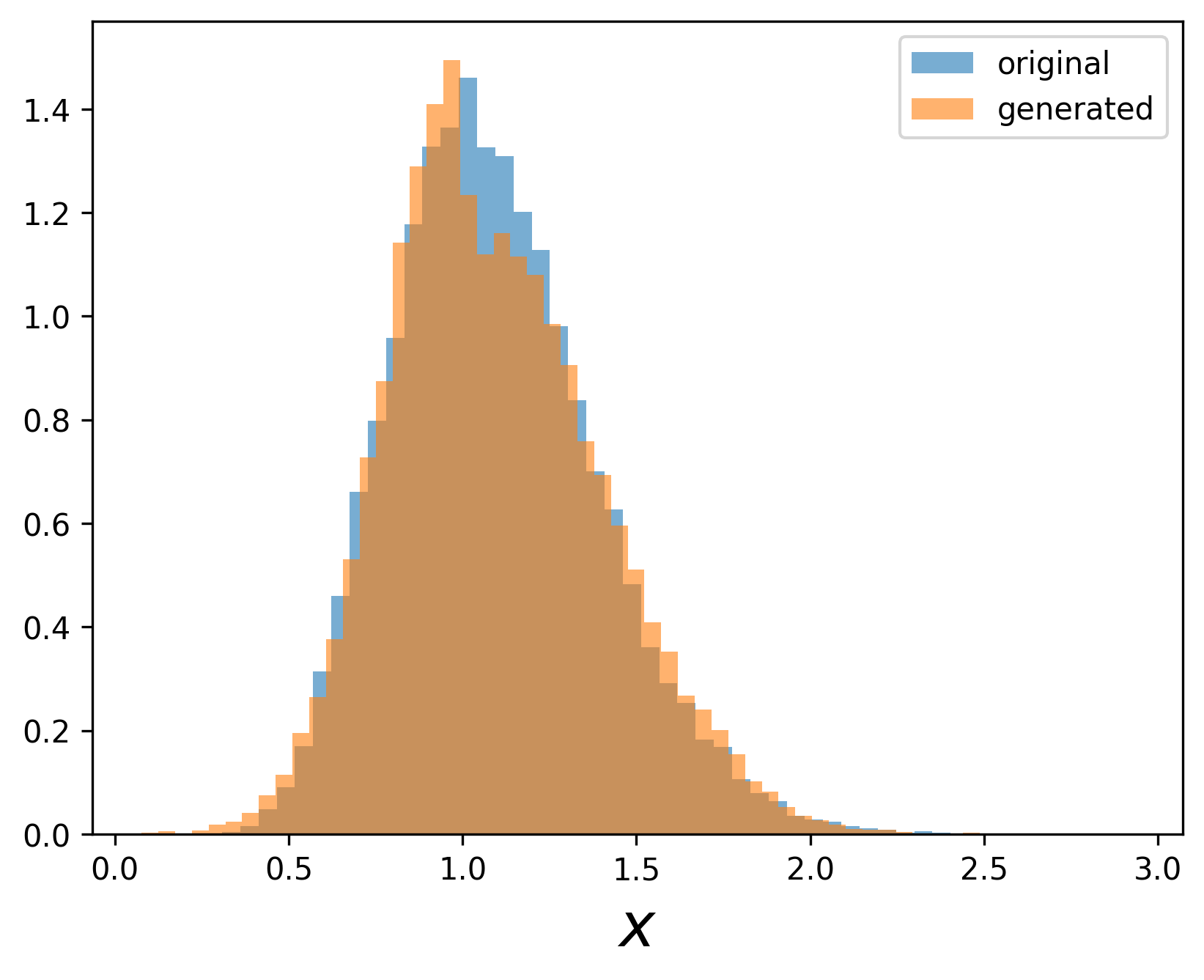}
\includegraphics[width=4.3cm]{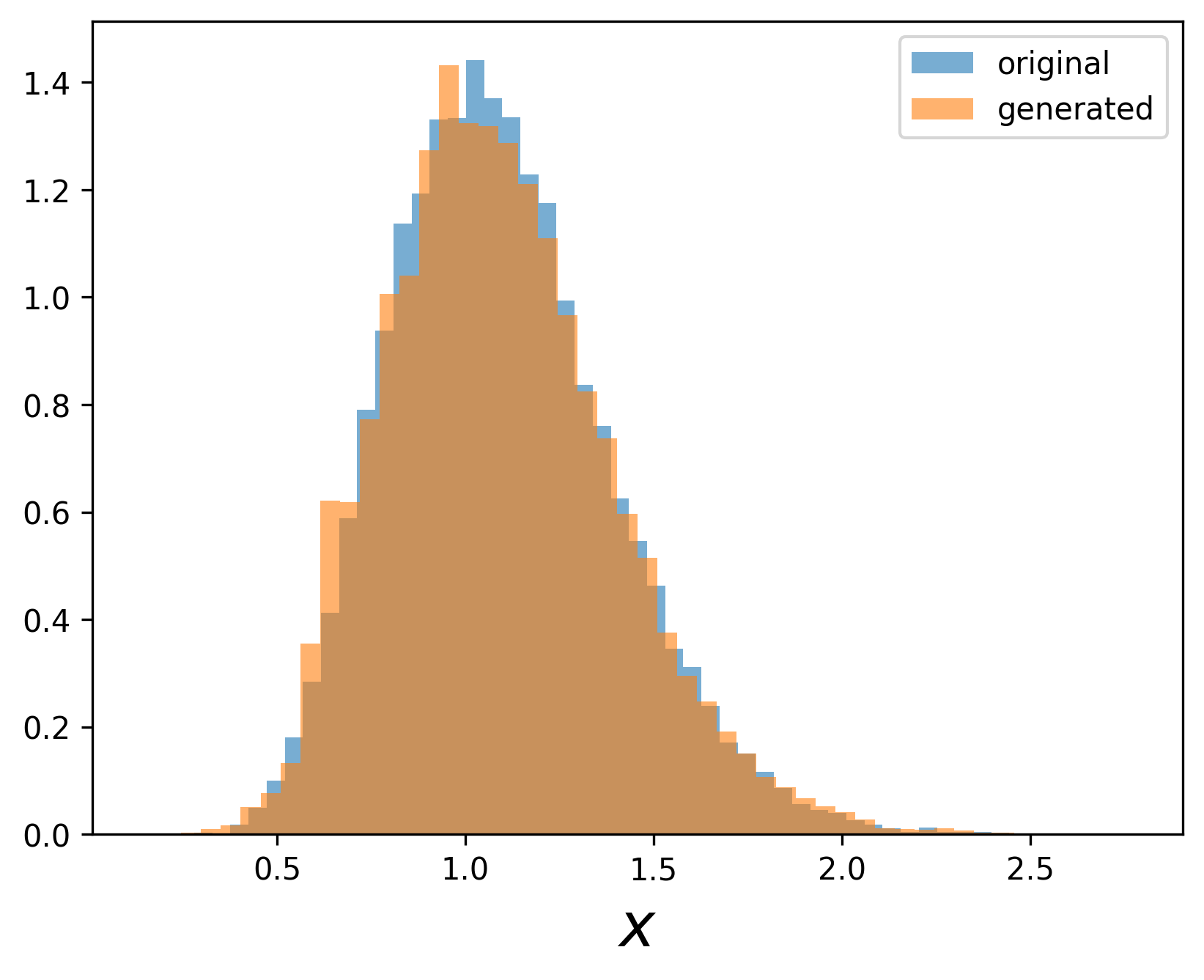}
\includegraphics[width=4.3cm]{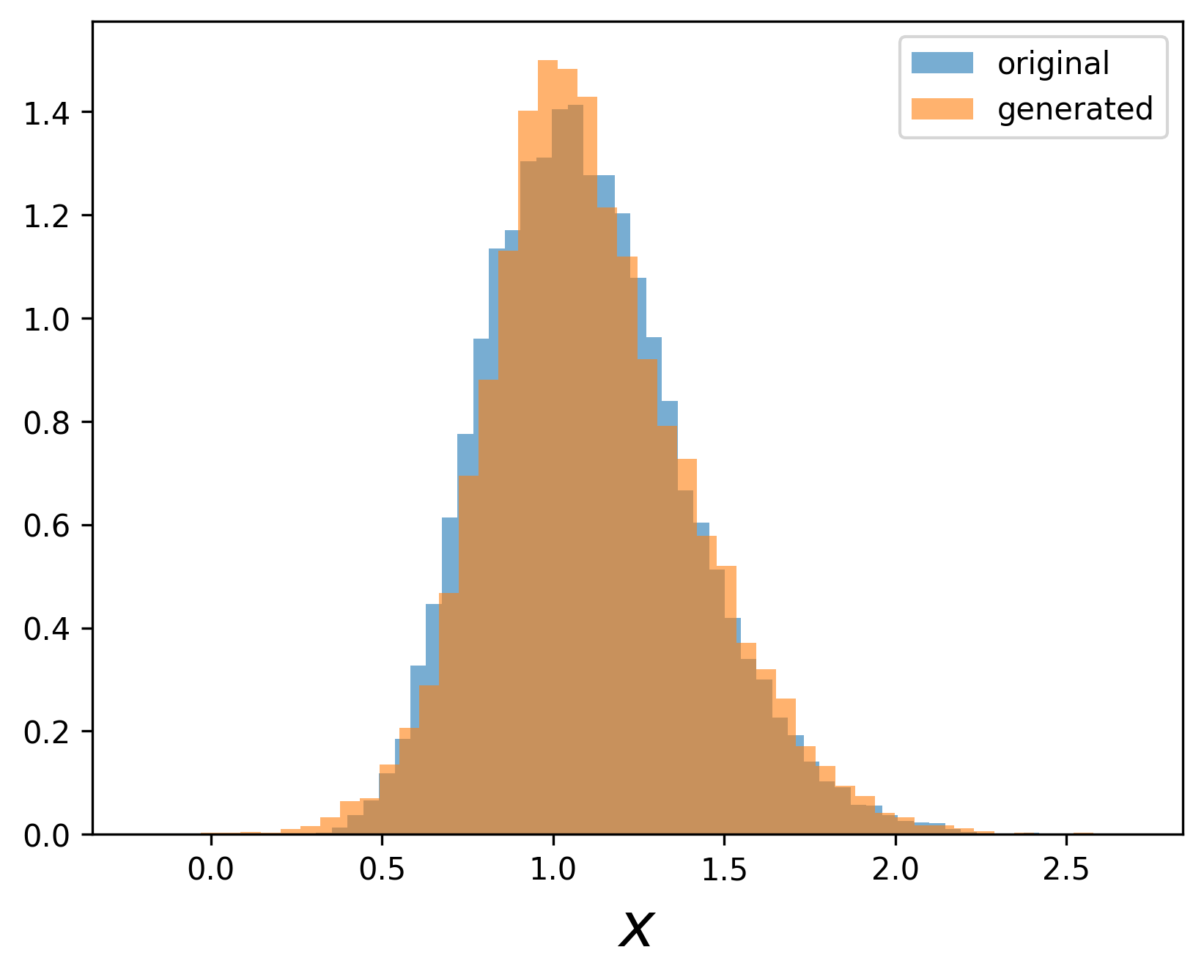}\\
\includegraphics[width=4.3cm]{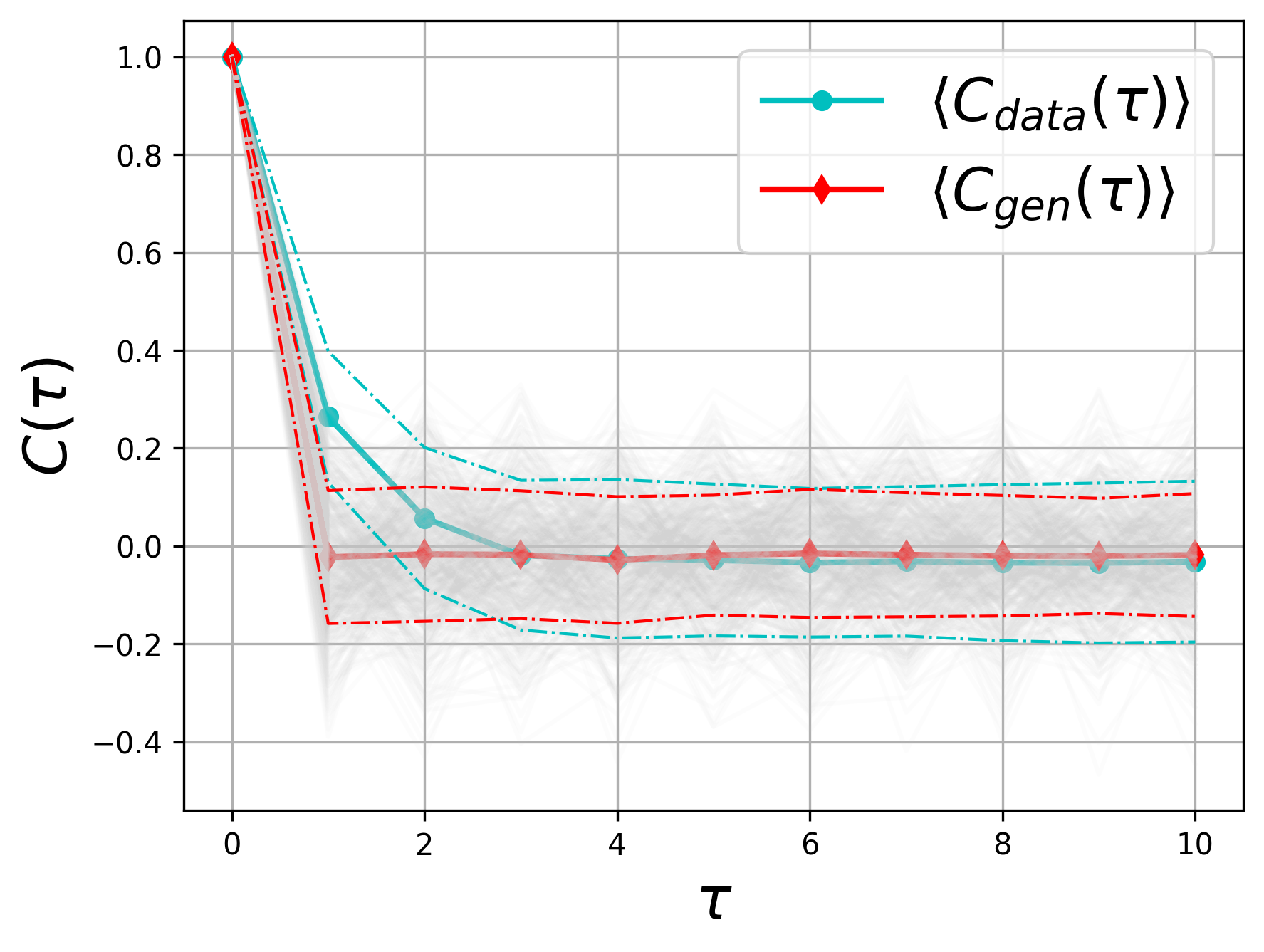}
\includegraphics[width=4.3cm]{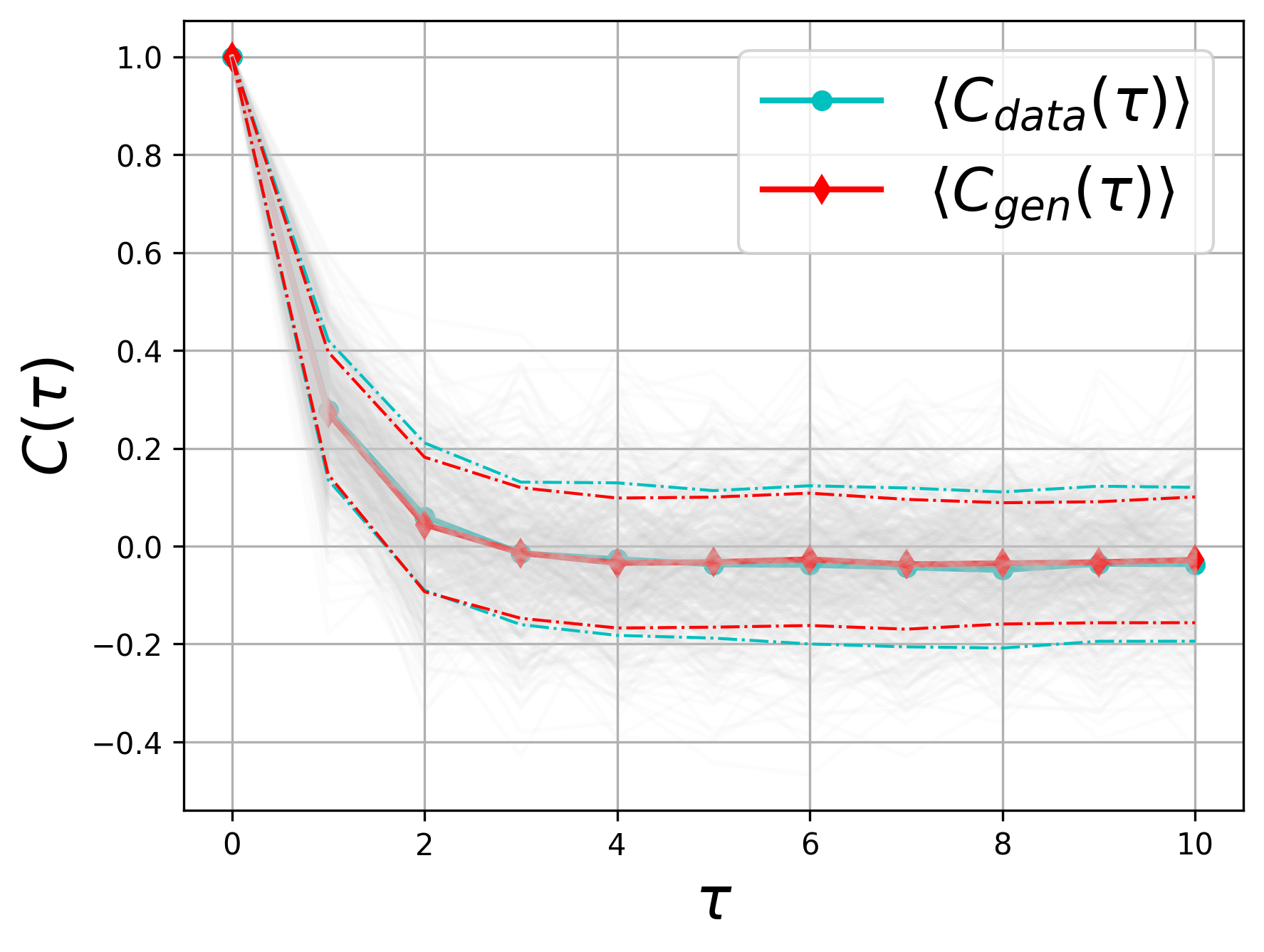}
\includegraphics[width=4.3cm]{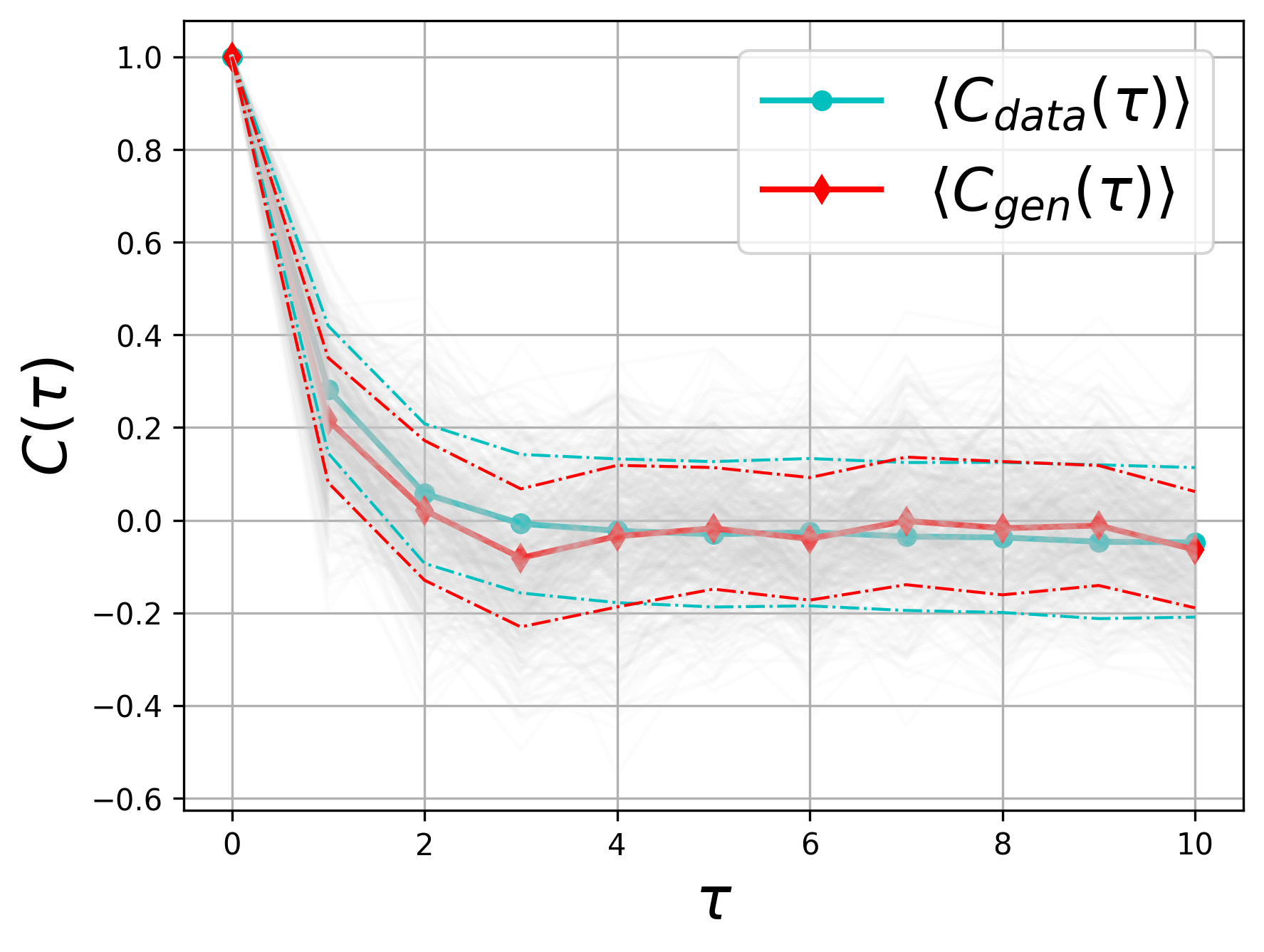}
\caption[CIR experiment comparison]{Comparison of the data obtained from the original CIR model (\ref{e.CIR}) and from the diffusion model for localization radii $r=0$ (left), $r=2$ (middle) and $r=20$ (right). Top: Empirical histogram. Bottom: Auto-correlation function, averaged over all $2,500$ samples. The dashed lines mark deviations of the sample mean that are $1$ standard deviation away.  The light grey lines show the individual auto-correlation functions of the generated data.}
\label{fig:compCIR}
\end{center}
\end{figure}

We show in \Cref{fig:compCIR} a comparison of the empirical histograms and the auto-correlation functions of the training data and the data generated by the diffusion model. The histograms are produced from $5,000$ training and generated time series. The auto-correlation function $\langle C(\tau)\rangle$ is computed as an ensemble average over the samples. When the localization radius is chosen too small with $r=0$, i.e., assuming a $\delta$-correlated process, the auto-correlation function rapidly decays because the localized diffusion models have no information about the correlations present in the data. Interestingly, the empirical histogram is relatively well approximated even with $r=0$. On the other extreme, for large localization radius $r=20$, the number of independent training samples with $N=50$ is not sufficiently large to generate $D=50$-dimensional samples, and the auto-correlation function exhibits an increased variance. We find that a localization radius of $r=2$ yields excellent agreement of the histogram and the auto-correlation function. Varying the localization radius from $r=2$ to $r=8$ yields similar results.

For training, we estimate the score function at index $i$ for $i=r+1,\ldots, D-r$ from the localized state $(x_r)_i = [x_{i-r},\ldots,x_i,\ldots, x_{i+r}]\in \R^{2r+1}$.
Due to stationarity of the process, each component of the score function $s_i((x_r)_i)$ will be the same except near the boundaries, i.e., $i\leq r$ or $i>D-r$.
This allows us to train a single score function which takes a $(2r+2)$-dimensional input ($2r+1$ for the localized state and $1$ for the diffusion time) to generate a $1$-dimensional output of the score function at location $r < i \le D-r$.
To deal with the boundaries of the time series for $i=1,\dots ,r$ and $i=D-r+1,\ldots,D$, we pad with the time series $x$, reflected around $i$.
During the training process, we employ independent noise for each localized region. We have checked that the results do not change if the noise in the diffusion model is kept constant for each local input or varied when cycling through the localized regions.

\begin{figure}[!b]
\begin{center}
\includegraphics[width=0.4\textwidth]{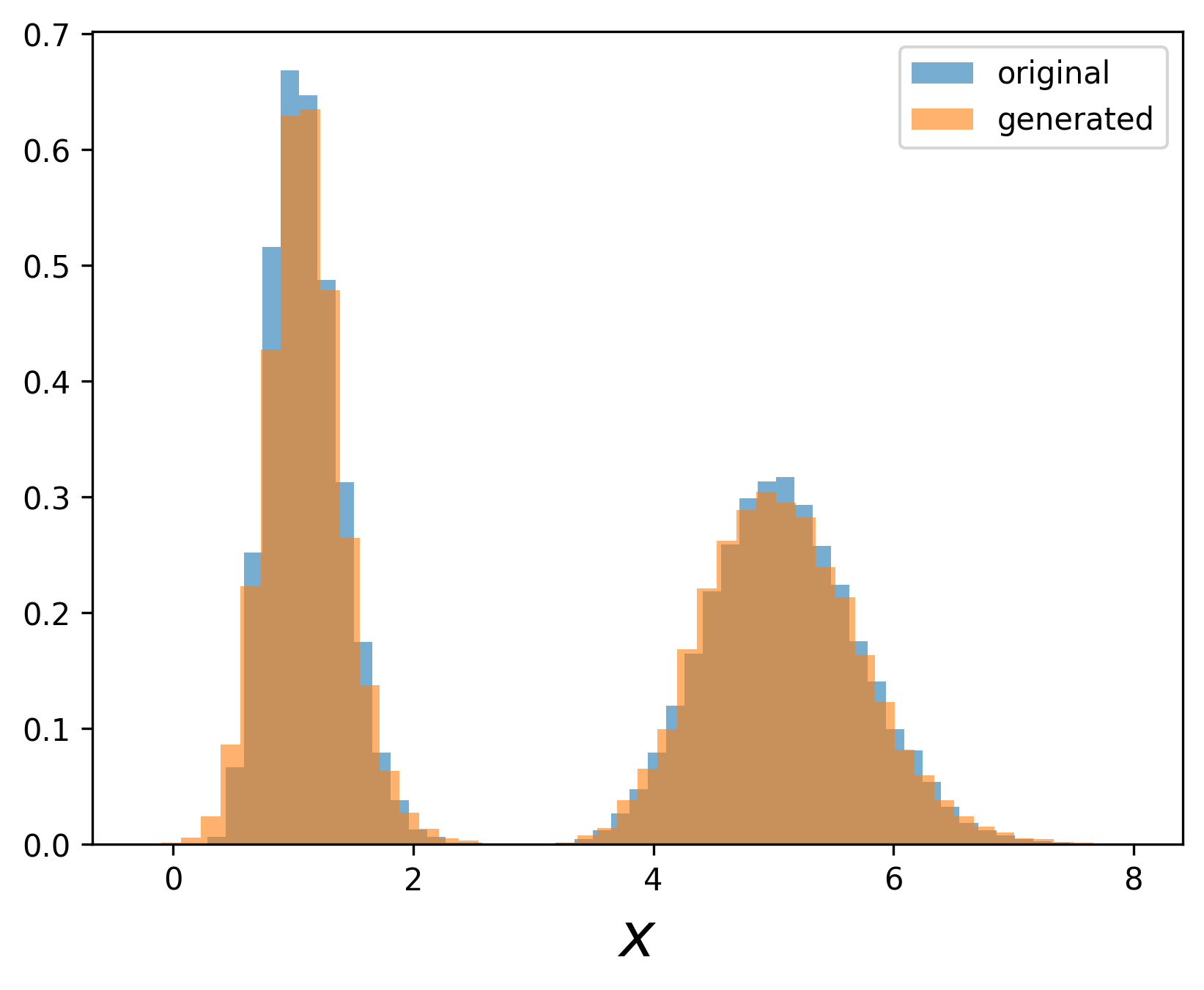} \quad
\includegraphics[width=0.4\textwidth]{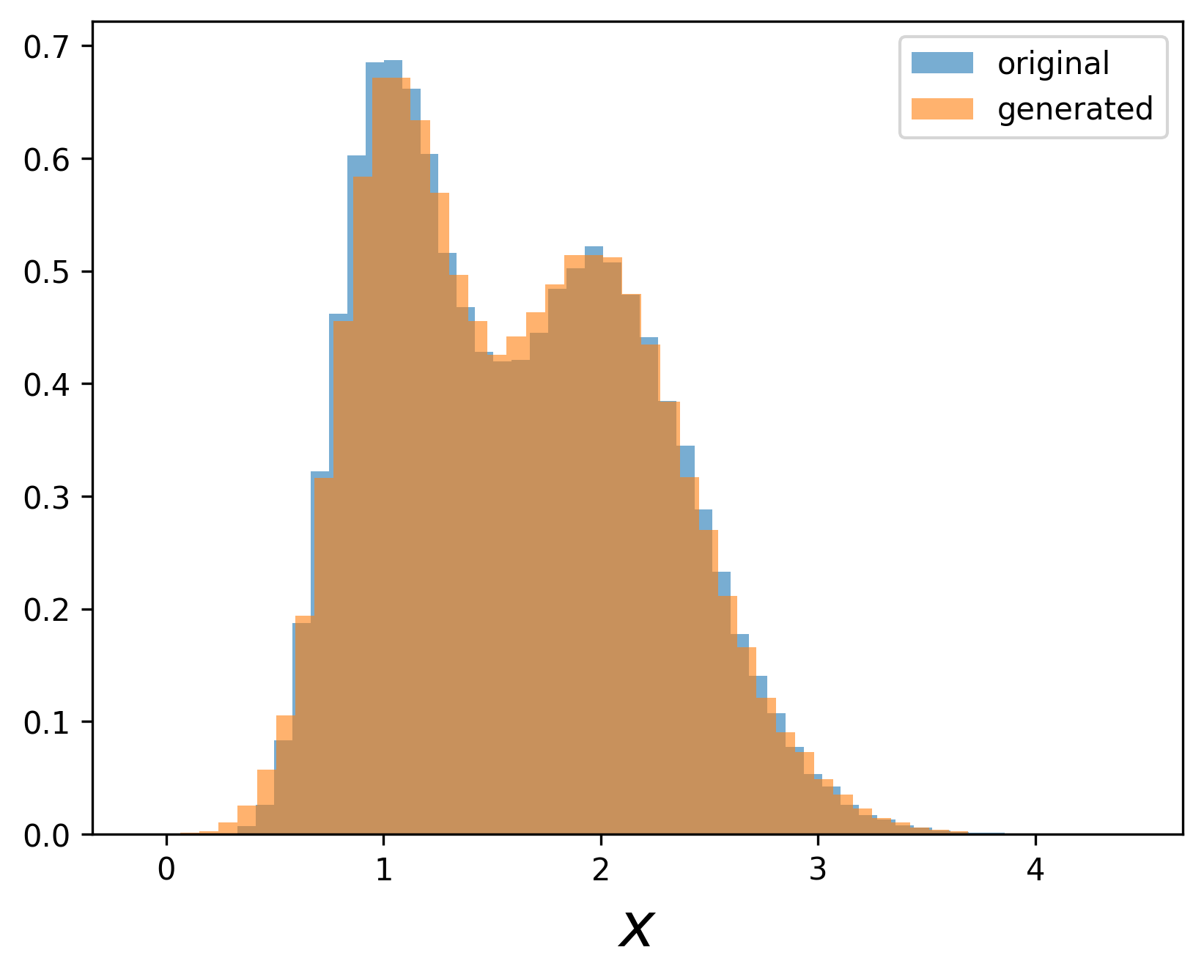}
\caption[Bimodal CIR mixture comparison]{Comparison of the data obtained from a bimodal mixture model of the original CIR model (\ref{e.CIR}) and from the diffusion model with localization radius $r=2$. Left: mixture model using $b=1.1$ and $b=5.1$, and a localization radius of $r=2$. Right: mixture model using $b=1.1$ and $b=2.1$, and a localization radius of $r=4$. }
\label{fig:compCIR_bimodal}
\end{center}
\end{figure}

To further test localization on more complex non-unimodal distributions, we consider a mixture model where we train two sets of trajectories, each of the same cardinality $N=500$, and each consisting of independent trajectory segments of the same CIR model considered above, but differing in the value of the parameter $b$. We consider two cases: in the first, the two distributions are well separated with $b=1.1$ and $b=5.1$, respectively. In the second, the separation is less pronounced with $b=1.1$ and $b=2.1$, respectively. When employing the localization strategy to this bimodal distribution, the diffusion model generates chimera trajectories where parts are generated from one of the two models and other parts from the other model. Moreover, for two well separated distributions, a considerable amount of the generated samples fall into the range between the two peaks that does not lie in the effective support of the original data distribution. The localization prevents the diffusion model from separating the two modes, especially when they merge during the forward noising process. To mitigate this effective averaging of the diffusion model, we augment the input for training the network to determine the score function by the mean of each sample trajectory, which serves as a proxy for a label of the mode. Note that we only use information contained in each data sample and do not add information about which mode the data were drawn from. This increases the input dimension from $2r+2$ to $2r+3$. This again allows reliable generation using small localization radii, as seen in \Cref{fig:compCIR_bimodal}. For distributions with less well separated modes, a comparably larger localization radius is required.

\section{Conclusions}

In this work, we study how locality structure can be exploited in diffusion models to sample high-dimensional distributions.
We show that the locality structure is approximately preserved in the forward diffusion process, which guarantees that localization error decays exponentially in the localization radius.
We propose the localized diffusion model, where we learn the score function within a localized hypothesis space by optimizing a localized score matching loss.
We show that the localized diffusion model avoids the curse of dimensionality, and the rate of the statistical error depends on the effective dimension rather than the ambient dimension.
Through both theoretical analysis and numerical experiments, we demonstrate that a suitable localization radius can balance the localization and statistical error to reduce the overall error.
This validates the effectiveness of the localization method in diffusion models for localized distributions.

However, several interesting questions remain open.
First, the locality structure should not rely on the log-concavity of the distributions, and it would be interesting to extend the theoretical results to non-log-concave distributions.
Second, designing the localized hypothesis space requires prior knowledge of the locality structure. Although it can be learned by many existing methods, it would be interesting to investigate how to combine them, or even learn the locality structure adaptively in the diffusion model.
We leave these questions for future work.

\section*{Acknowledgments}

GAG thanks Yuguang Hu and Xiyu Wang for valuable discussions on the implementation of diffusion models.

\section*{Funding}
The work of SR has been funded by Deutsche Forschungsgemeinschaft (DFG) - Project-ID 318763901 - SFB1294.
GAG acknowledges funding from the Australian Research Council, grant DP220100931.
The work of SL is partially supported by NUS Overseas Research Immersion Award (ORIA).
The work of XTT and SL is supported by Singapore MOE grant A-8002956-00-00.
YM acknowledges support from the US Department of Energy, Office of Advanced Scientific Computing Research, under grant numbers DE-SC0026245 and DE-SC0023187, and from the Air Force Office of Scientific Research under award number FA9550-20-1-0397.

\appendix

\section{Additional remarks}

\subsection{Manifold structure vs.~locality structure}
\label{app:ManiVsLoc}

We compare the scaling of the condition number $\kappa$ in two different regimes, which are typical of distributions with manifold structure and locality structure, respectively.
Consider a simplified Ginzburg--Landau model for $x\in \mR^D$ with density
\begin{equation}
    p(x) \propto \exp \Brac{ - \frac{1}{2} \sum_{i=1}^D x_i^2 - \frac{\gamma}{2} \sum_{i=1}^{D-1} \Brac{ \frac{ x_{i+1} - x_i }{h} }^2 },
\end{equation}
Here the first term describes the potential energy, and the second term describes the interaction energy with mesh size $h>0$ and interaction strength $\gamma > 0$.
Note that
\begin{equation}
    - \nabla^2 \log p(x) = I + \frac{\gamma}{h^2} \begin{bmatrix}
        1 & -1 & 0 & \cdots & 0 \\
        -1 & 2 & -1 & \cdots & 0 \\
        0 & -1 & 2 & \cdots & 0 \\
        \vdots & \vdots & \vdots & \ddots & \vdots \\
        0 & 0 & 0 & \cdots & 1
    \end{bmatrix}.
\end{equation}
Its eigenvalues are $\lambda_k = 1 + 4\gamma h^{-2} \sin^2 \frac{(k-1)\pi}{2D}$ ($k=1,2,\dots,D$).
When $D \gg 1$, it holds that
\begin{equation}
    \lambda_{\min} = \lambda_1 = 1, \quad \lambda_{\max} = \lambda_D \asymp 1 + 4 \gamma h^{-2}.
\end{equation}
Different scaling of $h$ with $D$ leads to different types of high-dimensional distributions, reflected in how the condition number $\kappa := \lambda_{\max}/\lambda_{\min}$ scales with $D$.

\begin{itemize}
    \item Manifold type. Suppose $x \in \mR^D$ is a discretization of a field on a \emph{fixed} spatial domain $[0,1]$, with grid points $s_k = kh$ ($k=0,1,\dots,D-1$) and $h = (D-1)^{-1}$.
    As $D$ increases, the domain remains fixed while the mesh becomes finer, so nearby components of $x$ become increasingly strongly correlated.
    In this regime, the condition number grows with $D$, which is typical of distributions exhibiting a manifold-type structure:
    \begin{equation}
        \kappa = \frac{\lambda_{\max}}{\lambda_{\min}} \asymp 1 + 4 \gamma h^{-2} \asymp 4\gamma D^2, \quad \text{as } D \to \infty.
    \end{equation}
    \item Locality type. Suppose instead $x \in \mR^D$ is a discretization of a field on an \emph{expanding} spatial domain $[0,L]$, with fixed mesh size $h_0>0$, i.e., $s_k = kh_0$ ($k=0,1,\dots,D-1$) and $L = (D-1)h_0$.
    As $D$ increases, the system size grows, but the local interaction range remains unchanged. Thus the components of $x$ remain only locally correlated.
    In this regime, the condition number stays bounded uniformly in $D$, which is typical of distributions with locality structure:
    \begin{equation}
        \kappa = \frac{\lambda_{\max}}{\lambda_{\min}} \asymp 1 + 4 \gamma h_0^{-2}, \quad \text{as } D \to \infty.
    \end{equation}
\end{itemize}

\subsection{Sparse ReLU neural network}
\label{app:NN}

Following \cite{pmlr-v202-oko23a}, we introduce the hyperparameters of a sparse NN as follows:
\begin{itemize}
    \item $\sfL \in \mZ_+$ denotes the depth of the NN.
    \item $\sfW = (\sfw_0,\dots,\sfw_\sfL) \in \mR^{\sfL+1}$ denotes the width vector of the NN.
    \item $\sfS,\sfB$ denote the sparsity and boundedness of the parameters.
\end{itemize}
Consider the ReLU NN class with hyperparameters $(\sfL,\sfW,\sfS,\sfB)$:
\begin{equation}
\label{eqn:sparse_relu}
\begin{split}
    &\NN(\sfL,\sfW,\sfS,\sfB) = \{ u_\theta : \mR^{\sfw_0} \to \mR^{\sfw_\sfL} \mid \theta \in \Pa(\sfL,\sfW,\sfS,\sfB) \}, \\
    \Pa(\sfL,\sfW,\sfS,\sfB) =~& \left\{ \theta = \{W_l,b_l\}_{l=1}^\sfL \mid W_l \in \mR^{\sfw_l\times \sfw_{l-1}}, b_l \in \mR^{\sfw_l}, \norm{\theta}_0 \le \sfS, \norm{\theta}_\infty \le \sfB \right\}, \\
    u_\theta&(x) = W_\sfL \sigma(W_{\sfL-1} \sigma(\cdots \sigma(W_1 x + b_1) \cdots) + b_{\sfL-1}) + b_\sfL,
\end{split}
\end{equation}
where $\sigma(x) = \max\{0,x\}$ is the ReLU activation function (operated element-wise for a vector) and $\norm{\theta}_0, \norm{\theta}_\infty$ are the vector $\ell_0$ and $\ell_\infty$ norms of the parameter $\theta$.

The sparse ReLU NN class can be used as the hypothesis space of the score function.
For instance, one can choose the hypothesis space $\msU_j = \NN(\sfL_j,\sfW_j,\sfS_j,\sfB_j) $ in (3.1) for some hyperparameters $\sfL_j,\sfW_j,\sfS_j,\sfB_j$.
The localized hypothesis space $\msH_r$ (3.1) can then be realized by $b$ sparse ReLU NNs.

It is known that, with a suitable choice of hyperparameters, sparse ReLU neural network classes can approximate smooth functions at the \emph{minimax rate} \citep{MR4134774}; that is, the approximation error attains the rate of the optimal nonparametric estimator (up to logarithmic factors).
For learning the score function $s(x,t)$ in diffusion models, as discussed in Remark~3.4, the minimax rate is $N^{-\frac{2\gamma}{2\gamma+d}}$, where $\gamma$ is the smoothness of the target density $p_0$, and $d$ denotes the dimension of the spatial variable $x$.
The following theorem provides a representative result; see, e.g., Theorem 3.1 in \cite{pmlr-v202-oko23a} and Theorem 5 in \cite{JMLR:v27:25-0121}.

\begin{thm}[Score approximation by sparse ReLU NNs]
\label{thm:sparse_relu_approx}

Let $\mcP$ be a class of probability distributions on $[-1,1]^d$ whose density satisfies the following assumptions \emph{uniformly} over $p \in \mcP$:
\begin{itemize}
    \item Smoothness: there exist $\gamma, K>0$ such that $ \norm{ p }_{C^\gamma([-1,1]^d)} < K$.
    \item Boundedness: there exists $M>0$ such that $ M^{-1} \le p(x) \le M $ for all $x \in [-1,1]^d$.
    \item Boundary regularity: there exist $\delta, L> 0$ such that $ \sup_{n \in \mN} \norm{ p }_{ C^n([-1,1]^d \setminus [-1+\delta,1-\delta]^d) } < L $.
\end{itemize}
Given a sample size $N\ge 2$, choose
\begin{equation}
    \underline{T}_N=N^{-c_{\underline{T}}},
    \qquad
    \mean{T}_N=c_{\mean{T}}\log N,
    \qquad
    I_N=[\underline{T}_N,\mean{T}_N],
\end{equation}
where $c_{\underline{T}},c_{\mean{T}}>0$ are fixed constants.
Let $\msF_N$ be the clipped sparse ReLU class
\begin{equation}
    \msF_N
    =
    \left\{
        f\in\NN(\sfL_N,\sfW_N,\sfS_N,\sfB_N):
        \sfw_0=d+1,\ \sfw_{\sfL_N}=d,\
        \sup_{t\in I_N}\sigma_t\norm{f(\cdot,t)}_\infty\le C_f \log N
    \right\},
\end{equation}
with hyperparameters chosen so that
\begin{equation}
\label{eqn:relu_density_scaling}
\begin{aligned}
    \sfL_N = \mcO ((\log N)^{C_\sfL}),\qquad \norm{\sfW_N}_\infty = \mcO (N(\log N)^{C_{\sfW}}),\\
    \sfS_N = \mcO ( N (\log N)^{C_{\sfS}} ), \qquad
    \sfB_N = \exp ( \mcO((\log N)^{C_{\sfB}}) ).
\end{aligned}
\end{equation}
Then there exist constants $C,c>0$ such that
\begin{equation}
    \sup_{p_0\in\mcP}
    \inf_{f\in\msF_N}
    \int_{I_N}
    \norm{f(\cdot,t)-s(\cdot,t)}_{L^2(p_t)}^2
    \mdd t
    \le
    C N^{-\frac{2\gamma}{d}}(\log N)^c .
\end{equation}
Here $s(x,t) = \nabla \log p_t(x)$ is the score function of the forward process initialized with $p_0$.
\end{thm}

To verify the conditions in Assumption~3.3, it suffices to assume that each marginal distribution $\pi_0^j$ (the marginal of $p_0$ on $x_{\mcN_j^r}$) belongs to a class $\mcP_j$ satisfying the conditions in \Cref{thm:sparse_relu_approx}.
One can take $\msU_j^\varepsilon = \msF_N$ where $d = d_{j,r}$ and $\varepsilon = C N^{-\frac{2\gamma}{d}}(\log N)^c$.
It then remains to verify the entropy condition (3.23), but this is standard \cite{MR4134774,pmlr-v202-oko23a} due to the following entropy bound for sparse ReLU NNs.

\begin{prop}
\label{prop:relu_entropy}
Consider the sparse ReLU NN class \eqref{eqn:sparse_relu}. Then there exists a constant $C>0$ such that for any $t > 0$,
\begin{equation}
    \log \msN( \NN(\sfL,\sfW,\sfS,\sfB) , \norm{\cdot}_\infty, t) \le C\, \sfS_N \Brac{ 1 + \log_+ (\sfL_N \norm{ \sfW_N }_\infty \sfS_N \sfB_N / t ) }.
\end{equation}
\end{prop}

Taking the hyperparameters as in \eqref{eqn:relu_density_scaling}, we have for sufficiently large $N$,
\begin{equation}
    \log \msN( \msF_N , \norm{\cdot}_\infty, t) \le C N (\log N)^{C'} \log_+ (N^c/t) \le C_1 \varepsilon^{-d/(2\gamma)} (\log \varepsilon^{-1})^{C_2} \log_+ (\Gamma_\varepsilon / t).
\end{equation}
This verifies the entropy condition (3.23) in Assumption~3.3.

\subsection{Gaussian mixture model}
\label{app:GMM}

Consider a Gaussian mixture with known $K$,
\begin{equation}
\label{eqn:Gauss_mix}
    p_0 = \sum_{k=1}^K w_k p_0^{(k)},\quad p_0^{(k)} = \GN(m^{(k)}, C^{(k)}),
\end{equation}
with component-dependent means and covariances. Under the forward noising,
\begin{equation}
\begin{aligned}
    p_t^{(k)} &= \GN(\alpha_t m^{(k)}, C_t^{(k)}), \qquad
    C_t^{(k)} := \alpha_t^2 C^{(k)} + \sigma_t^2 I, \\
    \log p_t^{(k)}(x)
    &= - \frac{1}{2}\log\det C_t^{(k)}
    - \frac{1}{2} ( x - \alpha_t m^{(k)} )\matT (C_t^{(k)})^{-1} ( x - \alpha_t m^{(k)} )
    + {\rm const}.
\end{aligned}
\end{equation}
The constant is independent of $k \in [K]$, hence cancels in the posterior weights. Including $\log w_k$, define the quadratic logits
\begin{equation}
\label{eqn:GMM_true_logits}
\begin{aligned}
    T_t^{(k)}(x) &:= x\matT A_t^{(k)}x+(a_t^{(k)})\matT x+b_t^{(k)},\\
    A_t^{(k)} &:= -\frac{1}{2}(C_t^{(k)})^{-1},\qquad
    a_t^{(k)} := \alpha_t(C_t^{(k)})^{-1}m^{(k)},\\
    b_t^{(k)} &:= \log w_k-\frac{1}{2}\log\det C_t^{(k)}
    -\frac{\alpha_t^2}{2}(m^{(k)})\matT(C_t^{(k)})^{-1}m^{(k)}.
\end{aligned}
\end{equation}
Then the posterior weights in (3.29) are
\begin{equation}
    ( \lambda_t^{(1)}(x), \dots, \lambda_t^{(K)}(x) )
    =
    \softmax ( T_t^{(1)}(x), \dots, T_t^{(K)}(x) ).
\end{equation}
Here
\(
    \softmax(x_1,\dots,x_K) := \Brac{ \frac{\exp(x_1)}{\sum_{\ell=1}^K \exp(x_\ell)}, \dots, \frac{\exp(x_K)}{\sum_{\ell=1}^K \exp(x_\ell)} }.
\)
Thus the posterior weight map is a quadratic feature map followed by softmax, and therefore can be estimated efficiently.

\subsubsection{Oracle estimator of posterior weights}

We estimate this composite map by first fitting the mixture. Fix constants $w_{\min},M,c_0,c_1$ such that
\begin{equation}
\label{eqn:GMM_param_bounds}
    w_k\ge w_{\min}>0,\quad \normo{m^{(k)}}\le M,\quad c_0I\preceq C^{(k)}\preceq c_1I,\qquad k\in[K].
\end{equation}
Let $\Sigma=(\Sigma^{(1)},\ldots,\Sigma^{(K)})$ and define the parameter space
\begin{equation}
\label{eqn:GMM_param_space}
\begin{aligned}
    \Theta_{\rm GMM}
    =
    \{(v,\mu,\Sigma):\
    &v_k\ge w_{\min},\ \sum_{k=1}^K v_k=1,\\
    &\normo{\mu^{(k)}}\le 2M,\
    \Sigma^{(k)}=(\Sigma^{(k)})\matT,\
    c_0 I \preceq \Sigma^{(k)} \preceq c_1 I,\ k\in[K]\},
\end{aligned}
\end{equation}
Given independent samples $X^{(1)},\dots,X^{(N)}$ from $p_0$, we take the constrained MLE
\begin{equation}
\label{eqn:GMM_MLE}
    (\widehat v, \widehat \mu,\widehat \Sigma)
    \in
    \argmax_{(v,\mu,\Sigma)\in\Theta_{\rm GMM}}
    \frac{1}{N}
    \sum_{i=1}^{N}
    \log \Brac{ \sum_{k=1}^K v_k \GN(X^{(i)};\mu^{(k)},\Sigma^{(k)}) }.
\end{equation}
The plug-in logits are
\begin{equation}
\label{eqn:GMM_plugin_logits}
\begin{aligned}
    \widehat T_t^{(k)}(x) &:= x\matT\widehat A_t^{(k)}x+(\widehat a_t^{(k)})\matT x+\widehat b_t^{(k)}, \\
    \widehat \Sigma_t^{(k)} &:= \alpha_t^2 \widehat \Sigma^{(k)} + \sigma_t^2 I,\quad \widehat A_t^{(k)} := -\frac{1}{2}(\widehat \Sigma_t^{(k)})^{-1},\quad \widehat a_t^{(k)} := \alpha_t(\widehat \Sigma_t^{(k)})^{-1}\widehat \mu^{(k)},\\
    \widehat b_t^{(k)} &:= \log \widehat v_k-\frac{1}{2}\log\det\widehat \Sigma_t^{(k)}
    -\frac{\alpha_t^2}{2}(\widehat \mu^{(k)})\matT(\widehat \Sigma_t^{(k)})^{-1}\widehat \mu^{(k)}.
\end{aligned}
\end{equation}
The estimated posterior weights are then
\begin{equation}
\label{eqn:GMM_plugin_lambda}
    (\widehat{\lambda}_t^{(1)}(x), \dots, \widehat{\lambda}_t^{(K)}(x) )
    =
    \softmax ( \widehat T_t^{(1)}(x), \dots, \widehat T_t^{(K)}(x) ).
\end{equation}
This automatically satisfies the feasibility condition (3.35).

In the analysis, to avoid the label permutation redundancy, consider the optimal label alignment
\begin{equation}
\label{eqn:GMM_MLE_align}
    \widehat\pi
    \in
    \argmin_{\pi\in {\rm S}_K}
    \left\{
        \sum_{k=1}^K \norme{\log \widehat v_{\pi(k)}-\log w_k}^2
        +
        \sum_{k=1}^K \normo{\widehat \mu^{(\pi(k))}-m^{(k)}}^2
        +
        \sum_{k=1}^K \normo{\widehat \Sigma^{(\pi(k))}-C^{(k)}}_{\rm F}^2
    \right\},
\end{equation}
where ${\rm S}_K$ is the symmetric group of order $K$.
After relabeling the MLE according to $\widehat\pi$, we still denote it by $(\widehat v,\widehat\mu,\widehat\Sigma)$.
We write $p_{v,\mu,\Sigma}$ for the Gaussian mixture with parameters $(v,\mu,\Sigma)$ \eqref{eqn:Gauss_mix}, with the same label-alignment convention.

\begin{prop}
\label{prop:GMM_lambda_approx}
Assume that $p_0$ satisfies the parameter bounds \eqref{eqn:GMM_param_bounds}.
Assume also the usual local nonsingularity condition: there exist $r_0,\kappa>0$ such that, after optimal label alignment as in \eqref{eqn:GMM_MLE_align},
\begin{equation}
\label{eqn:GMM_curvature}
    \KL(p_{w,m,C}\|p_{v,\mu,\Sigma})
    \ge
    \kappa
    \left\{
        \sum_{k=1}^K \norme{\log v_k-\log w_k}^2
        + \sum_{k=1}^K \normo{\mu^{(k)}-m^{(k)}}^2
        + \sum_{k=1}^K \normo{\Sigma^{(k)}-C^{(k)}}_{\rm F}^2
    \right\}
\end{equation}
whenever the expression in braces is at most $r_0^2$, while the KL divergence is bounded away from zero outside this neighborhood.
Let $(\widehat v,\widehat \mu,\widehat \Sigma)$ be the aligned constrained MLE \eqref{eqn:GMM_MLE}, and let $\widehat{\lambda}$ be the composite estimator \eqref{eqn:GMM_plugin_logits}--\eqref{eqn:GMM_plugin_lambda}.
Denote
\(
    N_{\rm GMM} := K-1+KD+K\frac{D(D+1)}{2}.
\)
Let
\begin{equation}
    J_N := \int_{\underline{T}}^{\mean{T}}\sigma_t^{-2}\mdd t,\quad
    \mathfrak m_4 := \sup_{t\in[\underline{T},\mean{T}]}\E_{x\sim p_t}\Brac{1+\norm{x}^2}^2 .
\end{equation}
Then for any $\tau>0$, with probability at least $1-\mee^{-\tau}$, the posterior weight error satisfies
\begin{equation}
\label{eqn:GMM_Delta_lambda_bound}
    \Delta_\lambda^N(\tau) := \int_{\underline{T}}^{\mean{T}} \sigma_t^{-2} \E_{x \sim p_t} \Rectbrac{ \Brac{ \sum_{k=1}^K \norme{ \widehat\lambda_t^{(k)}(x) - \lambda_t^{(k)}(x) } }^2 } \mdd t
    \le C \mathfrak m_4 J_N \frac{N_{\rm GMM}+\tau}{N}.
\end{equation}
Here $C$ depends on $K,w_{\min},M,c_0,c_1,\kappa$ but not on $N$.
Under \eqref{eqn:GMM_param_bounds}, $\mathfrak m_4 \le C(1+D^2+M^4)$.
\end{prop}

\begin{proof}

For regular finite-dimensional mixtures on $\Theta_{\rm GMM}$, the centered log-likelihood class has entropy of order $N_{\rm GMM}\log(A/\epsilon)$ and sub-exponential envelopes. A standard empirical-process bound for the constrained MLE gives, with probability at least $1-\mee^{-\tau}$,
\[
    \KL(p_{w,m,C}\|p_{\widehat v,\widehat \mu,\widehat \Sigma})
    \le
    C\frac{N_{\rm GMM}+\tau}{N}.
\]
The nonsingularity condition \eqref{eqn:GMM_curvature} implies
\begin{equation}
\label{eqn:GMM_param_error}
    \sum_{k=1}^K\norme{\log\widehat v_k-\log w_k}^2
    +
    \sum_{k=1}^K\normo{\widehat \mu^{(k)}-m^{(k)}}^2
    +
    \sum_{k=1}^K\normo{\widehat \Sigma^{(k)}-C^{(k)}}_{\rm F}^2
    \le
    C' \frac{N_{\rm GMM}+\tau}{N}.
\end{equation}
Let $\underline c=\min\{c_0,1\}$ and $\overline c=\max\{c_1,1\}$.
The parameter bounds and the definition of $\Theta_{\rm GMM}$ imply that, for all $k\in[K]$ and $t\in[\underline{T},\mean{T}]$,
\[
    \underline c I\preceq C_t^{(k)}\preceq \overline c I,
    \quad
    \underline c I\preceq \widehat \Sigma_t^{(k)}\preceq \overline c I.
\]
Hence the resolvent identity gives
\[
    \normo{(\widehat \Sigma_t^{(k)})^{-1}-(C_t^{(k)})^{-1}}
    \le
    \normo{(\widehat \Sigma_t^{(k)})^{-1}}\,
    \normo{\widehat \Sigma_t^{(k)}-C_t^{(k)}}\,
    \normo{(C_t^{(k)})^{-1}}
    \le C \normo{\widehat \Sigma^{(k)}-C^{(k)}}_{\rm F},
\]
and the same spectral bounds imply
\[
    \normeo{\log\det\widehat \Sigma_t^{(k)}-\log\det C_t^{(k)}}
    \le C\normo{\widehat \Sigma^{(k)}-C^{(k)}}_{\rm F}.
\]
Set $\rho_\lambda:=\sqrt{(N_{\rm GMM}+\tau)/N}$.
Together with \eqref{eqn:GMM_param_error} and $\norm{m^{(k)}}\le M$, this yields
\[
    \max_k \normo{\widehat A_t^{(k)}-A_t^{(k)}} \le C\rho_\lambda,
    \quad
    \max_k \normo{\widehat a_t^{(k)}-a_t^{(k)}} \le C\rho_\lambda,
    \quad
    \max_k \normeo{\widehat b_t^{(k)}-b_t^{(k)}} \le C\rho_\lambda,
\]
again uniformly in $t$.
Therefore
\[
    \max_k \norme{\widehat T_t^{(k)}(x)-T_t^{(k)}(x)}
    \le C\rho_\lambda(1+\norm{x}^2).
\]
Since softmax is Lipschitz from $\ell^\infty$ to $\ell^1$ with a constant depending only on $K$, we obtain
\[
    \sum_{k=1}^K \norme{\widehat\lambda_t^{(k)}(x)-\lambda_t^{(k)}(x)}
    \le C\rho_\lambda(1+\norm{x}^2).
\]
Square the above inequality, take expectation under $p_t$, and integrate over $[\underline{T},\mean{T}]$ against $\sigma_t^{-2}\mdd t$.
Using the definition of $\mathfrak m_4$, we obtain \eqref{eqn:GMM_Delta_lambda_bound}.
\end{proof}

\section{Supplementary experiment}

\subsection{Two-dimensional lattice Gaussian model}
\label{sec:LatticeGMRF}

We next consider a Gaussian model on a two-dimensional lattice, which provides a simple example with a two-dimensional graph structure.
Let $\Lambda_K = [K]^2$, and let $\sfG$ denote the nearest-neighbor graph on $\Lambda_K$.
We index the vertices of $\sfG$ by $[D]$ where $D = K^2$.
Consider the centered Gaussian distribution
\begin{equation}
\label{eq:LatticeGMRF}
    p_0(x) \propto \exp\Brac{ -\frac{1}{2} x\matT P_0 x},
    \quad
    P_0 = I + \gamma L,
\end{equation}
where $L(i,j) = \textrm{deg}(i) \ind_{i=j} - \ind_{\sfd_{\sfG}(i,j)=1}$ is the graph Laplacian. Here $\textrm{deg}(i)$ is the degree of vertex $i$.
By definition, $\forall \sfd_{\sfG}(i,j)>1, P_0(i,j) = 0$.
Thus $p_0$ is strictly localized with respect to $\sfG$.
Note that $\gamma>0$ controls the strength of the nearest-neighbor interaction, and by adjusting $\gamma$ one can control the condition number of $P_0$.

As in \eqref{eqn:EffLoc2}, we compute the effective localization radius of $P_t =  ( \alpha_t^2 P_0^{-1} + \sigma_t^2 I )^{-1}$:
\begin{equation}
\label{eqn:EffLoc2}
    r_{\rm loc}(t) := \max \Big\{ 0\le r < 2K ~:~  \frac{1}{|\{(i,j): \sfd_{\sfG}(i,j)=r\}|} \sum_{i,j: \sfd_{\sfG}(i,j)=r} |P_t(i,j)| \ge \epsilon \cdot \frac{1}{D} {\rm tr} (P_t) \Big\}.
\end{equation}
We take the threshold rate $\epsilon = 0.001$. We plot the function $r_{\rm loc}(t)$ for different $\gamma$ (thus different condition numbers $\kappa$).
We take $K=20$ so that $D=400$, and $\gamma \in \{0.25,1,4\}$.
The top-left plot in \Cref{fig:Lattice_LocErr} shows that the effective localization radius $r_{\rm loc}(t)$ presents a similar pattern to that in Figure~1, and that $r_{\rm loc}(t)$ increases with the condition number $\kappa$.

\begin{figure}[!t]
\centering
\hspace{-10pt}
\includegraphics[width=4.4cm]{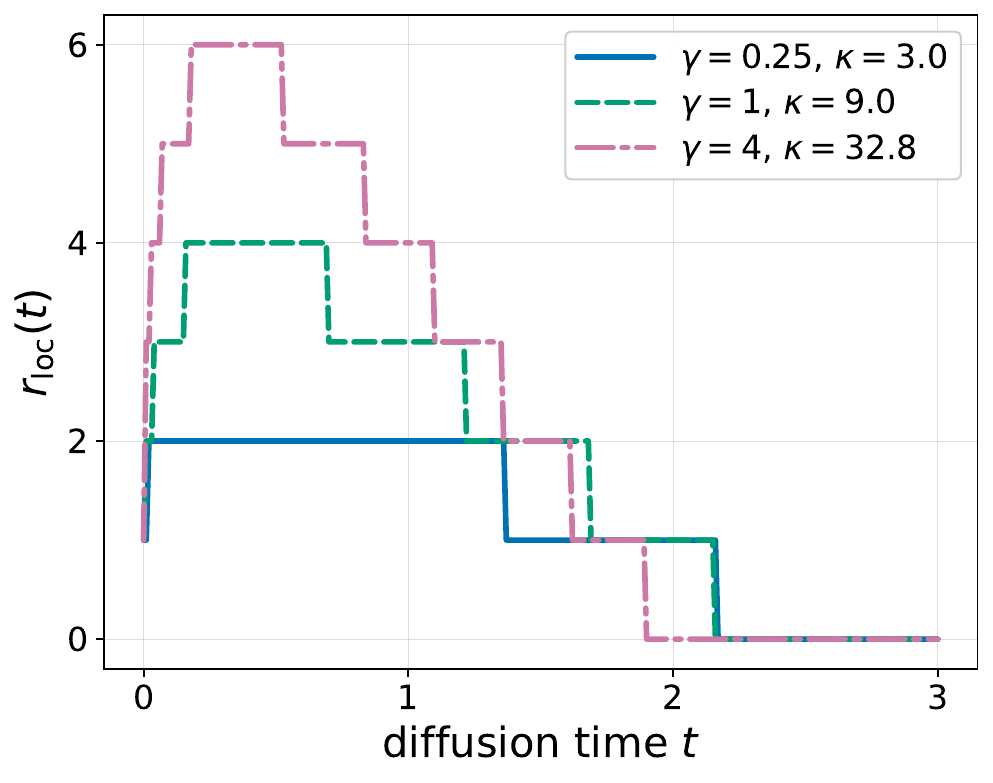}
\includegraphics[width=4.4cm]{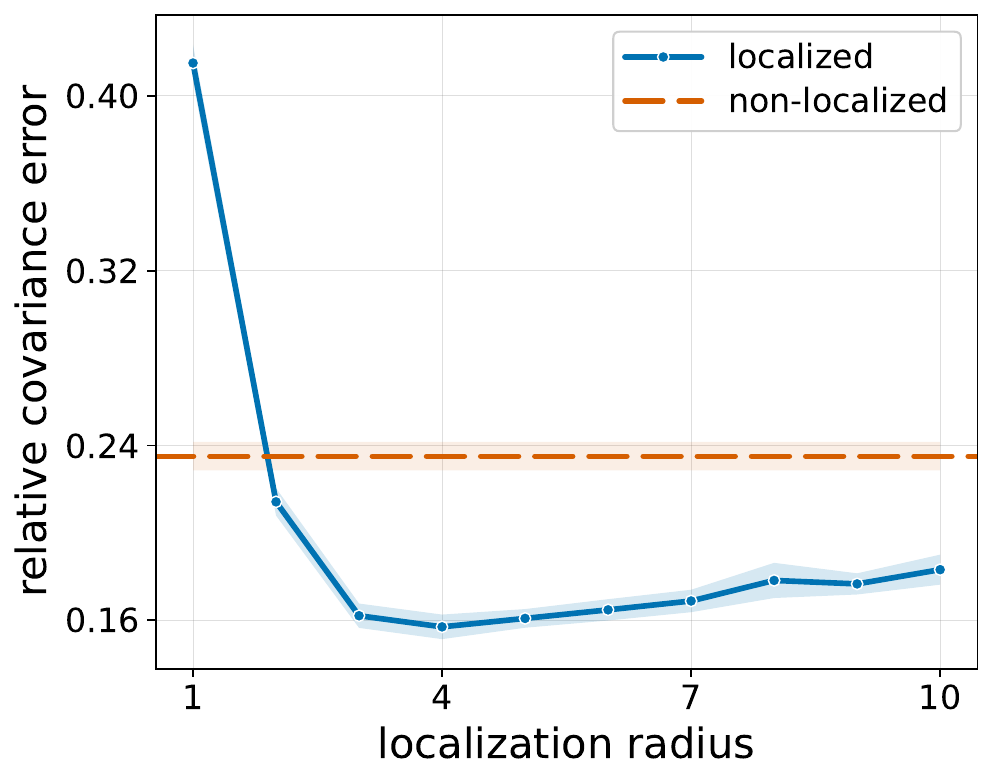}
\includegraphics[width=4.4cm]{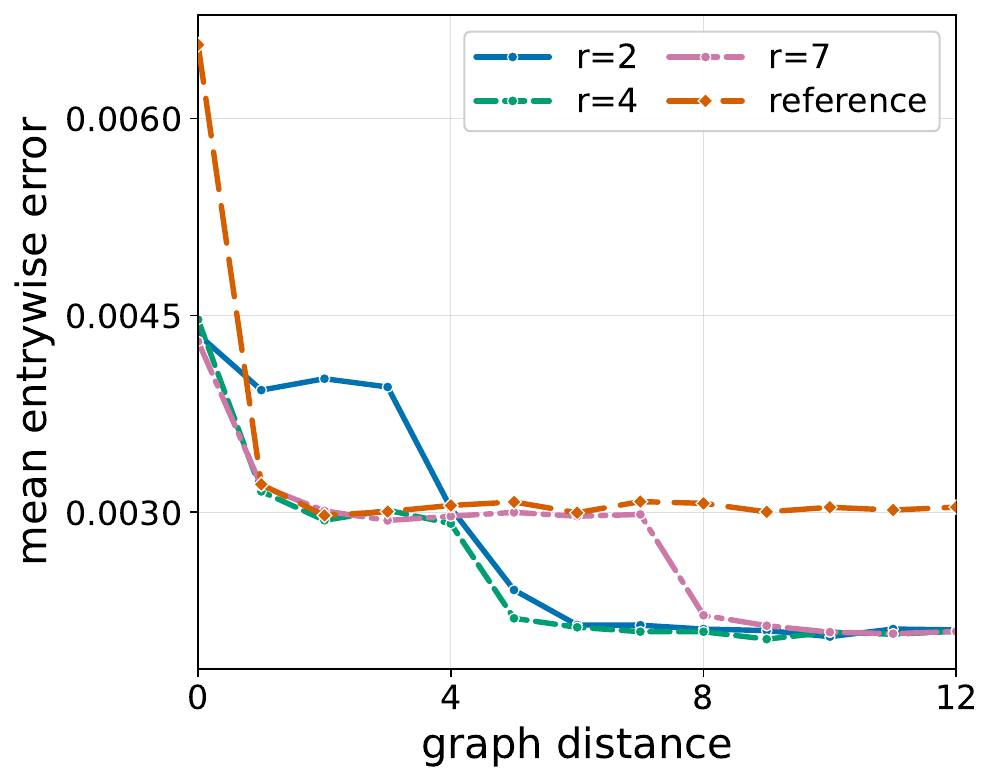} \\
\includegraphics[width=13.3cm]{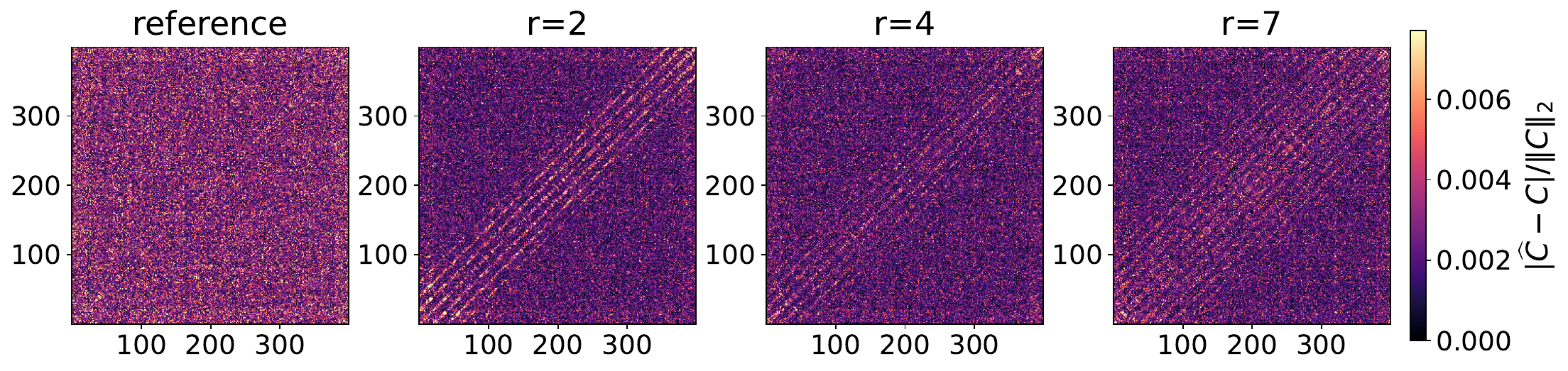}
\caption{Top left: the localization radius $r_{\rm loc}(t)$ \eqref{eqn:EffLoc2} for the lattice Gaussian model \eqref{eq:LatticeGMRF} with different $\gamma$.
Top middle: Relative $\ell^2$-error (4.9) of the sample covariance for different localization radii $r$; the reference error is from the non-localized sampling process (4.7). The shaded area denotes the $1\sigma$ region. Top right: Entrywise relative error averaged over entries with the same graph distance \eqref{eqn:CovErr_Dist}.
Bottom: Entrywise error of the sample covariance with different localization radius $r \in \{2,4,7\}$; the reference is from (4.7).}
\label{fig:Lattice_LocErr}
\end{figure}

Next, we fix $\gamma=1$ and use the localized sampling process (4.8) to sample the distribution, and compare it with the reference sampling process (4.7).
The top-middle plot in \Cref{fig:Lattice_LocErr} shows that the relative $\ell^2$-error of the localized sampler first decreases and then increases as the localization radius $r$ grows.
This again reflects the balance between localization error and statistical error.
Besides the relative $\ell^2$-error of the sample covariance, we also compute the entrywise relative error averaged over entries with the same graph distance: for each graph distance $\ell$,
\begin{equation}
\label{eqn:CovErr_Dist}
    {\rm err} (\ell) := \frac{1}{|\{(i,j): \sfd_{\sfG}(i,j)=\ell\}|} \sum_{i,j: \sfd_{\sfG}(i,j)=\ell} \frac{|\widehat{C}(i,j)-C(i,j)|}{\norm{C}_2}.
\end{equation}
Here $C = P_0^{-1}$.
The top-right and bottom plots show that the non-localized sampler produces persistent long-range covariance errors, reflecting spurious long-range correlations caused by insufficient sample size.
In contrast, localization suppresses these long-range errors, while too small a localization radius leads to large short-range errors.
Thus a suitable localization radius $r$ is important for balancing localization and statistical errors, consistent with Theorem~3.4.

\subsection{One-dimensional Ginzburg--Landau model}
\label{sec:GinzburgLandau1D}

In this section, we consider the one-dimensional Ginzburg--Landau model with $D=32$:
\begin{equation}
\label{eq:GinzburgLandau1D}
    p_0(x) \propto \exp \Brac{ - \sum_{j=1}^D \Brac{
        \lambda (x_{j+1}-x_j)^2 + \mu (x_j^2-1)^2 }
    },
\end{equation}
where $x_{D+1}=x_1$.
Here we fix $\lambda = 3.2$ and $\mu = 0.5$. 
It is straightforward to verify that the model is strictly localized with respect to the nearest-neighbor graph $\sfG$ on the cyclic lattice $[D]$.
However, this distribution is bimodal and thus not log-concave.

We first generate the reference samples from $p_0$ by PT-MALA \cite{10.1214/ss/1177011137}.
Due to the translation invariance of the model, it suffices to learn only one score component $s_j$.
We train such localized score networks for different localization radii $r \in \{2,4,6,\infty\}$ (here $r=\infty$ refers to the non-localized case).
For a given $r$, the neighborhood $\mcN_j^r$ is the cyclic patch of size $2r+1$ around site $j$.
As in the bimodal CIR experiment, we augment the input of the score network with the sample average $\bar x$ to adapt to the multimodality.

The hyperparameters are chosen as follows.
For each localization radius \(r\), we train using \(4096\) reference samples, minibatches of size \(512\), and \(8000\) AdamW iterations with learning rate \(7.5\times 10^{-4}\) and weight decay \(10^{-5}\).
The score network has three hidden layers of width \(96\) with SiLU activations.
In the sampling stage, we use \(1000\) reverse steps and generate \(2048\) samples for each localization radius.

We use the following metrics to evaluate the quality of the generated samples.
(1) The pooled coordinate marginal distribution, which is the histogram of all entries \(X_j\) in the generated samples.
(2) The phase distribution, which is the distribution of the sample average \(\bar X = D^{-1}\sum_{j=1}^D X_j\).
(3) The centered spatial autocorrelation \(\rho(\ell)\), which measures the correlation between entries separated by graph distance \(\ell\):
\begin{equation}
\label{eqn:autocorr}
    \rho(\ell) = \frac{C(\ell)}{C(0)},\quad
    C(\ell) = D^{-1}\sum_{j=1}^D \mE[(X_j-\bar X)(X_{j+\ell}-\bar X)].
\end{equation}
(4) The frequency-binned energy spectrum \(E(k)\) \cite{chen2025lipschitz}, which captures the distribution of energy across different frequency modes: let \(\widehat X_m \) denote the discrete Fourier transform of the generated samples, then
\begin{equation}
\label{eqn:energy_spectrum}
    E(k) = \sum_{k\le |m|<k+1} \mE[|\widehat X_m|^2],
\end{equation}
We compare the above metrics with those from the reference samples, and report the errors in \Cref{fig:GinzburgLandau1D}.

\begin{figure}[t]
\centering
\includegraphics[width=13.0cm]{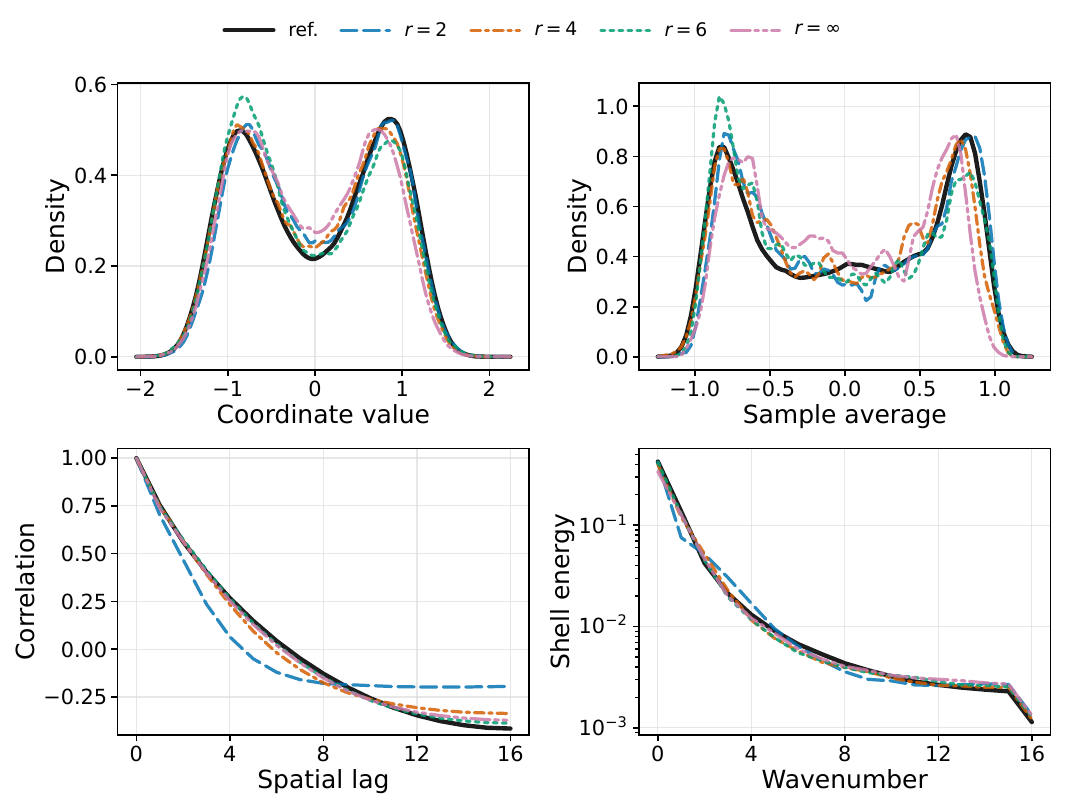}
\caption{One-dimensional Ginzburg--Landau experiment \eqref{eq:GinzburgLandau1D}.
The panels compare the reference law with generated samples for the pooled coordinate marginal, the sample-average phase distribution, the autocorrelation \eqref{eqn:autocorr}, and the frequency-binned energy spectrum \eqref{eqn:energy_spectrum}.}
\label{fig:GinzburgLandau1D}
\end{figure}

We also report the numerical comparisons in \Cref{tab:GinzburgLandau1D}.
We first compare the \(L^1\) error between the generated and reference distributions of the pooled coordinate marginal on common bins.
We also compare the $\ell_2$ errors of the autocorrelation \eqref{eqn:autocorr} and energy spectrum \eqref{eqn:energy_spectrum}.
Finally, we compare the Maximum Mean Discrepancy (MMD) using the Gaussian kernel with the bandwidth chosen as the median squared pairwise distance in a reference subsample.

\begin{table}[t]
\centering
\caption{Comparison for the one-dimensional Ginzburg--Landau experiment under different localization radii \(r\).}
\label{tab:GinzburgLandau1D}
\small
\begin{tabular}{c|cccc}
\hline
$r$ & marginal err & autocorr. err & energy err & MMD$^2$ \\
\hline
$2$ & $0.0886$ & $0.343$ & $0.140$ & $0.00105$ \\
$4$ & $\mathbf{0.0690}$ & $0.108$ & $0.0730$ & $\mathbf{0.000629}$ \\
$6$ & $0.101$ & $\mathbf{0.0366}$ & $\mathbf{0.0239}$ & $0.00394$ \\
$\infty$ & $0.157$ & $0.0537$ & $0.193$ & $0.00429$ \\
\hline
\end{tabular}
\end{table}

Among the tested radii, $r=4$ gives the closest coordinate marginal and MMD, while $r=6$ gives the closest spatial autocorrelation and energy spectrum.
The full-context case does not dominate these diagnostics at this sample size, which is consistent with the localization--statistical error tradeoff: increasing $r$ reduces local approximation error but also increases the effective dimension of the learned local problem.

\section{Proofs in \texorpdfstring{Section~2}{Section 2}}
\label{app:DMnLoc}

\subsection{Proof of \texorpdfstring{Theorem~2.1}{Theorem 2.1}}
\label{app:wLocDist}

\begin{proof}
Recall
\[
    p_t(y) = \int \GN ( y; \alpha_t x, \sigma_t^2 I ) p_0(x) \mdd x.
\]
We first compute the Hessian of the log density of $p_t$.
\begin{align*}
    \nabla^2 \log p_t(y) =~& \frac{ \nabla^2 p_t(y) }{ p_t(y) } - \frac{\nabla p_t(y)}{ p_t(y)} \frac{\nabla p_t(y) \matT}{ p_t(y)} \\
    =~& \frac{ 1 }{ p_t(y) } \int \Rectbrac{ - \frac{1}{\sigma_t^2} I +  \Brac{ - \frac{ y - \alpha_t x}{ \sigma_t^2 } } \Brac{ - \frac{ y - \alpha_t x}{ \sigma_t^2 } }\matT } \GN ( y; \alpha_t x, \sigma_t^2 I ) p_0(x) \mdd x \\
    &- \frac{ 1 }{ p_t(y) } \int \Brac{ - \frac{ y - \alpha_t x}{ \sigma_t^2 } } \GN ( y; \alpha_t x, \sigma_t^2 I ) p_0(x) \mdd x \\
    &\cdot \frac{ 1 }{ p_t(y) } \int \Brac{ - \frac{ y - \alpha_t x}{ \sigma_t^2 } }\matT \GN ( y; \alpha_t x, \sigma_t^2 I ) p_0(x) \mdd x \\
    =~& - \sigma_t^{-2} I + \sigma_t^{-4} \mE_{p_{0|t}(x|y)} \Rectbrac{ \Brac{ y - \alpha_t x } \Brac{ y - \alpha_t x }\matT } \\
    &- \sigma_t^{-4} \mE_{p_{0|t}(x|y)} \Rectbrac{ y - \alpha_t x } ( \mE_{p_{0|t}(x|y)} \Rectbrac{ y - \alpha_t x } )\matT \\
    =~& - \sigma_t^{-2} I + \sigma_t^{-4} \Cov_{p_{0|t}(x|y)} \Brac{ y - \alpha_t x, y - \alpha_t x } = - \sigma_t^{-2} I + \alpha_t^2 \sigma_t^{-4} \Cov_{p_{0|t}(x|y)} \Brac{ x, x },
\end{align*}
where $p_{0|t}(\cdot|y)$ is the distribution of $x$ conditioned on $y$.
As a consequence, for any $ i \ne j$, we have
\begin{equation}
\label{eqn:pf_HessPt}
    \nabla_{ij}^2 \log p_t(y) = \alpha_t^2 \sigma_t^{-4} \Cov_{p_{0|t}(x|y)} \Brac{ x_{i}, x_{j} }.
\end{equation}
Consider the conditional distribution $p_{0|t}(x|y)$, whose log density is
\[
    \log p_{0|t}(x|y) = - \log p_t(y) + \log p_0(x) - \frac{1}{2\sigma_t^2} \norm{ y - \alpha_t x }^2 - \frac{D}{2} \log (2\pi \sigma_t^2).
\]
Fix $y$, and denote for simplicity $q(x) = p_{0|t}(x|y)$. Then
\[
    \nabla^2 \log q(x) = \nabla^2 \log p_0(x) - \frac{\alpha_t^2}{\sigma_t^2} I.
\]
By assumption, $\nabla_{ij}^2 \log p_0 = 0$ if $i \notin \mcN_j$, and $m I \preceq - \nabla^2 \log p_0 \preceq M I$. Thus,
\[
    \forall i\notin \mcN_j, \quad \nabla_{ij}^2 \log q(x) = 0.
\]
\begin{equation}
\label{eqn:pf_PtBound}
    \Brac{ m + \frac{\alpha_t^2}{\sigma_t^2} } I \preceq - \nabla^2 \log q \preceq \Brac{ M + \frac{\alpha_t^2}{\sigma_t^2} } I.
\end{equation}
By Proposition~2.2, for any Lipschitz functions $f,g$, we have
\[
    \norme{ \Cov_{q(x)} \Brac{ f(x_i), g(x_j) } } \le \norme{f}_\Lip \norme{g}_\Lip \Brac{ m + \frac{\alpha_t^2}{\sigma_t^2} }^{-1} \Brac{ 1 - \frac{ m \sigma_t^2 + \alpha_t^2 }{ M \sigma_t^2 + \alpha_t^2 } }^{\sfd_\sfG(i,j)}.
\]
Recall \eqref{eqn:pf_HessPt}, and by definition of the matrix norm,
\[
    \norm{ \nabla_{ij}^2 \log p_t(y) } = \sup_{\norm{t_i} = \norm{t_j} = 1} t_i\matT \nabla_{ij}^2 \log p_t(y) t_j = \sup_{\norm{t_i} = \norm{t_j} = 1} \alpha_t^2 \sigma_t^{-4} \Cov_{q(x)} \Brac{ t_i\matT x_i, t_j\matT x_j } .
\]
Take $f(x_i) = t_i\matT x_i$ and $g(x_j) = t_j\matT x_j$, and note $\norme{f}_\Lip = \norme{g}_\Lip = 1$, we obtain
\[
    \norm{ \nabla_{ij}^2 \log p_t(y) } \le \alpha_t^2 \sigma_t^{-4} \Brac{ m + \frac{\alpha_t^2}{\sigma_t^2} }^{-1} \Brac{ 1 - \frac{ m \sigma_t^2 + \alpha_t^2 }{ M \sigma_t^2 + \alpha_t^2 } }^{\sfd_\sfG(i,j)}.
\]
The conclusion follows by noting the above bound holds for all $y$ and $i \ne j$.
\end{proof}

\subsection{Proof of \texorpdfstring{Proposition~2.2}{Proposition 2.2}}
\label{app:CorrExpDecay}

\begin{proof}
By subtracting the mean, we assume w.l.o.g.\ that $\mE_{p(x)} [f(x_i)] = \mE_{p(x)} [g(x_j)] = 0$. Then
\[
    \Cov_{p(x)} \Brac{ f(x_i), g(x_j) } = \int f(x_i) g(x_j) p(x) \mdd x.
\]
Consider the marginal Stein equation \cite{cui2025stein}
\[
    - \Delta u_f (x) - \nabla \log p(x) \cdot \nabla u_f(x) = f(x_i).
\]
By \Cref{lem:GradEsti}, the following gradient estimate of $u_f$ holds:
\[
    \norm{ \nabla_j u_f }_\infty \le \frac{1}{m} \Brac{ 1 - \frac{m}{M} }^{\sfd_\sfG(i,j)} \norme{ f }_\Lip.
\]
By integration by parts, it holds that
\begin{align*}
    \int f(x_i) g(x_j) p(x) \mdd x =~& \int \Brac{ - \Delta u_f (x) - \nabla \log p(x) \cdot \nabla u_f(x) } g(x_j) p(x) \mdd x \\
    =~& \int \nabla u_f (x) \cdot \nabla_x g(x_j) p(x) \mdd x \\
    &+ \int \nabla u_f (x) \cdot \nabla p(x) g(x_j) \mdd x - \int \nabla u_f (x) \cdot \nabla \log p(x) g(x_j) p(x) \mdd x \\
    =~& \int \nabla_j u_f (x) \cdot \nabla g(x_j) p(x) \mdd x.
\end{align*}
Here we use $\nabla_{x_i} g(x_j) = 0$ if $i\ne j$. Combined, we obtain
\begin{align*}
    \norme{ \Cov_{p(x)} \Brac{ f(x_i), g(x_j) } } =~& \norme{ \int \nabla_j u_f (x) \cdot \nabla g(x_j) p(x) \mdd x } \\
    \le~& \int \norm{ \nabla_j u_f (x) } \norm{ \nabla g(x_j) } p(x) \mdd x \le \frac{1}{m} \Brac{ 1 - \frac{m}{M} }^{\sfd_\sfG(i,j)} \norme{ f }_\Lip \norme{g}_\Lip.
\end{align*}
This completes the proof.
\end{proof}

\begin{lem}
\label{lem:GradEsti}
Suppose $p$ is localized w.r.t.~an undirected graph $\sfG$ and is log-concave and smooth, i.e., $\exists 0<m\le M <\infty$ such that $m I \preceq - \nabla^2 \log p(x) \preceq M I$. For any $i$ and Lipschitz function $f: \mR^{d_i} \to \mR$, consider the marginal Stein equation
\begin{equation}
    - \Delta_p u_f (x) := - \Delta u_f (x) - \nabla \log p(x) \cdot \nabla u_f(x) = f(x_i) - \mE_{p(x)} [f(x_i)].
\end{equation}
The following gradient estimate holds:
\begin{equation}
    \norm{ \nabla_j u_f }_\infty \le \frac{1}{m} \Brac{ 1 - \frac{m}{M} }^{\sfd_\sfG(i,j)} \norme{ f }_\Lip.
\end{equation}
\end{lem}

\begin{proof}
The proof is based on a refinement of the analysis in \cite{cui2025stein}. Note that $\Delta_p = \Delta + \nabla \log p \cdot \nabla $ is the generator of the Langevin dynamics
\begin{equation}
\label{eq:pf_Langevin}
    \mdd X_t^x = \nabla \log p(X_t^x ) \mdd t + \sqrt{2} \mdd W_t, \quad X_0^x = x.
\end{equation}
The Stein equation with such generator-type operators is known to admit explicit solutions \cite{MR1035659}:
\[
    u_f(x) = - \int_0^\infty \mE \Rectbrac{ f(X_{t,i}^x) - \mE_\pi[f(x_i)] } \mdd t.
\]
See also Lemma 2.1 in \cite{cui2025stein}.
Differentiating w.r.t.~$x_j$ gives
\[
    \nabla_j u_f(x) = - \int_0^\infty \mE \Rectbrac{ \nabla_j X_{t,i}^x \cdot \nabla f (X_{t,i}^x) } \mdd t.
\]
Here $\nabla_j X_{t,i}^x$ is the partial derivative w.r.t.~$x_j$ of the sample path. Differentiating on both sides is valid due to the exponential decay of $\nabla_j X_{t,i}^x$. Since $f$ is Lipschitz, we obtain
\begin{equation}
\label{eq:pf_GradCtrl}
    \norm{ \nabla_j u_f(x) } \le \int_0^\infty \mE \Rectbrac{ \normo{ \nabla_j X_{t,i}^x } \norm{ \nabla f (X_{t,i}^x) }  } \mdd t \le \norme{f}_\Lip \int_0^\infty \mE \normo{ \nabla_j X_{t,i}^x } \mdd t.
\end{equation}
It remains to control $\nabla_j X_{t,i}^x$. Differentiating w.r.t.~$x$ in \eqref{eq:pf_Langevin}, we obtain
\[
    \mdd \nabla X_t^x = - H_t \cdot \nabla X_t^x \mdd t, \quad H_t := - \nabla^2 \log p(X_t^x).
\]
Let $G_t = \mee^{mt} \nabla X_t^x$ and $\sfH_t = H_t - m I $. Then
\[
    \diff{}{t} G_t = \mee^{mt} \Brac{ m \nabla X_t^x - H_t \nabla X_t^x } = - \sfH_t G_t, \quad G_0 = \nabla X_0^x = I.
\]
By assumption, $0 \preceq \sfH_t \preceq (M-m) \sfI$, and $\sfH_t$ has dependency graph $\sfG$.
By Lemma 2.3 in \cite{cui2025stein},
\[
    \normo{ \nabla_j X_t^x } = \mee^{-mt} \normo{ G_t(i,j) } \le \mee^{-mt} \mee^{-(M-m)t} \sum_{k=\sfd_\sfG(i,j)}^\infty \frac{t^k (M-m)^k}{k!}.
\]
Recall \eqref{eq:pf_GradCtrl}, this implies
\begin{align*}
    \norm{ \nabla_j u_f(x) } \le~& \norme{f}_\Lip \int_0^\infty \mE \normo{ \nabla_j X_{t,i}^x } \mdd t \\
    \le~& \norme{f}_\Lip \int_0^\infty \mee^{-Mt} \sum_{k=\sfd_\sfG(i,j)}^\infty \frac{t^k (M-m)^k}{k!} \mdd t \\
    =~& \norme{f}_\Lip \frac{1}{M} \sum_{k=\sfd_\sfG(i,j)}^\infty \Brac{ 1 - \frac{m}{M} }^k = \frac{1}{m} \Brac{ 1 - \frac{m}{M} }^{\sfd_\sfG(i,j)} \norme{f}_\Lip.
\end{align*}
The conclusion follows by noting the above bound holds for all $x$.
\end{proof}

\section{Proofs in \texorpdfstring{Section~3}{Section 3}}
\label{app:LocDM}

\subsection{Proof of \texorpdfstring{Proposition~3.1}{Proposition 3.1}}
\label{app:ErrDecomp}

\begin{proof}
Let $\sfQ$ and $\widehat{\sfQ}$ denote the path measures for the reverse process (2.2) and the sampling process (3.5), respectively, i.e., $\sfQ_t = \Law (Y_t), \widehat{\sfQ}_t = \Law (\widehat{Y}_t)$. By the data-processing inequality, we have
\[
    \KL ( p_{\underline{T}} \| \widehat{q}_{\mean{T}-\underline{T}} ) = \KL ( \sfQ_{\mean{T}-\underline{T}} \| \widehat{\sfQ}_{\mean{T}-\underline{T}} ) \le \KL ( \sfQ_{[0,\mean{T}-\underline{T}]} \| \widehat{\sfQ}_{[0,\mean{T}-\underline{T}]} ).
\]
By the Girsanov theorem \cite{MR3155209}, we have
\begin{align*}
    \KL ( \sfQ_{[0,\mean{T}-\underline{T}]} \| \widehat{\sfQ}_{[0,\mean{T}-\underline{T}]} ) =~& \KL ( \sfQ_0 \| \widehat{\sfQ}_0 ) + \int_0^{\mean{T}-\underline{T}} \mE_{y(t) \sim \sfQ_t} \Rectbrac{ \norm{ \widehat{s}(y(t),\mean{T}-t) - s(y(t),\mean{T}-t) }^2 } \mdd t \\
    =~& \KL ( p_{\mean{T}} \| \GN(0,I) ) + \int_{\underline{T}}^{\mean{T}} \mE_{x(t)\sim p_t} \Rectbrac{ \norm{ \widehat{s}(x(t),t) - s(x(t),t) }^2 } \mdd t.
\end{align*}
By the convergence of the OU process \cite{MR3155209}, we have \[
    \KL ( p_{\mean{T}} \| \GN(0,I) ) \le \mee^{-2\mean{T}} \KL ( p_0 \| \GN(0,I) ).
\]
The conclusion follows by combining the above relations.
\end{proof}

\subsection{Proof of \texorpdfstring{Theorem~3.2}{Theorem 3.2}}
\label{app:loc_err}

\begin{proof}
Note the optimal solution is given by (3.10), i.e.,
\[
    s_j^*(x,t) = \mE_{x'\sim p_t} \Rectbrac{ \nabla_j \log p_t(x') ~\Big|~ x_{\mcN_j^r}' = x_{\mcN_j^r} }.
\]
By \eqref{eqn:pf_PtBound}, $p_t$ is $\Brac{m+\frac{\alpha_t^2}{\sigma_t^2}}$-strongly log-concave, so that the conditional distribution $p_t(x_{\mcN_j^{r_\bot}}|x_{\mcN_j^r})$ is also $\Brac{m+\frac{\alpha_t^2}{\sigma_t^2}}$-strongly log-concave. By the Poincar\'e inequality \cite{MR3155209},
\begin{align*}
    \norm{ s_j^*(x,t) - s_j (x,t) }_{L^2(p_t)}^2 =~& \mE_{x \sim p_t} \Rectbrac{ \mE_{x'\sim p_t} \Rectbrac{ \norm{ s_j^*(x',t) - \nabla_j \log p_t(x') }^2 ~\Big|~ x_{\mcN_j^r}' = x_{\mcN_j^r} } } \\
    \le~& \mE_{x \sim p_t} \Rectbrac{ \Brac{m+\frac{\alpha_t^2}{\sigma_t^2}}^{-1} \mE_{x'\sim p_t} \Rectbrac{ \normo{ \nabla_{\mcN_j^{r_\bot}} \nabla_j \log p_t(x') }_{\rm F}^2 ~\Big|~ x_{\mcN_j^r}' = x_{\mcN_j^r} } }.
\end{align*}
Here $\normo{\cdot}_{\rm F}$ denotes the Frobenius norm. By Theorem~2.1, it holds that for any $i \ne j$,
\[
    \norm{ \nabla_{ij}^2 \log p_t(x) }_\infty \le \frac{\alpha_t^2}{ \sigma_t^2 \Brac{ m \sigma_t^2 +\alpha_t^2 } } \Brac{ 1 - \frac{ m \sigma_t^2 + \alpha_t^2 }{ M \sigma_t^2 + \alpha_t^2 } }^{\sfd_\sfG(i,j)}.
\]
Since $\normo{\nabla_{ij}^2 \log p_t(x)}_{\rm F}^2 \le d_j \normo{ \nabla_{ij}^2 \log p_t(x) }_\infty^2 $, we obtain that
\begin{align*}
    & \mE_{x'\sim p_t} \Rectbrac{ \normo{ \nabla_{\mcN_j^{r_\bot}} \nabla_j \log p_t(x') }_{\rm F}^2 ~\Big|~ x_{\mcN_j^r}' = x_{\mcN_j^r} } \\
    =~& \sum_{i:\sfd_\sfG(i,j)>r} \mE_{x'\sim p_t} \Rectbrac{ \normo{ \nabla_{ij}^2 \log p_t(x') }_{\rm F}^2 ~\Big|~ x_{\mcN_j^r}' = x_{\mcN_j^r} } \\
    \le~& d_j \sum_{i:\sfd_\sfG(i,j)>r} \frac{\alpha_t^4}{ \sigma_t^4 \Brac{ m \sigma_t^2 +\alpha_t^2 }^2 } \Brac{ 1 - \frac{ m \sigma_t^2 + \alpha_t^2 }{ M \sigma_t^2 + \alpha_t^2 } }^{2\sfd_\sfG(i,j)}.
\end{align*}
Therefore,
\begin{align*}
    & \int_0^{\mean{T}} \norm{ s_j^*(x,t) - s_j (x,t) }_{L^2(p_t)}^2 \mdd t \\
    \le~& \int_0^{\mean{T}} \Rectbrac{ d_j \sum_{i:\sfd_\sfG(i,j)>r} \Brac{m+\frac{\alpha_t^2}{\sigma_t^2}}^{-1} \frac{\alpha_t^4}{ \sigma_t^4 \Brac{ m \sigma_t^2 +\alpha_t^2 }^2 } \Brac{ 1 - \frac{ m \sigma_t^2 + \alpha_t^2 }{ M \sigma_t^2 + \alpha_t^2 } }^{2\sfd_\sfG(i,j)} } \mdd t \\
    \le~& d_j \sum_{k=r+1}^\infty |\{i:\sfd_\sfG(i,j)=k\}| \int_0^\infty \frac{\alpha_t^4}{ \sigma_t^2 \Brac{ m \sigma_t^2 +\alpha_t^2 }^3 } \Brac{ 1 - \frac{ m \sigma_t^2 + \alpha_t^2 }{ M \sigma_t^2 + \alpha_t^2 } }^{2k} \mdd t \\
    \le~& d_j \max\{1,m^{-1}\} \log \kappa \sum_{k=r+1}^\infty |\{i:\sfd_\sfG(i,j)=k\}| (1-\kappa^{-1})^{2k}.
\end{align*}
The last step uses \Cref{lem:IntBound}. By the Abel transformation and the sparsity assumption (3.11),
\begin{align*}
    \sum_{k=r+1}^\infty |&\{i:\sfd_\sfG(i,j)=k\}| (1-\kappa^{-1})^{2k} = \sum_{k=r+1}^\infty \Rectbrac{ |\mcN_j^k| - |\mcN_j^{k-1}| } (1-\kappa^{-1})^{2k} \\
    =~& \sum_{k=r+1}^\infty |\mcN_j^k| \Rectbrac{ (1-\kappa^{-1})^{2k} - (1-\kappa^{-1})^{2(k+1)}  } - |\mcN_j^r| (1-\kappa^{-1})^{2(r+1)} \\
    \le~& S \kappa^{-1} (2-\kappa^{-1}) \sum_{k=r+1}^\infty k^{\nu} (1-\kappa^{-1})^{2k} \le 2 S \kappa^{-1} (1-\kappa^{-1})^{2r} \sum_{k=1}^\infty (k+r)^{\nu} (1-\kappa^{-1})^{2k}.
\end{align*}
One can show that $ \sum_{k \in\mZ_+} k^n x^k \le n! x (1-x)^{-n-1}$ (see Lemma 2.4 in \cite{cui2025stein}), so that
\begin{align*}
    \sum_{k=1}^\infty (k+r)^{\nu} (1-\kappa^{-1})^{2k} = \sum_{k=1}^\infty \Brac{ 1 + \frac{r}{k} }^{\nu} k^{\nu} (1-\kappa^{-1})^{2k} \le (r+1)^{\nu} \sum_{k=1}^\infty k^{\nu} (1-\kappa^{-1})^{2k} &\\
    \le (r+1)^{\nu} \nu! (1-\kappa^{-1})^2 [1-(1-\kappa^{-1})^2]^{-\nu-1} \le (r+1)^{\nu} \nu!(1-\kappa^{-1})^2 \kappa^{2(\nu+1)}.
\end{align*}
Combining the above inequalities, we obtain
\begin{align*}
    \int_{\underline{T}}^{\mean{T}} &\norm{ s_j^*(x,t) - s_j (x,t) }_{L^2(p_t)}^2 \mdd t \le \int_0^{\mean{T}} \norm{ s_j^*(x,t) - s_j (x,t) }_{L^2(p_t)}^2 \mdd t \\
    \le~& d_j \max\{1,m^{-1}\} \log \kappa \cdot 2 S \kappa^{-1} (1-\kappa^{-1})^{2r} \cdot (r+1)^{\nu} \nu!(1-\kappa^{-1})^2 \kappa^{2(\nu+1)} \\
    =~& C d_j (r+1)^{\nu} (1-\kappa^{-1})^{2(r+1)}.
\end{align*}
where we denote $C = 2 S \max\{1,m^{-1}\} \nu!  \kappa^{2\nu+1} \log \kappa $.

The second claim follows from the property of conditional expectation:
\begin{align*}
    \| s_{\theta,j}(x,t) &- s_j (x,t) \|_{L^2(p_t)}^2 = \normo{ u_{\theta,j}(x_{\mcN_j^r},t) - s_j (x,t) }_{L^2(p_t)}^2 \\
    =~& \mE_{x \sim p_t} \Rectbrac{ \mE_{x'\sim p_t} \Rectbrac{ \normo{ u_{\theta,j}(x_{\mcN_j^r},t) - u_j^*(x_{\mcN_j^r},t) + u_j^*(x_{\mcN_j^r},t) - s_j (x',t) }^2 ~\Big|~ x_{\mcN_j^r}' = x_{\mcN_j^r} } } \\
    =~& \mE_{x \sim p_t} \Rectbrac{  \normo{ u_{\theta,j}(x_{\mcN_j^r},t) - u_j^*(x_{\mcN_j^r},t) }^2 } \\
    &+ \mE_{x \sim p_t} \Rectbrac{ \mE_{x'\sim p_t} \Rectbrac{ \normo{ u_j^*(x_{\mcN_j^r},t) - s_j (x',t) }^2 ~\Big|~ x_{\mcN_j^r}' = x_{\mcN_j^r} } } \\
    =~& \normo{ s_{\theta,j}(x,t) - s_j^*(x,t) }_{L^2(p_t)}^2 + \normo{ s_j^*(x,t) - s_j (x,t)  }_{L^2(p_t)}^2.
\end{align*}
This completes the proof.
\end{proof}

\begin{lem}
\label{lem:IntBound}
Let $\kappa = M/m \ge 1$ and $k\ge 1$. It holds that
\[
    \int_0^\infty \frac{\alpha_t^4}{ \sigma_t^2 \Brac{ m \sigma_t^2 +\alpha_t^2 }^3 } \Brac{ 1 - \frac{ m \sigma_t^2 + \alpha_t^2 }{ M \sigma_t^2 + \alpha_t^2 } }^{2k} \mdd t \le \max\{1,m^{-1}\} (1-\kappa^{-1})^{2k} \log \kappa .
\]
\end{lem}
\begin{proof}
Denote $\lambda = \dfrac{\alpha_t^2}{\sigma_t^2} = \dfrac{\mee^{-2t}}{1-\mee^{-2t}}$, then $\sigma_t^2 = \dfrac{1}{1+\lambda}$ and $ \dfrac{\mdd \lambda}{\mdd t} = - 2 \lambda (1+\lambda) $. The integral is
\[
    \int_0^\infty \frac{\lambda^2 (1+\lambda)^2}{ \Brac{ m + \lambda }^3 } \Brac{ 1 - \frac{ m + \lambda }{ M + \lambda } }^{2k} \frac{\mdd \lambda}{2\lambda(1+\lambda)} = \int_0^\infty \frac{\lambda (1+\lambda)}{ 2\Brac{ m + \lambda }^3 } \Brac{ 1 - \frac{ m + \lambda }{ M + \lambda } }^{2k} \mdd \lambda.
\]
Let $x = \lambda/m$, and the integral can be bounded by
\[
    \int_0^\infty \frac{mx (1+mx)}{ 2\Brac{ m + mx }^3 } \Brac{ 1 - \frac{ m + mx }{ M + mx } }^{2k} m \mdd x \le \frac{\max\{1,m\}}{2m} \int_0^\infty \frac{x}{ \Brac{ 1 + x }^2 } \Brac{ 1 - \frac{ 1 + x }{ \kappa + x } }^{2k} \mdd x.
\]
Note that
\begin{align*}
    \frac{1}{(1-\kappa^{-1})^{2k}} &\int_0^\infty \frac{x}{ \Brac{ 1 + x }^2 } \Brac{ 1 - \frac{ 1 + x }{ \kappa + x } }^{2k} \mdd x = \int_0^\infty \frac{x}{ \Brac{ 1 + x }^2 } \Brac{ \frac{ \kappa }{ \kappa + x } }^{2k} \mdd x \\
    =~& \int_0^\infty \frac{ y }{ \Brac{ \kappa^{-1} + y }^2 } \Brac{ \frac{ 1 }{ 1 + y } }^{2k} \mdd y \le \int_0^\infty \frac{ y }{ \Brac{ \kappa^{-1} + y }^2 } \Brac{ \frac{ 1 }{ 1 + y } }^2 \mdd y \\
    <~& \int_0^{\kappa^{-1}} \kappa^2 y \mdd y + \int_{\kappa^{-1}}^1 \frac{\mdd y}{y} + \int_1^\infty \frac{\mdd y}{y^3} = 1 + \log \kappa \le 2 \log \kappa.
\end{align*}
The conclusion follows by combining the above inequalities.
\end{proof}

\subsection{Proof of \texorpdfstring{Proposition~3.3}{Proposition 3.3}}
\label{app:dsm_j}

\begin{proof}
Denote $z := x_{\mcN_j^r}$.
For the first equality, note that $\pi_{t|0}^j (z(t)|z) = \GN ( z(t); \alpha_t z, \sigma_t^2 I )$.
It holds that
\[
    \nabla_j \log \pi_{t|0}^j (z(t)|z) = - \sigma_t^{-2} ( z_j(t) - \alpha_t z_j ).
\]
Note one can realize $z(t) = \alpha_t z + \sigma_t \eta (t) \sim \pi_{t|0}^j(z(t)|z)$ if $\eta (t) \sim \GN(0,I)$, so that
\begin{align*}
    \mE_{z(t) \sim \pi_{t|0}^j(\cdot|z)} &\Rectbrac{ \norm{ u_{\theta,j}(z(t), t) - \nabla_j \log \pi_{t|0}^j (z(t)|z) }^2 } \\
    =~& \mE_{\eta(t) \sim \GN(0,I)} \Rectbrac{ \norm{ u_{\theta,j}(\alpha_t z + \sigma_t \eta (t), t) + \sigma_t^{-1} \eta_j (t) }^2 } .
\end{align*}
This verifies the first equality.
For the second equality, we first claim that
\begin{equation}
\label{eqn:pf_ujstar}
    u_j^*(z,t) = \nabla_j \log \pi_t^j(z).
\end{equation}
Given this, the third equality follows from the basic trick in denoising score matching: take $y = z(t)$ and $\nu(y,z) = \pi_0^j(z) \pi_{t|0}^j(y|z)$ in the following identity:
\begin{align*}
    & \mE_{z \sim \nu(z) } \mE_{y\sim \nu(y|z)} \norm{ s_\theta(y) - \nabla_y \log \nu(y|z) }^2 \\
    =~& \mE_{z \sim \nu(z) } \mE_{y\sim \nu(y|z)} \Rectbrac{ \norm{ s_\theta(y) }^2 - 2 (s_\theta(y))\matT \nabla_y \log \nu(y|z)  + \norm{ \nabla_y \log \nu(y|z) }^2 } \\
    =~& \mE_{z \sim \nu(z) } \mE_{y\sim \nu(y|z)} \Rectbrac{ \norm{ s_\theta(y) }^2 + 2 \text{tr} \Brac{ \nabla s_\theta(y) } + \norm{ \nabla_y \log \nu(y|z) }^2 } \\
    =~& \mE_{y\sim \nu(y)} \Rectbrac{ \norm{ s_\theta(y) }^2 + 2 \text{tr} \Brac{ \nabla s_\theta(y) } + \norm{ \nabla_y \log \nu(y) }^2 } + \const \\
    =~& \mE_{y\sim \nu(y)}  \norm{ s_\theta(y) - \nabla_y \log \nu(y) }^2 + \const.
\end{align*}
Here the second inequality follows from integration by parts; in the third inequality, we take
\[
    \const = \mE_{z \sim \nu(z) } \mE_{y\sim \nu(y|z)} \norm{ \nabla_y \log \nu(y|z) }^2 - \mE_{y\sim \nu(y)} \norm{ \nabla_y \log \nu(y) }^2,
\]
which is independent of $\theta$; the last equality follows from the same integration by parts trick.

It then suffices to prove \eqref{eqn:pf_ujstar}.
Note that by definition,
\[
    u_j^*(z,t) = \mE_{y'\sim p_t} \Rectbrac{ s_j (y',t) ~\Big|~ y_{\mcN_j^r}' = z } = \dfrac{  \int \nabla_j p_t(z,y_{\mcN_j^{r_\bot}}) \mdd y_{\mcN_j^{r_\bot}} }{ \int p_t(z,y_{\mcN_j^{r_\bot}}) \mdd y_{\mcN_j^{r_\bot}} }.
\]
Since $p_t(y) = \int \GN ( y; \alpha_t x, \sigma_t^2 I ) p_0(x) \mdd x$, it holds that
\[
    \nabla_j p_t(y) = \int - \sigma_t^{-2} (y_j - \alpha_t x_j) \GN ( y; \alpha_t x, \sigma_t^2 I ) p_0(x) \mdd x.
\]
Plugging the above relation into the expression of $u_j^*(z,t)$ and marginalizing out $x_{\mcN_j^{r_\bot}},y_{\mcN_j^{r_\bot}}$, we obtain
\begin{align*}
    u_j^*(z,t) =~& \dfrac{ \int \Brac{ - \sigma_t^{-2} (z_j - \alpha_t x_j) } \GN ( (z, y_{\mcN_j^{r_\bot}}) ; \alpha_t x, \sigma_t^2 I ) p_0(x) \mdd x \mdd y_{\mcN_j^{r_\bot}} }{ \int \GN ( (z, y_{\mcN_j^{r_\bot}}) ; \alpha_t x, \sigma_t^2 I ) p_0(x) \mdd x \mdd y_{\mcN_j^{r_\bot}} } \\
    =~& \dfrac{ \int \Brac{ - \sigma_t^{-2} (z_j - \alpha_t x_j) } \GN ( z ; \alpha_t x_{\mcN_j^r}, \sigma_t^2 I ) \pi_0^j(x_{\mcN_j^r}) \mdd x_{\mcN_j^r} }{ \int \GN ( z ; \alpha_t x_{\mcN_j^r}, \sigma_t^2 I ) \pi_0^j(x_{\mcN_j^r}) \mdd x_{\mcN_j^r} }.
\end{align*}
On the other hand, by the definition of $\pi_t^j(z)$ and the same marginalization, we have
\begin{align*}
    \nabla_j &\log \pi_t^j(z) = \frac{ \nabla_j \pi_t^j(z) }{ \pi_t^j(z) } = \frac{ \int \nabla_j \GN(z;\alpha_t x,\sigma_t^2 I) \pi_0^j(x) \mdd x }{ \int \GN(z;\alpha_t x,\sigma_t^2 I) \pi_0^j(x) \mdd x  } \\
    &= \dfrac{ \int \Brac{ - \sigma_t^{-2} (z_j - \alpha_t x_j) } \GN ( z; \alpha_t x, \sigma_t^2 I ) \pi_0^j(x) \mdd x }{ \int \GN ( z; \alpha_t x, \sigma_t^2 I ) \pi_0^j(x) \mdd x } = u_j^*(z,t) .
\end{align*}
This completes the proof.
\end{proof}

\subsection{Prototype error analysis for denoising score matching}
\label{app:gen_samp_comp}

We present an error analysis for the denoising score matching estimator of (3.4) in a general function class.
The key result is \Cref{thm:pf_gen_samp_comp}, which will be used in the proof of Theorem~3.4 and Theorem~3.5.
To adapt the analysis for the localized diffusion models, we consider denoising score matching for a component of the score function, instead of the full score.

\subsubsection{Problem setup}

Consider a target distribution $p_0 \in C^2(\mR^D)$.
Fix $j \in [b]$, and write the $j$-th component of the true score function as $s_j(y,t)$.
We will omit the subscript $j$ when there is no confusion.
Let $I_N = [\underline{T}, \mean{T}]$ be the time interval where $\underline{T} \asymp N^{-\alpha}$ for some $\alpha > 0$ and $\mean{T} \asymp \log N$.
We aim to approximate $s_j$ by a function
\[
    f \in \msF_N \subseteq \{ f: \mR^D \times I_N \to \mR^{d_j} \}.
\]
Here we do not require $f$ to be localized.
For any $f: \mR^D \times I_N \to \mR^{d_j}$, define the loss function
\begin{equation}
\label{eqn:app_loss_def}
    \ell_{f} (x) := \int_{I_N}
    \E_{\epsilon(t) \sim \GN(0,I)}
    \norm{
        f(\alpha_tx+\sigma_t\epsilon(t),t)
        +
        \sigma_t^{-1}\epsilon_j(t)
    }^2
    \mdd t .
\end{equation}
Given samples \(X^{(1)},\dots,X^{(N)}\), define the empirical distribution by
\[
    \mP_N = \frac1N \sum_{i=1}^N \delta_{X^{(i)}}.
\]
Write $\mP f$ for the expectation of $f$ over $X\sim p_0$ unless otherwise specified, and $\mP_N f$ for the expectation of $f$ over the empirical distribution $\mP_N$.
The population loss and empirical loss of $f$ are given by
\[
    \mcL (f) := \mP \ell_{f}, \quad \widehat{\mcL} (f) := \mP_N \ell_{f}.
\]
Define the population loss and empirical loss minimizers by
\begin{equation}
\label{eqn:pf_comp_def}
    f^\circ \in \argmin_{f} \mcL (f),\quad \widehat{f} \in \argmin_{f \in \msF_N} \widehat{\mcL} (f).
\end{equation}
We aim to control the $j$-th component of the score estimation error
\begin{equation}
    \mcE (f) := \int_{I_N} \mE_{y \sim p_t} \Rectbrac{ \norm{ f(y,t) - s_j(y,t) }^2 } \mdd t.
\end{equation}
Note that the true score function $s_j$ satisfies
\[
    s_j(y,t) = \mE_{x \sim p_0, \epsilon(t) \sim \GN(0,I)} \Rectbrac{ - \sigma_t^{-1} \epsilon_j(t) ~\Big|~ \alpha_t x + \sigma_t \epsilon(t) = y }.
\]
Therefore, we can show that
\[
    \mcE (f) = \mcL(f) - \mcL(s_j).
\]
In fact, denote $y = \alpha_t x + \sigma_t \epsilon(t)$ where $x \sim p_0$ and $\epsilon(t) \sim \GN(0,I)$, we have
\begin{align*}
    \mcL(f) - \mcL(s_j) =~& \E \int_{I_N} \Rectbrac{ \norm{ f(y,t) + \sigma_t^{-1}\epsilon_j(t) }^2 - \norm{ s_j(y,t) + \sigma_t^{-1}\epsilon_j(t) }^2 } \mdd t \\
    =~& \E \int_{I_N} \E \Rectbrac{ \norm{ f(y,t) - s_j(y,t) }^2 + 2 (f(y,t) - s_j(y,t))\matT (s_j(y,t) + \sigma_t^{-1}\epsilon_j(t)) ~\Big|~ y } \mdd t \\
    =~& \E \int_{I_N} \norm{ f(y,t) - s_j(y,t) }^2 \mdd t = \mcE(f).
\end{align*}
We aim to control the generalization error $\mcE(\widehat{f})$ by the approximation error $\mcE(f^\circ)$ and the statistical error (or the excess risk).

\subsubsection{Main results}

We need the following assumption on $\msF_N$.

\begin{asm}
\label{asm:pf_func_class}

(1) Uniform boundedness. There exists $M_N = \mcO(\log N)$ such that
\begin{equation}
\label{eqn:pf_f_bound}
    \forall f \in \msF_N,\quad \sup_{t \in [\underline{T}, \mean{T}]} \sigma_t \norm{f(\cdot,t)}_\infty \le M_N.
\end{equation}
(2) There exists $A_N,\Gamma_N > 0$ such that for any $t > 0$,
\begin{equation}
\label{eqn:pf_F_ent_cond}
    \log \msN ( t,\msF_N, \norm{\cdot}_{L^\infty} ) \le A_N \log_+ ( \Gamma_N / t ).
\end{equation}
Moreover, $A_N = \mcO(N^p \poly \log N)$ and $\Gamma_N = \exp \Brac{ \mcO( \poly \log N) }$ for some $p > 0$.
\end{asm}

\begin{thm}
\label{thm:pf_gen_samp_comp}
Assume that $p_0 \in C^2$, and $\msF_N$ satisfies \Cref{asm:pf_func_class}.
Then for any $\tau \in (0,1)$, with probability at least $1-\mee^{-\tau}$, it holds that
\begin{equation}
\label{eqn:pf_gen_samp_comp}
    \mcE(\widehat{f}) \le C \Brac{ \varepsilon_{\rm stat} + \mcE(f^\circ) }, \quad \varepsilon_{\rm stat} = \mcO \Brac{ \Brac{ \frac{\tau}{N} + N^{-(1-p)} } \poly \log N }.
\end{equation}
Here $C$ is a universal constant.
Moreover, if the constants in \Cref{asm:pf_func_class} are independent of $D$, then $\varepsilon_{\rm stat}$ is also independent of $D$.
\end{thm}

\begin{rem}
Since $f$ takes values in $\mR^{d_j}$, $M_N$ usually depends on $d_j$ rather than $D$.
When $\msF_N$ is induced by a low-dimensional function class (as in Section~3.4), its entropy can be independent of $D$.
\end{rem}

\subsubsection{Proofs}

First, we state two elementary lemmas.

\begin{lem}
It holds that
\begin{equation}
\label{eqn:pf_JN}
    J_N = \int_{I_N} \sigma_t^{-2} \mdd t = \mcO(\log N).
\end{equation}
\end{lem}

\begin{proof}
Since $\mean{T} \asymp \log N$, direct computation gives
\begin{align*}
    J_N =~& \int_{\underline{T}}^{\mean{T}} \frac{1}{1-\mee^{-2t}} \mdd t = \mean{T} - \underline{T} + \frac{1}{2} \log \frac{1-\mee^{-2\underline{T}}}{1-\mee^{-2\mean{T}}} \\
    \le~& \mean{T} + \frac{1}{2} \log \frac{1}{1-\mee^{-2\mean{T}}} \le C \log N.
\end{align*}
This completes the proof.
\end{proof}

\begin{lem}
\label{lem:pf_log_int}
There exists a universal constant $C>0$ such that $\forall a, b > 0$,
\[
    \int_0^a \sqrt{\log_+(b/t)}\,\mdd t
    \le
    C a\left(1+\sqrt{\log_+(b/a)}\right).
\]
\end{lem}

\begin{proof}
First note $ \log_+(ab) \le \log_+(a) + \log_+(b) $.
Let $t = as$, then
\begin{align*}
    \int_0^a \sqrt{\log_+(b/t)}\,\mdd t =~& a\int_0^1 \sqrt{\log_+(b/(as)) } \,\mdd s \\
    \le~& a\int_0^1 \sqrt{\log_+(b/a)}\,\mdd s + a\int_0^1 \sqrt{\log(1/s)}\,\mdd s \\
    \le~& C a \left(1+\sqrt{\log_+(b/a)}\right),
\end{align*}
since $\int_0^1 \sqrt{\log(1/s)}\,\mdd s < \infty$ is a universal constant.
\end{proof}

For $f\in\msF_N$, define (recall \eqref{eqn:pf_comp_def})
\begin{equation}
\label{eqn:pf_delta_def}
    \delta_f(y,t):=f(y,t)-f^\circ(y,t),
    \quad
    \Delta^2(f):=
    \int_{I_N}
    \norm{\delta_f(\cdot,t)}_{L^2(p_t)}^2
    \mdd t .
\end{equation}
Note the upper bounds of $\delta_f$ and $\Delta^2(f)$:
\begin{equation}
\label{eqn:pf_delta_bound}
    \norm{\delta_f (y,t)} \le 2 M_N \sigma_t^{-1}, \quad \Delta^2(f) \le \int_{I_N} 4 M_N^2 \sigma_t^{-2} \mdd t = 4 M_N^2 J_N.
\end{equation}

The following lemma is a key lemma that controls $\sup_{f \in \msF_N} | (\mP - \mP_N) g_f |$ for certain classes of $f$-dependent function $g_f$ with a \emph{fast rate} $\mcO(N^{-1})$, following the standard local Rademacher complexity theory with a peeling argument.

\begin{lem}
\label{lem:pf_gen_peeling}
Let $g_f(x)$ be a $f$-dependent function such that there exists $B_N > 0$,
\begin{equation}
\label{eqn:pf_g_bound}
    \forall f \in \msF_N, \quad \norm{g_f}_\infty \le B_N, \quad \mP g_f^2 \le B_N \Delta^2(f).
\end{equation}
Moreover, there exists $L_N > 0$ such that
\begin{equation}
\label{eqn:pf_g_lip}
    \forall f,f' \in \msF_N, \quad \norm{g_f - g_{f'}}_{L^2(\mP)} \le L_N \norm{f-f'}_\infty.
\end{equation}
Then for any $\tau > 0$, with probability at least $1-2\mee^{-\tau}$, it holds that
\[
    \sup_{f \in \msF_N} | (\mP - \mP_N) g_f | \le \frac{\Delta^2(f)}{4} + \frac{C B_N}{N} \Brac{ \tau +  A_N \log_+ ( N L_N \Gamma_N / B_N ) }.
\]
\end{lem}

\begin{proof}[Proof of \Cref{lem:pf_gen_peeling}]

Denote the function class
\[
    \msG_N = \{ g_f: f \in \msF_N \}, \quad \msG_N (a) = \{ g_f: f \in \msF_N, \Delta^2(f) \le a \}.
\]
By the Lipschitz condition \eqref{eqn:pf_g_lip}, we have
\[
    \log \msN ( t,\msG_N, \norm{\cdot}_{L^2(\mP)} ) \le \log \msN ( t/L_N,\msF_N, \norm{\cdot}_\infty ) \le A_N \log_+ ( L_N \Gamma_N / t ).
\]
For $g_f \in \msG_N (a)$, by \eqref{eqn:pf_g_bound}, we have
\[
    \norm{g_f}_\infty \le B_N, \quad \norm{g_f}_{L^2(\mP)} = \sqrt{\mP g_f} \le \sqrt{B_N a}.
\]
Denote the supremum of the empirical process
\[
    Z_a := \sup_{f \in \msF_N (a)} | (\mP - \mP_N) g_f |.
\]
By Bousquet's inequality, for any $\tau > 0$, with probability at least $1-\mee^{-\tau}$, it holds that
\[
    Z_a \le \E Z_a + \sqrt{\frac{2 \tau B_N (a + 2\E Z_a) }{N}} + \frac{B_N \tau}{3N}.
\]
On the other hand, by the maximal inequality for empirical processes \cite{MR1385671}, we have
\begin{align*}
    \E Z_a \le~& \frac{C}{\sqrt{N}} \int_0^{\sqrt{B_N a}} \sqrt{ \log \msN ( t,\msG_N, \norm{\cdot}_{L^2(\mP)} ) } \mdd t \\
    \le~& C \sqrt{ \frac{A_N}{N} } \int_0^{\sqrt{B_N a}} \sqrt{ \log_+ ( L_N \Gamma_N / t ) } \mdd t \\
    \le~& C' \sqrt{\frac{A_N B_N a }{N} } \Brac{ 1 + \sqrt{ \log_+ ( L_N \Gamma_N / \sqrt{B_N a} ) } } =: \psi(a).
\end{align*}
Here we use \Cref{lem:pf_log_int} in the last step.
Let $a_0 = \min \{ a > 0 : \psi(a_0) = a_0/32 \}$.
This implies
\begin{equation}
    a_0 \le C'' \frac{A_N B_N }{N} \log_+ ( N L_N \Gamma_N / B_N ).
\end{equation}
Since $\psi(a)/a$ is non-increasing, we have $\psi(a) \le a/16$ for any $a \ge a_0$.
Therefore, for any $a \ge a_0$, with probability at least $1-\mee^{-\tau}$, it holds that for some universal constant $C > 0$,
\[
    Z_a \le \frac{a}{16} + \sqrt{\frac{3 \tau B_N a }{N}} + \frac{B_N \tau}{3N} \le \frac{a}{8}  + C \frac{B_N \tau}{N}.
\]
Here we use Young's inequality.
Denote $a_j = a_0 2^j$ and $\tau_j = \tau + 2 \log (j+1)$ for $j \ge 0$. Then
\[
    \Pr \Brac{ \forall j \ge 0,~ Z_{a_j} > \frac{a_j}{8} + C \frac{B_N \tau_j}{N} } \le \sum_{j=0}^\infty \mee^{-\tau_j} \le 2 \mee^{-\tau}
\]
Therefore, with probability at least $1-2\mee^{-\tau}$, for any $g_f$ such that $\Delta^2(f) \ge a_0$, there exists $j \ge 0$ such that $a_j \le \Delta^2(f) < a_{j+1}$, and we have
\[
    | (\mP - \mP_N) g_f | \le Z_{a_{j+1}} \le \frac{a_{j+1}}{8} + C \frac{B_N \tau_{j+1}}{N} \le \frac{\Delta^2(f)}{4} + C \frac{B_N (\tau+\log(j+1))}{N}.
\]
The upper bound \eqref{eqn:pf_delta_bound} implies that $j \le \log_2 (4 M_N^2 J_N / a_0)$.
When $\Delta^2(f) < a_0$, we have
\[
    | (\mP - \mP_N) g_f | \le Z_{a_0} = \frac{a_0}{8} + C \frac{B_N \tau}{N}.
\]
Thus, for all $f \in \msF_N$, with probability at least $1-2\mee^{-\tau}$, it holds that
\[
    | (\mP - \mP_N) g_f | \le \frac{\Delta^2(f)}{4} + \frac{C B_N}{N} \Brac{ \tau +  A_N \log_+ ( N L_N \Gamma_N / B_N ) }.
\]
This completes the proof.
\end{proof}

Finally, we are ready to prove \Cref{thm:pf_gen_samp_comp}.

\begin{proof}[Proof of \Cref{thm:pf_gen_samp_comp}]

Recall the definition of $\delta_f$ and $\Delta(f)$ in \eqref{eqn:pf_delta_def}, and define
\[
    \mcQ_f = \int_{I_N}
    \E_{\epsilon(t) \sim \GN(0,I)} \norm{
        \delta_f (\alpha_tx+\sigma_t\epsilon(t),t)
    }^2
    \mdd t.
\]
\[
    \mcC_f = \int_{I_N}
    \E_{\epsilon(t) \sim \GN(0,I)} \ip{
        \delta_f (\alpha_tx+\sigma_t\epsilon(t),t)
    }{ f^\circ(\alpha_tx+\sigma_t\epsilon(t),t)
        +
        \sigma_t^{-1}\epsilon_j(t) }
    \mdd t.
\]
Note that $\mP \mcQ_f = \Delta^2(f)$. Denote $y = \alpha_t x + \sigma_t \epsilon(t)$ where $x \sim p_0$ and $\epsilon(t) \sim \GN(0,I)$. Then
\begin{align*}
    \norme{ \mP \mcC_f } =~& \norme{ \E \int_{I_N} \ip{ \delta_f (y,t) }{ f^\circ(y,t) + \sigma_t^{-1}\epsilon_j(t) } \mdd t } \\
    =~& \norme{ \E \int_{I_N}
    \E \Rectbrac{ \ip{ \delta_f (y,t) }{ f^\circ(y,t) + \sigma_t^{-1}\epsilon_j(t) } | y } \mdd t } \\
    =~& \norme{ \E \int_{I_N} \ip{ \delta_f (y,t) }{ f^\circ(y,t) - s_j(y,t) } \mdd t } \\
    \le~& \Brac{ \E \int_{I_N} \norm{ \delta_f (y,t) } \mdd t }^{1/2} \Brac{ \E \int_{I_N} \norm{  f^\circ(y,t) - s_j(y,t) } \mdd t }^{1/2} \\
    =~& \Delta(f) \cdot \sqrt{\mcE(f^\circ)}.
\end{align*}
Direct computation shows that (recall \eqref{eqn:app_loss_def} for the definition of $\ell_f$)
\[
    \ell_f-\ell_{f^\circ}
    = \mcQ_f+2\mcC_f.
\]
Since $\widehat{f} \in \argmin_{f \in \msF_N} \widehat{\mcL} (f)$, we have
\[
    \mP_N ( \ell_{\widehat{f}} - \ell_{f^\circ} ) = \widehat{\mcL}(\widehat{f}) - \widehat{\mcL}(f^\circ) \le 0.
\]
\[
    \St~ \mP ( \ell_{\widehat{f}} - \ell_{f^\circ} ) = \Brac{ \mP - \mP_N } ( \ell_{\widehat{f}} - \ell_{f^\circ} ) + \mP_N ( \ell_{\widehat{f}} - \ell_{f^\circ} ) \le \Brac{ \mP - \mP_N } ( \ell_{\widehat{f}} - \ell_{f^\circ} ).
\]
Since $\mP ( \ell_{\widehat{f}} - \ell_{f^\circ} ) = \mP \mcQ_{\widehat{f}} + 2 \mP \mcC_{\widehat{f}}$, we have
\begin{equation}
\label{eqn:pf_gen_samp_comp_decomp}
\begin{aligned}
    \Delta^2(\widehat{f}) =~& \mP \mcQ_{\widehat{f}} \le \Brac{ \mP - \mP_N } ( \mcQ_{\widehat{f}}+2\mcC_{\widehat{f}} ) - 2 \mP \mcC_{\widehat{f}} \\
    \le~& \Brac{ \mP - \mP_N } ( \mcQ_{\widehat{f}}+2\mcC_{\widehat{f}} ) + 2 \Delta(\widehat{f}) \cdot \sqrt{\mcE(f^\circ)}.
\end{aligned}
\end{equation}
It remains to control the empirical process term $\Brac{ \mP - \mP_N } (\mcQ_{\widehat{f}} + 2\mcC_{\widehat{f}})$.
We will apply \Cref{lem:pf_gen_peeling} to control it, and thus it suffices to verify the conditions therein.
Denote $g_f = \mcQ_f + 2 \mcC_f$, we have
\begin{align*}
    \norm{g_f}_\infty \le~& \int_{I_N} 4  M_N^2 \sigma_t^{-2} \mdd t + 2 \int_{I_N}
    \E_{\epsilon(t) \sim \GN(0,I)} \Rectbrac{ 2 M_N \sigma_t^{-1} \Brac{ M_N \sigma_t^{-1} + \sigma_t^{-1} \norm{ \epsilon_j(t) } } } \mdd t \\
    \le~& C M_N (M_N + \sqrt{d_j}) J_N.
\end{align*}
Moreover, $\mP g_f^2 \le 2 \mP \mcQ_f^2 + 8 \mP \mcC_f^2$, and
\[
    \mP \mcQ_f^2 \le \norm{\mcQ_f}_\infty \cdot \mP \mcQ_f \le 4 M_N^2 J_N \Delta^2(f).
\]
By Cauchy--Schwarz, denote $y = \alpha_t x + \sigma_t \epsilon(t)$
\begin{align*}
    \mP \mcC_f^2 \le~& \mP \Rectbrac{ \int_{I_N} \E_{\epsilon(t) \sim \GN(0,I)} \norm{ f (y,t) - f^\circ (y,t) }^2 \mdd t \cdot \int_{I_N} \E_{\epsilon(t) \sim \GN(0,I)} \norm{ f^\circ(y,t) + \sigma_t^{-1}\epsilon_j(t) }^2
    \mdd t } \\
    \le~& \mP \Rectbrac{ \int_{I_N} \E_{\epsilon(t) \sim \GN(0,I)} \norm{ f (y,t) - f^\circ(y,t) }^2 \mdd t \cdot \int_{I_N} \Brac{ 2 M_N^2 \sigma_t^{-2} + 2 \sigma_t^{-2} d_j } \mdd t } \\
    =~& 2 \Brac{ M_N^2 + d_j } J_N \Delta^2(f).
\end{align*}
Combined, we obtain
\[
    \mP g_f^2 \le C (M_N + \sqrt{d_j})^2 J_N \Delta^2(f).
\]
Thus, we can take $B_N = C (M_N + \sqrt{d_j})^2 J_N$ in \eqref{eqn:pf_g_bound}.
Next, we verify the Lipschitz condition \eqref{eqn:pf_g_lip}.
Note $g_f - g_{f'} = \ell_f - \ell_{f'}$.
Direct computation gives (denote $y = \alpha_t x + \sigma_t \epsilon(t)$)
\begin{align*}
    \norm{g_f - g_{f'}}_{L^2(\mP)} \le~& \norm{  \int_{I_N}
    \E_{\epsilon(t) \sim \GN(0,I)} \norm{
        f (y,t) - f' (y,t)
    }^2
    \mdd t }_{L^2(\mP)} \\
    &+ \norm{ \int_{I_N}
    \E_{\epsilon(t) \sim \GN(0,I)} \ip{
        f (y,t) - f' (y,t)
    }{ f'(y,t) + \sigma_t^{-1}\epsilon_j(t) }
    \mdd t }_{L^2(\mP)} \\
    \le~& \norm{ f - f' }_\infty \Brac{ \int_{I_N} \E \norm{
        f (y,t) - f' (y,t)
    }^2 \mdd t }^{1/2} \\
    &+ \norm{ f - f' }_\infty \Brac{ \int_{I_N} \E \norm{
        f'(y,t) + \sigma_t^{-1}\epsilon_j(t)
    }^2 \mdd t }^{1/2} \\
    \le~& \norm{ f - f' }_\infty \Brac{ 2 M_N \sqrt{J_N} + 2 (M_N + \sqrt{d_j}) \sqrt{J_N} }.
\end{align*}
Thus, we can take $L_N = C (M_N + \sqrt{d_j}) \sqrt{J_N}$ in \eqref{eqn:pf_g_lip}.
Then by \Cref{lem:pf_gen_peeling}, for any $\tau > 0$, with probability at least $1-2\mee^{-(\tau+\log 2)} = 1-\mee^{-\tau}$, it holds that
\[
    \Brac{ \mP - \mP_N } ( \mcQ_{\widehat{f}} + 2\mcC_{\widehat{f}} ) \le \frac{\Delta^2(\widehat{f})}{4} + \frac{C B_N }{N} \Brac{ \tau + \log 2 +  A_N \log_+ ( N L_N \Gamma_N / B_N ) }.
\]
Plugging the above bound into \eqref{eqn:pf_gen_samp_comp_decomp}, and noting that $2 \Delta(\widehat{f}) \cdot \sqrt{\mcE(f^\circ)} \le \frac{\Delta^2(\widehat{f})}{4} + 4 \mcE(f^\circ)$, we obtain
\[
    \Delta^2(\widehat{f}) \le \frac{\Delta^2(\widehat{f})}{2} + C \varepsilon_{\rm stat} + 4 \mcE(f^\circ) ~\St~ \Delta^2(\widehat{f}) \le C' \Brac{ \varepsilon_{\rm stat} + \mcE(f^\circ) }.
\]
where we denote $ \varepsilon_{\rm stat} = \frac{B_N}{N} \Brac{ \tau +  A_N \log_+ ( N L_N \Gamma_N / B_N ) } $.
Finally, notice that
\begin{align*}
    \mcE(\widehat{f}) =~& \int_{I_N} \mE_{y \sim p_t} \Rectbrac{ \norm{ f(y,t) - s_j(y,t) }^2 } \mdd t \\
    \le~& 2 \int_{I_N} \mE_{y \sim p_t} \Rectbrac{ \norm{ f(y,t) - f^\circ(y,t) }^2 } \mdd t + 2 \int_{I_N} \mE_{y \sim p_t} \Rectbrac{ \norm{ f^\circ(y,t) - s_j(y,t) }^2 } \mdd t \\
    =~& 2 \Delta^2(\widehat{f}) + 2 \mcE(f^\circ) \le C \Brac{ \varepsilon_{\rm stat} + \mcE(f^\circ) }.
\end{align*}
This completes the proof.
\end{proof}

\subsection{Proof of \texorpdfstring{Theorem~3.4}{Theorem 3.4}}
\label{app:samp_comp}

\begin{proof}
By the Pythagorean equality (3.15),
\[
    \norm{ \widehat{s}_j (x,t) - s_j (x,t) }_{L^2(p_t)}^2 = \normo{ \widehat{s}_j (x,t) - s_j^*(x,t) }_{L^2(p_t)}^2 + \normo{ s_j^*(x,t) - s_j (x,t) }_{L^2(p_t)}^2.
\]
By Proposition~3.1 and Theorem~3.2, it holds that
\begin{equation}
\label{eqn:pf_err_decomp}
\begin{aligned}
    \KL ( &p_{\underline{T}} \| \widehat{q}_{\mean{T}-\underline{T}} )
    \le \mee^{-2\mean{T}} \KL ( p_0 \| \GN(0,I) ) + \int_{\underline{T}}^{\mean{T}} \norm{ \widehat{s}(x,t) - s(x,t) }_{L^2(p_t)}^2 \mdd t \\
    =~& \mee^{-2\mean{T}} \KL ( p_0 \| \GN(0,I) ) + \sum_{j=1}^b \int_{\underline{T}}^{\mean{T}} \Brac{ \normo{ \widehat{s}_j (x,t) - s_j^*(x,t) }_{L^2(p_t)}^2 + \normo{ s_j^*(x,t) - s_j (x,t) }_{L^2(p_t)}^2 } \mdd t \\
    \le~& \mee^{-2\mean{T}} \KL ( p_0 \| \GN(0,I) ) + \sum_{j=1}^b \int_{\underline{T}}^{\mean{T}} \normo{ \widehat{s}_j (x,t) - s_j^*(x,t) }_{L^2(p_t)}^2 \mdd t + C D (r+1)^{\nu} \mee^{-c(r+1)}.
\end{aligned}
\end{equation}
Note that $ s_j^*(x,t) = u_j^*(x_{\mcN_j^r},t) $ and $ \widehat{s}_j (x,t) = \widehat{u}_j (x_{\mcN_j^r},t)$, where $ u_j^*$ is the optimal localized score (3.10) and
\[
    \widehat{u}_j \in \argmin_{ u_{\theta,j} \in \msU_j^\varepsilon } \widehat{\mcL}_j^N (u_{\theta,j}).
\]
Here $\widehat{\mcL}_j^N (u_{\theta,j})$ is the localized empirical loss (3.4).
Define its population version by
\[
    \mcL_j (u_{\theta,j}) = \E_{x\sim p_0} \int_{\underline{T}}^{\mean{T}} \mE_{\epsilon (t) \sim \GN(0,I)} \Rectbrac{ \norm{ u_{\theta,j}(\alpha_t x_{\mcN_j^r} + \sigma_t \epsilon_{\mcN_j^r}(t), t) + \sigma_t^{-1} \epsilon_j(t) }^2 } \mdd t.
\]
Accordingly, denote the population minimizer
\[
    u_j^\circ \in \argmin_{ u_{\theta,j} \in \msU_j^\varepsilon } \mcL_j (u_{\theta,j}).
\]

Fix $j \in [b]$ for now.
Let $z = x_{\mcN_j^r}$, and define the $j$-th component score error by
\[
    \mcE_j(u) = \int_{\underline{T}}^{\mean{T}} \norm{ u (z,t) - u_j^* (z,t) }_{L^2(\pi_t^j)}^2 \mdd t.
\]
Recall here $\pi_t^j$ is the marginal distribution of $p_t$ on $z$ (3.17).
It is straightforward to see that
\begin{equation}
\label{eqn:pf_equiv_err}
    \mcE_j(\widehat{u}_j) = \int_{\underline{T}}^{\mean{T}} \normo{ \widehat{s}_j (x,t) - s_j^*(x,t) }_{L^2(p_t)}^2 \mdd t.
\end{equation}
In Proposition~3.3, we show that $u_j^*$ is the true score of $\pi_t^j$, i.e., $u_j^*(z,t) = \nabla \log \pi_t^j(z)$.
Thus, we can apply \Cref{thm:pf_gen_samp_comp} with $p_0$ being $\pi_0^j$, $\mcE$ being $\mcE_j$, and $f^\circ$ being $u_j^\circ$.
It suffices to verify \Cref{asm:pf_func_class} for the current settings.

Take $\msF_N = \msU_j^\varepsilon$ and recall here $\varepsilon = N^{-2\gamma/(\deff+2\gamma)} $.
\eqref{eqn:pf_f_bound} is direct from (3.21), since
\[
    \forall u \in \msU_j^\varepsilon,\quad \sup_{t \in [\underline{T}, \mean{T}]} \sigma_t \norm{u(\cdot,t)}_\infty \le M_j^\varepsilon = \mcO(\log \varepsilon^{-1}) = \mcO (\log N).
\]
\eqref{eqn:pf_F_ent_cond} is direct from (3.23) as well, since $\log \msN ( t,\msU_j^\varepsilon, \norm{\cdot}_{L^\infty} ) \le A_j^\varepsilon \log_+ ( \Gamma_j^\varepsilon / t )$ where
\[
    A_j^\varepsilon = \mcO( \varepsilon^{-d_{j,r}/(2\gamma) } \poly \log \varepsilon^{-1} ) = \mcO( N^{d_{j,r}/(\deff + 2\gamma) } \poly \log \varepsilon^{-1} ),
\]
\[
    \Gamma_j^\varepsilon = \exp \Brac{ \mcO( \poly \log \varepsilon^{-1} ) } \le \exp \Brac{ \mcO( \poly \log N ) }.
\]
And one can take $p$ in \Cref{asm:pf_func_class} as
\[
    p = \deff/(\deff + 2\gamma) \ge d_{j,r}/(\deff + 2\gamma).
\]
By \Cref{thm:pf_gen_samp_comp}, for any $\tau \in (0,1)$, with probability at least $1-\mee^{-\tau}$, it holds that
\[
    \mcE_j (\widehat{u}_j) \le C \Brac{ \varepsilon_{\rm stat} + \mcE_j (u_j^\circ) }, \quad \varepsilon_{\rm stat} = \mcO \Brac{ \Brac{ \frac{\tau}{N} + N^{-2\gamma/ (\deff + 2\gamma )} } \poly \log N }.
\]
Note that $\mcE_j (u) = \mcL_j (u) - \mcL_j (u_j^*)$ and $u_j^\circ \in \argmin_{ u_{\theta,j} \in \msU_j^\varepsilon } \mcL_j (u_{\theta,j})$, so by Assumption~3.3,
\[
    \mcE_j (u_j^\circ) = \inf_{ u_{\theta,j} \in \msU_j^\varepsilon } \mcE_j (u_{\theta,j}) = \inf_{ u_{\theta,j} \in \msU_j^\varepsilon } \int_{\underline{T}}^{\mean{T}} \normo{ u_{\theta,j}- \nabla_j \log \pi_t^j }_{L^2(\pi_t^j)}^2 \mdd t \le \varepsilon = N^{-2\gamma/ (\deff + 2\gamma )}.
\]
Therefore, we conclude that
\begin{equation}
\label{eqn:pf_Ej_bound}
    \mcE_j (\widehat{u}_j) \le C_1 N^{-2\gamma/ (\deff + 2\gamma )} (\log N)^{C_2} + \tau N^{-1} (\log N)^{C_2}.
\end{equation}
For each $j$, we get a high probability bound as above.
Enlarge $\tau$ to be $\tau + \log b$, and take a union bound of the events, then with probability at least
\(
    1 - \sum_{j=1}^b \mee^{ - (\tau + \log b)} = 1 - \mee^{-\tau},
\)
it holds that \eqref{eqn:pf_Ej_bound} holds uniformly for all $j\in [b]$ with $\tau$ replaced by $\tau + \log b$.
Plugging this result into \eqref{eqn:pf_err_decomp} and using \eqref{eqn:pf_equiv_err}, we obtain the desired conclusion.
\end{proof}

\section{Proofs in \texorpdfstring{Section~3.4}{Section 3.4}}
\label{app:mix_samp_comp}

\subsection{Proof of \texorpdfstring{Theorem~3.5}{Theorem 3.5}}
\label{app:pf_mix_samp_comp}

\begin{proof}

As in Theorem~3.4, the main task is to control the score estimation error.
The statistical error is controlled via \Cref{thm:pf_gen_samp_comp}, and we need to control the approximation error using the new hypothesis class $\msH_r^{\widehat{\lambda}}$ (3.31).
Fix $j \in [b]$ for now.
For any $\bu_j := (u_j^{(1)},\ldots,u_j^{(K)})$, denote
\[
    f_{\bu_j} (x,t) := \sum_{k=1}^K \widehat{\lambda}_t^{(k)} (x) u_j^{(k)} (x_{\mcN_j^r},t), \quad \msF_j^\varepsilon := \{ f_{\bu_j} : \bu_j \in (\msU_j^\varepsilon)^{\otimes K} \}.
\]
Here $\varepsilon = N^{-2\gamma/(\deff+2\gamma)}$, and $\msU_j^\varepsilon$ is defined in Assumption~3.3.
Following \Cref{app:gen_samp_comp}, denote
\[
    \mcL (f) = \E_{x\sim p_0} \ell_f, \quad \widehat{\mcL} (f) = \frac{1}{N} \sum_{i=1}^N \ell_f (x_i), \quad \ell_f (x) := \int_{\underline{T}}^{\mean{T}} \mE_{\epsilon (t) \sim \GN(0,I)} \Rectbrac{ \norm{ f(x,t) + \sigma_t^{-1} \epsilon_j(t) }^2 } \mdd t,
\]
\[
    \mcE (f) := \int_{\underline{T}}^{\mean{T}} \mE_{x \sim p_t} \Rectbrac{ \norm{ f(x,t) - s_j(x,t) }^2 } \mdd t.
\]
By definition (3.32), $\widehat{\mcL}(f_{\bu_j}) = \widehat{\mcL}_j^N (\bu_j;\widehat{\lambda})$.
Since $\widehat{\bu}_j \in \argmin_{\bu_j \in (\msU_j^\varepsilon)^{\otimes K}} \widehat{\mcL}_j^N (\bu_j;\widehat{\lambda})$, we have
\[
    f_{\widehat{\bu}_j} \in \argmin_{f \in \msF_j^\varepsilon} \widehat{\mcL} (f).
\]
Accordingly, define population minimizer $f^\circ \in \argmin_{f \in \msF_j^\varepsilon} \mcL (f)$ and note
\[
    f^\circ = f_{\bu_j^\circ},\quad \bu_j^\circ \in \argmin_{\bu_j \in (\msU_j^\varepsilon)^{\otimes K}} \mcL_j (\bu_j;\widehat{\lambda}) := \mcL (f_{\bu_j}).
\]
We aim to apply \Cref{thm:pf_gen_samp_comp} to control $\mcE (f_{\widehat{\bu}_j})$.
It suffices to verify \Cref{asm:pf_func_class} for $\msF_j^\varepsilon$.
We will frequently use the feasibility condition (3.35), i.e.,
\[
    \widehat{\lambda}_t^{(k)} (x) \ge 0, \quad \sum_{k=1}^K \widehat{\lambda}_t^{(k)} (x) \equiv 1.
\]
By (3.21), for any $t\in [\underline{T}, \mean{T}]$,
\[
    \norme{ f_{\bu_j}(x,t) } \le \sum_{k=1}^K \widehat{\lambda}_t^{(k)} (x) \norme{ u_j^{(k)} (x_{\mcN_j^r},t) } \le \sum_{k=1}^K \widehat{\lambda}_t^{(k)} (x) M_j^\varepsilon \sigma_t^{-1} = M_j^\varepsilon \sigma_t^{-1}.
\]
Thus, we can take $M_N = M_j^\varepsilon = \mcO(\log \varepsilon^{-1}) = \mcO (\log N)$ in \eqref{eqn:pf_f_bound}.
Next, notice
\begin{align*}
    \norme{ f_{\bu_j} (x,t) - f_{\bv_j} (x,t) } \le~& \sum_{k=1}^K \widehat{\lambda}_t^{(k)} (x) \norme{ u_j^{(k)} (x_{\mcN_j^r},t) - v_j^{(k)} (x_{\mcN_j^r},t) } \\
    \le~& \sum_{k=1}^K \widehat{\lambda}_t^{(k)} (x) \cdot \max_{1\le k\le K} \normo{ u_j^{(k)} - v_j^{(k)} }_\infty = \max_{1\le k\le K} \normo{ u_j^{(k)} - v_j^{(k)} }_\infty.
\end{align*}
Thus, $ \normo{f_{\bu_j} - f_{\bv_j}}_\infty \le \max_{1\le k\le K} \normo{ u_j^{(k)} - v_j^{(k)} }_\infty $.
This implies that
\[
    \log \msN ( t,\msF_j^\varepsilon, \norm{\cdot}_{L^\infty} ) \le K \cdot \log \msN ( t, \msU_j^\varepsilon, \norm{\cdot}_{L^\infty} ) \le K A_j^\varepsilon \log_+ ( \Gamma_j^\varepsilon / t ).
\]
By Assumption~3.3, $A_j^\varepsilon = \mcO( N^p \poly \log N ) $ and $\Gamma_j^\varepsilon = \exp \Brac{ \mcO( \poly \log N ) }$ for $ p = \deff/(\deff + 2\gamma)$.
Therefore, by \Cref{thm:pf_gen_samp_comp}, for any $\tau \in (0,1)$, with probability at least $1-\mee^{-\tau}$, it holds that
\begin{equation}
\label{eqn:pf_mix_samp_comp_decomp}
    \mcE (f_{\widehat{\bu}_j}) \le C \Brac{ \varepsilon_{\rm stat} + \mcE (f_{\bu_j^\circ}) }, \quad \varepsilon_{\rm stat} = \mcO \Brac{ \Brac{ \frac{\tau}{N} + N^{-2\gamma/ (\deff + 2\gamma )} } \poly \log N }.
\end{equation}
Note this bound is independent of $D$.
It remains to control
\[
    \mcE (f_{\bu_j^\circ}) = \mcL(f_{\bu_j^\circ}) - \mcL(s_j) = \inf_{ f_{\bu_j} \in \msF_j^\varepsilon } \mcL(f_{\bu_j}) - \mcL(s_j) = \inf_{ f_{\bu_j} \in \msF_j^\varepsilon } \mcE(f_{\bu_j}).
\]
Recall (3.29), for any $f_{\bu_j} \in \msF_j^\varepsilon$,
\begin{equation}
\label{eqn:pf_mix_err_decomp}
\begin{aligned}
    \mcE (f_{\bu_j}) =~& \int_{\underline{T}}^{\mean{T}} \mE_{x \sim p_t} \Rectbrac{ \norm{ \sum_{k=1}^K \Brac{ \widehat{\lambda}_t^{(k)} (x) u_j^{(k)} (x_{\mcN_j^r},t) - \lambda_t^{(k)} (x) s_j^{(k)} (x,t) } }^2 } \mdd t \\
    \le~& 2 \int_{\underline{T}}^{\mean{T}} \mE_{x \sim p_t} \Rectbrac{ \norm{ \sum_{k=1}^K \Brac{ \widehat{\lambda}_t^{(k)} (x)  - \lambda_t^{(k)} (x) } u_j^{(k)} (x_{\mcN_j^r},t) }^2 } \mdd t \\
    &+ 2 \int_{\underline{T}}^{\mean{T}} \mE_{x \sim p_t} \Rectbrac{ \norm{ \sum_{k=1}^K \lambda_t^{(k)} (x) \Brac{ u_j^{(k)} (x_{\mcN_j^r},t) - s_j^{(k)} (x,t) } }^2 } \mdd t \\
    =:~& 2 \mcE^\lambda(f_{\bu_j}) + 2 \mcE^u(f_{\bu_j}).
\end{aligned}
\end{equation}
By Assumption~3.6, with probability at least $1-\mee^{-\tau}$, it holds that $\Delta_\lambda^N(\tau)$ defined in (3.36) is finite.
Then for $\mcE^\lambda$, by (3.21), it holds that
\begin{equation}
\label{eqn:pf_mix_lambda_err}
    \mcE^\lambda(f_{\bu_j}) \le (M_j^\varepsilon)^2 \int_{\underline{T}}^{\mean{T}} \sigma_t^{-2} \E_{x \sim p_t} \Rectbrac{ \Brac{ \sum_{k=1}^K \norme{ \widehat\lambda_t^{(k)}(x) - \lambda_t^{(k)}(x) } }^2 } \mdd t = (M_j^\varepsilon)^2 \Delta_\lambda^N(\tau).
\end{equation}
For $\mcE^u$, by convexity of $\norm{\cdot}^2$ and Jensen's inequality, it holds that
\begin{equation}
\label{eqn:pf_mix_u_err_decomp}
\begin{aligned}
    \mcE^u(f_{\bu_j}) \le~& \int_{\underline{T}}^{\mean{T}} \mE_{x \sim p_t} \Rectbrac{ \sum_{k=1}^K \lambda_t^{(k)} (x) \norm{ u_j^{(k)} (x_{\mcN_j^r},t) - s_j^{(k)} (x,t) }^2 } \mdd t \\
    =~& \sum_{k=1}^K w_k \int_{\underline{T}}^{\mean{T}} \mE_{x \sim p_t^{(k)}} \norm{ u_j^{(k)} (x_{\mcN_j^r},t) - s_j^{(k)} (x,t) }^2 \mdd t.
\end{aligned}
\end{equation}
Here we use $ p_t(x) \lambda_t^{(k)}(x) = w_k p_t^{(k)} (x) $.
To control the above term, consider
\[
    u_j^{*(k)} (z,t) := \nabla_j \log \pi_t^{j,(k)} (z),
\]
where $\pi_t^{j,(k)}$ denotes the marginal distribution of $p_t^{(k)}$ on $z=x_{\mcN_j^r}$.
Note that for all $k \in [K]$, $p_0^{(k)}$ satisfies the conditions of Theorem~3.2.
Therefore,
\[
    \int_{\underline{T}}^{\mean{T}} \normo{ u_j^{*(k)} (x_{\mcN_j^r},t) - s_j^{(k)} (x,t) }_{L^2(p_t^{(k)})}^2 \mdd t \le C d_j (r+1)^{\nu} \mee^{-c(r+1)}.
\]
Moreover, the Pythagorean equality (3.15) holds:
\begin{align*}
    \E_{x \sim p_t^{(k)}} &\normo{ u_j^{(k)} (x_{\mcN_j^r},t) - s_j^{(k)} (x,t) }^2 - \E_{x \sim p_t^{(k)}} \normo{ u_j^{*(k)} (x_{\mcN_j^r},t) - s_j^{(k)} (x,t) }^2 \\
    =~& \E_{x \sim p_t^{(k)}} \normo{ u_j^{(k)} (x_{\mcN_j^r},t) - u_j^{*(k)} (x_{\mcN_j^r},t) }^2 = \E_{z \sim \pi_t^{j,(k)}} \normo{ u_j^{(k)} (z,t) - u_j^{*(k)} (z,t) }^2.
\end{align*}
Taking infimum over $f_{\bu_j} \in \msF_j^\varepsilon$ in \eqref{eqn:pf_mix_u_err_decomp} and applying the above bound, we have
\begin{align*}
    \inf_{f_{\bu_j} \in \msF_j^\varepsilon} \mcE^u(f_{\bu_j}) \le~& \sum_{k=1}^K w_k \inf_{u_j^{(k)} \in \msF_j^\varepsilon} \int_{\underline{T}}^{\mean{T}} \mE_{x \sim p_t^{(k)}} \norm{ u_j^{(k)} (x_{\mcN_j^r},t) - s_j^{(k)} (x,t) }^2 \mdd t \\
    =~& \sum_{k=1}^K w_k \inf_{u_j^{(k)} \in \msF_j^\varepsilon} \int_{\underline{T}}^{\mean{T}} \E_{z \sim \pi_t^{j,(k)}} \normo{ u_j^{(k)} (z,t) - u_j^{*(k)} (z,t) }^2 \mdd t \\
    &\quad + \sum_{k=1}^K w_k \int_{\underline{T}}^{\mean{T}} \E_{x \sim p_t^{(k)}} \normo{ u_j^{*(k)} (x_{\mcN_j^r},t) - s_j^{(k)} (x,t) }^2 \mdd t \\
    \le~& \sum_{k=1}^K w_k \varepsilon + \sum_{k=1}^K w_k C d_j (r+1)^{\nu} \mee^{-c(r+1)} = \varepsilon + C d_j (r+1)^{\nu} \mee^{-c(r+1)}.
\end{align*}
Here we use Assumption~3.3 in the last inequality.
Combined with \eqref{eqn:pf_mix_err_decomp} and \eqref{eqn:pf_mix_lambda_err}, we have
\[
    \mcE (f_{\bu_j^\circ}) = \inf_{f_{\bu_j} \in \msF_j^\varepsilon} \mcE (f_{\bu_j}) \le 2 (M_j^\varepsilon)^2 \Delta_\lambda^N(\tau) + 2 \varepsilon + 2 C d_j (r+1)^{\nu} \mee^{-c(r+1)}.
\]
Plugging this bound into \eqref{eqn:pf_mix_samp_comp_decomp}, and in the event where both Assumption~3.6 and \eqref{eqn:pf_mix_samp_comp_decomp} hold, we have
\begin{align*}
    \mcE (f_{\widehat{\bu}_j}) \le~& C \Brac{ \varepsilon_{\rm stat} + (M_j^\varepsilon)^2 \Delta_\lambda^N(\tau) + \varepsilon + d_j (r+1)^{\nu} \mee^{-c(r+1)} } \\
    \le~& C' \Brac{ N^{-2\gamma/ (\deff + 2\gamma )} + \tau N^{-1} + \Delta_\lambda^N(\tau) } ( \log N )^{C''} + C''' d_j (r+1)^{\nu} \mee^{-c(r+1)}.
\end{align*}
Finally, by Proposition~3.1, it holds that with probability at least $1-2\mee^{-\tau}$,
\begin{align*}
    \KL ( p_{\underline{T}} \| \widehat{q}_{\mean{T}-\underline{T}} )
    \le~& \mee^{-2\mean{T}} \KL ( p_0 \| \GN(0,I) ) + \sum_{j=1}^b \int_{\underline{T}}^{\mean{T}} \norm{ \widehat{s}_j(x,t) - s_j(x,t) }_{L^2(p_t)}^2 \mdd t \\
    =~& \mee^{-2\mean{T}} \KL ( p_0 \| \GN(0,I) ) + \sum_{j=1}^b \mcE (f_{\widehat{\bu}_j}) \\
    \le~& \mee^{-2\mean{T}} \KL ( p_0 \| \GN(0,I) ) + C D (r+1)^{\nu} \mee^{-c(r+1)} \\
    &+ b C' \Brac{ N^{-2\gamma/ (\deff + 2\gamma )} + (\tau+\log b) N^{-1} + \Delta_\lambda^N(\tau) } ( \log N )^{C''}.
\end{align*}
Here we replace $\tau$ by $\tau + \log b$ to take a union bound over the events in  \eqref{eqn:pf_mix_samp_comp_decomp} for $j\in [b]$, while keeping the original $\tau$ for the events in Assumption~3.6.
This completes the proof.
\end{proof}

\bibliography{LDM}
\bibliographystyle{siam}

\end{document}